\documentclass[11pt,a4paper]{article}
\usepackage[margin=2.5cm]{geometry}
\usepackage[utf8]{inputenc}
\usepackage{authblk} 
\usepackage{newtxtext,newtxmath}
\usepackage{url}
\usepackage{graphicx}
\usepackage{amssymb}
\usepackage{amsthm}

\usepackage{subcaption} 
\usepackage{caption} 
\usepackage{csquotes} 

\usepackage[round]{natbib}
\usepackage{algorithm}
\usepackage[noend]{algpseudocode}
\usepackage{wrapfig}
\usepackage{import}

\usepackage[dvipsnames]{xcolor}
\definecolor{mycherryred}{RGB}{194,37,68}
\definecolor{standsoutabit}{RGB}{85, 28, 128}
\definecolor{nicegreen}{RGB}{50, 153, 57}
\definecolor{rubenscomment}{RGB}{29, 168, 50}
\usepackage[colorlinks=true,
  citecolor=mycherryred,
]{hyperref}
\usepackage{enumitem}
\usepackage{braket}
\usepackage{rubens}

\usepackage{xpatch}
\xapptocmd\normalsize{%
 \abovedisplayskip=12pt plus 3pt minus 9pt
 \belowdisplayskip=12pt plus 3pt minus 9pt
 \abovedisplayshortskip=0pt plus 3pt
 \belowdisplayshortskip=7pt plus 3pt minus 4pt
}{}{}

\title{Predictive Representations for \\Skill Transfer in Reinforcement Learning}

\author[1]{Ruben Vereecken}
\author[2]{Luke Dickens}
\author[1]{Alessandra Russo}
\affil[1]{Computer Science, Imperial College, UK}
\affil[2]{Information Studies, UCL, UK}
\graphicspath{{./img/}}

\usepackage[normalem]{ulem}

\begin{document}

\title{Predictive Representations for \\Skill Transfer in Reinforcement Learning}



\maketitle

\begin{center}
    \itshape
    Research conducted: September 2018 -- June 2021. \\
    This manuscript represents the work as of June 2021.
\end{center}

\begin{abstract}
A key challenge in
scaling up
Reinforcement Learning
is generalizing learned behaviour.
Without the ability to carry forward acquired knowledge
an agent is doomed to learn each task from scratch.
In this paper we develop a new formalism for
transfer by virtue of state abstraction.
Based on task-independent,
compact observations (outcomes)
of the environment,
we introduce
\emph{Outcome-Predictive State Representations}
(OPSRs),
agent-centered
and task-independent
abstractions that are made up of predictions of outcomes.
We show formally and empirically
that they have the potential
for optimal but limited transfer,
then overcome this trade-off
by introducing OPSR-based \emph{skills},
i.e. abstract actions (based on options)
that can be reused between tasks
as a result of state abstraction.
In a series of empirical studies,
we learn OPSR-based skills from demonstrations
and show how they speed up learning considerably
in entirely new and unseen tasks
without any pre-processing.
We believe that the framework
introduced in this work
is a promising step towards
transfer in RL in general,
and towards
transfer through
combining state and action abstraction
specifically.

\end{abstract}

\section{Introduction}
\label{Introduction}
One core feature humans seem to share
is the tendency to think of a task
in terms of subtasks;
the act of navigating through a door
may require taking out a key,
putting it to a lock and
turning a handle
before moving through the doorway.
However, at no point while reading this would you think
of the minute motor actions involved.
There appear to be levels of abstraction,
higher-level more abstract actions such as
\emph{unlocking a door},
described in terms of
\emph{taking out a key},
\emph{putting it in the lock} etc,
while intermediate level actions such as
\emph{taking out a key}
might decompose into
lower-level more primitive actions such as muscle contractions.
It is this phenomenon that gave rise to
\emph{hierarchical reinforcement learning} (HRL),
a popular area of present-day research
concerned with the decomposition of tasks
and learning of subtasks.

Not only are humans adept at decomposing tasks,
there is also the notion that we can reuse abstract behavior
across a variety of settings
where subtasks can be said to reoccur.
The act of unlocking a door is, intuitively,
mostly independent of which door is being unlocked
or the shape of the key,
and the differences can be thought of as variations on a common theme.
In this paper, we refer to abstract actions in this context as \emph{skills},
and suggest that this reusability of behavior,
i.e. this \emph{transfer} of skills,
is the core advantage of action abstraction.
This notion it not new
and others have discussed action abstraction
as a basis for transfer in RL
\citep{Konidaris2012a, Barto2013},
yet there is one obstacle standing in the way of \emph{skill transfer:}
a behavior learnt in one particular task's state space
can not readily be used in a different one.
There is a fundamental need for \emph{state abstraction}.
The skill of opening a door is independent
of the color of the door or the shape of the key;
it acts on the abstract concepts \emph{key} and \emph{door}
and ignores specifics.
In a recent opinion piece,
\citeauthor{Konidaris2019}
advocates for
joint state and action abstraction
\citep{Konidaris2019}.
In this paper
we show
that state and action abstraction
naturally go hand in hand.

\paragraph{}
This paper presents a new formalism
for thinking about how to transfer
learned behavior from one MDP to another
(Section~\ref{sec:general}).
This formalizes the concept of
\emph{state abstraction}
and
both
describes how to transfer a policy
between MDPs
by means of the \emph{transfer cover},
i.e. the overlap between state abstractions,
and the impact on performance.
Based on
\emph{outcomes}
(Section~\ref{sec:outcomes}),
environmental features
that are reward-relevant
and task-agnostic,
we introduce
\emph{outcome equivalent state abstractions}
which can be constructed from
\emph{predictions of outcomes}
(Section~\ref{sec:outcome_state_abstraction}).
We define
\emph{plannable} MDPs as MDPs where
reactive, state-dependent policies,
can not outperform
non-reactive \emph{plans},
and prove that outcome equivalent state abstractions
allow optimal transfer between plannable MDPs
(Section~\ref{sec:transfer_optimality}).

To overcome the limiting factor of transfer cover,
we introduce the more practical
\emph{Outcome-Predictive State Representations}
(OPSR)
and integrate these ideas with action abstraction
in the form of \emph{skills}
(Section~\ref{sec:skills_that_transfer}).
We show how these skills can be learned in
a fully automated way
such that combined,
they achieve full transfer between tasks
by virtue
of the OPSR state abstraction.
In a collection of empirical studies,
we both learn these skills
from demonstrations
and show that they speed up learning considerably
in entirely new and unseen tasks
with zero pre-processing
(Section~\ref{sec:experimental_evaluation}).
Section~\ref{sec:related_work}
describes the work on state abstractions and
HRL that form the rich backdrop for this work,
even detailing a cognitive motivation
for OPSRs.
We conclude that the framework introduced in this work
constitutes a promising step
towards a general approach for transfer
(Section~\ref{sec:conclusion}).

\section{General Framework}%
\label{sec:general}

This section
builds the theoretical framework
of transfer based on state abstraction
that we use as a foundation
to build and assess a new
type of state abstraction
in Section~\ref{sec:predictive_representations}.
We develop formalisms for comparing state abstractions,
detail how transfer can be conducted
in both the single-
and multi-task setting,
and describe in each setting
how state abstractions
impact performance following transfer.

\subsection{Markov Decision Processes and Problems}%
\label{sub:markov_decision_processes}
Reinforcement Learning is concerned with
sequential decision-making
in an environment
which is perceived through states and rewards.
It is conventional
to describe a Reinforcement Learning task
as a Markov Decision Process.
%
\begin{definition}[Markov Decision Process]
A Markov Decision Process (MDP)
is defined as a 4-tuple
$\mdp \defas \langle \sset, \aset, p, r \rangle$,
where $\sset$ is the discrete set of states
and $\aset$ the set of actions.
The MDP dynamics are described by the two functions
$p$ and $r$:
\begin{itemize}
  \item
    We write
    the probability to transition between states $s, s' \in \sset$
    following action $a \in \aset$
    as
    $$p(s'|s, a) \defas \Pr(S_{t+1}=s'| S_{t}=s, A_{t}=a).$$
    This function is called the \emph{transition dynamics} of the MDP.
  \item
    We write the expected reward
    for a transition $(s,a,s') \in \sset \times \aset \times \sset$
    as
    $$r(s, a, s') \defas \expected \left[ R_{t+1} \mid S_{t}=s, A_{t}=a, S_{t+1}=s' \right].$$
    We will also rely on
    the notion of expected reward for an action
    in a given state:
    $$r(s, a) \defas \expected \left[ R_{t+1} \mid S_{t}=s, A_{t}=a \right].$$
    Both functions can be referred to as the \emph{reward dynamics} of the MDP
    and which one is actually in use
    can be inferred from the context.
\end{itemize}
\end{definition}
The goal in Reinforcement Learning is then
to find an optimal behaviour in an MDP.
A behavior is represented by a policy
$\pi(a|s) = \Pr(A_{t}=a | S_{t}=s)$.
Optimality is defined in terms of
\emph{expected long-term accumulated reward}.
To properly define this,
we need to instantiate a
Markov Decision Problem.
A Markov Decision Problem
is simply a Markov Decision Process
extended with a
\emph{discount factor}
$\gamma$,
so that we can define
the long-term accumulated reward
as a discounted sum of rewards.
\begin{equation}
  G_{t} \defas \sum_{k=1}^{\infty} \gamma^{k-1} R_{t+k}.
\end{equation}
We can now define the
\emph{value} of a policy for a given state
as the expected accumulated reward
from that state.
\begin{equation}
  v_\pi(s) \defas \expected \left[ G_t \mid S_t = s \right]
  \label{eq:policy_value}
\end{equation}
Where necessary
to avoid ambiguity,
we include the MDP $\mdp$
in the value function and instead write
$v_\pi(s; \mdp)$.
The task in Reinforcement Learning now consists of
finding a policy $\pi$
that maximizes the quantity
\eqref{eq:policy_value}
for some state $s$,
which is usually a fixed start state.

\paragraph{}
One last note on notation.
We sometimes rely on partially applied functions,
i.e.
functions where some of the parameters are given and some are left out.
For example,
$\pi(\cdot | s)$
is the policy for state $s$.
This allows us to say
for example
that a policy
has the same distribution over actions
for two different states $s$ and $s'$,
denoted $\pi(\cdot | s) = \pi(\cdot | s')$.
We often rely on comparing functions
in this fashion.
For another example,
$v_{\pi}(\cdot; \mdp) \geq v_{\pi'}(\cdot; \mdp)$
denotes that the value function
of $\pi$
is greater or equal than that of $\pi'$
for all states in its domain,
that being some state set $\sset$.
If the MDP $\mdp$ is clear from context,
this can again be simplified to
$v_\pi \geq v_{\pi'}$.
Such comparisons are only defined when
the functions are defined over the same domain.

\subsection{Representation and Abstraction}%
\label{sub:representation_and_abstraction}
\subsubsection{State representation}
Representations
-- and in particular, state representations --
are inextricably linked with
problem descriptions in contemporary literature.
For our purposes,
it is important to define explicitly what a representation is
in order to critically compare different representation approaches
and to develop our own.
A representation
\emph{stands in}
for the real thing.
Put another way,
it is a proxy for the raw, real-world thing that it represents.
It follows that a \emph{state} representation
is a description for states,
standing in for the underlying
environment's configurations.
Whether manually defined
or somehow learned,
state representations
help enumerate large or continuous state sets.
Not only that,
state representations
can also provide a common language
to describe tasks that we might think of as related.
For example,
we can represent
the states of all Gridworlds
as coordinate tuples
(see Example~\ref{ex:representation}).

\begin{definition}[state representation]
  A \emph{state representation} $\rep$
  is a function
  mapping from a set of
  \emph{raw states}
  $\rawsset$
  to a state representation set $\repsset$
  \[\rep: \rawsset \mapsto {\repsset}\]
  We term
  ${\repsset}$
  the \emph{state set} associated with $\rep$.
  Given a raw state
  $x \in \rawsset$,
  we term $s \defas \rep(x)$
  the representation of $x$,
  overloading the meaning of the term \emph{representation}
  to mean both the function and a raw state's image.
  Where the relation with the representation $\rep$ is clear,
  we drop the subscript from
  $\repsset$
  and simply write $\sset$.
\end{definition}

Since we cannot access
the raw states $\rawsset$
of the true environment being modelled,
state representations are always used
when defining a MDP,
albeit usually implicitly.
It is also possible
to define further state representations,
possibly based on the existing ones.

It has become common to
make use of
the notion of structure of a state set,
for example
by employing a similar policy
in states that \emph{look similar}.
To facilitate this,
\emph{state spaces}
can be used in place of
state sets.
A state space is simply
a state set with added structure
(such as a metric).
The type of space
is dependent on the problem at hand
and is often assumed implicitly.
Metric spaces are common
as they employ a distance metric
which can be used to measure the distance between states.
For example,
if raw Gridworld states
were reprsented by coordinate tuples,
neighboring states would be closer
than non-neighboring states
under the Euclidean metric on the coordinate space.

\begin{definition}[state space]
  Given a state representation
  $\rep: \rawsset \mapsto {\sset}$,
  if $\sset$
  is a \emph{space},
  i.e. a set with some added structure,
  then
  we term it the \emph{state space} associated with $\rep$
  and denote it with a non-caligraphic
  $S$.
\end{definition}

\noindent
A state representation
can be as simple
as a scheme for enumerating states
such that each has a unique identifier
which one can use to point at
(to represent)
that state.
We refer to
a rudimentary representation
that uniquely identifies each state
as a \emph{ground representation}.
\begin{definition}[ground representation]
  A state representation
  $\rep_0: \rawsset \mapsto {\sset_{\rho_0}}$
  is a
  \emph{ground representation}
  if and only if
  $\rep_0$ is injective,
  that is,
  raw states that are distinct
  have distinct representations.
\end{definition}

\begin{example}
  \label{ex:representation}
  Take the simple Gridworld in
  \figref{fig:example_state_rep}.
  There are 5 states
  corresponding to the location of the agent.
  In order to refer to these states
  we can employ a representation.
  A simple enumeration is feasible:
  we can number the states
  in ascending order
  according to nearness to the goal location,
  with $s_1$ denoting the state depicted in
  \figref{fig:example_state_rep}
  (with the agent located in the upper left cell)
  and $s_5$ denoting the state where the agent
  occupies the goal location.
  This is a valid ground representation,
  as is every permutation
  --- we could have numbered them the other way around.

  \begin{figure}[H]
    \centering
    \includegraphics[width=0.25\linewidth]{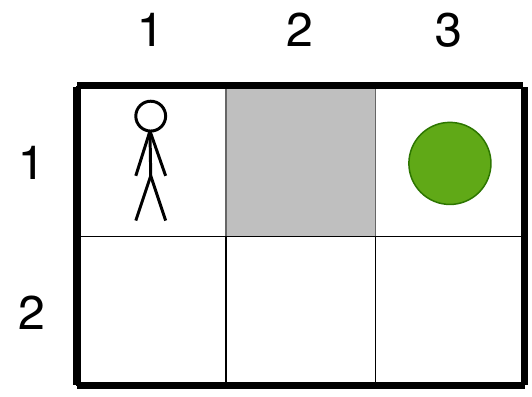}
    \caption{A simple Gridworld labelled to construct a representation.
    The shaded area is inaccessible to the agent
    and the goal cell with positive reward is indicated by the green circle.}
    \label{fig:example_state_rep}
  \end{figure}

  We have already established that multiple
  ground representations are possible for our simple Gridworld example,
  yet we could have followed another representation scheme altogether.
  Since the agent's locations are conveniently arranged in a grid,
  we could label each state with
  a $(x, y)$ coordinate pair.
  Employing the numbering
  in \figref{fig:example_state_rep},
  we can identify the following 5 states:
  $(1,1), (1,2), (2,2), (3,2), (3,1)$
  --
  another ground representation.
  We can summarise its signature as
  $\rep_0: \sset \mapsto \{1,2,3\} \times \{1, 2\}$.
  Some caution should be exercised with this description though:
  the codomain of $\rep_0$ does not span the entire image.
  That is, the state $(2,1)$ does not actually exist.
  Likewise we could have written
  $\rep_0: \sset \mapsto \naturalnumbers^2$,
  with the same caution.

  Given that states are laid out in a grid,
  it is tempting to use a metric
  that describes similarity or `closeness' of grid cells.
  Using the Euclidean metric for example,
  $(1,1)$ is closer to
  {$(1,3)$}
  than it is to
  {$(2,3)$}.
  Similarity between states under some metric
  is sometimes used as a basis to transfer behavior.
  It is far from obvious however
  that states that are close together
  in this example would benefit from similar behavior.
  Note, for instance, that $(2,3)$ is reachable from $(1,1)$ in fewer (non-diagonal) steps than is $(1,3)$ in spite of its greater Euclidean distance.
  \comment{(Now that I added this paragraph, I need to add a citation of an example where similarity between states is used for transfer)}.
\end{example}



\noindent
Many representations with distinct images
are ground representations,
yet they all fulfill the same purpose of uniquely labelling states.
For instance,
we could have easily devised a different numbering scheme
in the above example.
\emph{Isomorphism}
neatly captures this notion of similarity
and allows us to focus only on meaningful differences between representations.
\begin{definition}[isomorphic ground state representations]
  Given two ground state representations
  $\rep^\alpha: \rawsset \mapsto \sset_\alpha$
  and
  $\rep^\beta: \rawsset \mapsto \sset_\beta$,
  if there exists a bijection
  between
  $\sset_\alpha$
  and
  $\sset_\beta$
  (i.e. an isomorphism),
  then we say that
  $\rep_\alpha$
  and
  $\rep_\beta$
  are
  \emph{isomorphic}
  or
  \emph{equal up to an isomorphism}:
  \begin{equation*}
    \rep^\alpha \isomorph \rep^\beta
  \end{equation*}
\end{definition}

\noindent
It follows that every bijective ground representation
is isomorphic with the
identity state representation.

\subsubsection{State Abstraction}
Abstractions themselves are representations of a sort.
Yet
abstractions are more than that:
they {aim to} focus on just the
\emph{relevant information}
and ignore the irrelevant.
In the context of transfer,
good abstractions should focus
on relevant commonalities between tasks.
We see abstractions throughout this paper
as a molding of distinctions,
masking those that are less relevant.
To remove all that distinguishes two things
is to abstract away from them
and equate them as a result.

A state abstraction is a special form of state representation
which is potentially non-injective:
it supports
the mapping of distinct raw states onto the same abstract state.
%

\begin{definition}[state abstraction]
  Given an MDP
  $\mdp=\langle \sset, \aset, p, r \rangle$,
  a
  \emph{state abstraction}
  $\phi:~\sset~\mapsto~\asset$
  maps from a state representation set $\sset$
  to an abstract state set $\asset$.
  We term
  $\phi(s) \in \asset$
  the \emph{abstract state}
  corresponding to the
  \emph{ground state} $s \in \sset$.
  The inverse abstraction
  $\phi^{-1}:\asset \mapsto \powerset{(\sset)}$
  maps an abstract state
  $\phi(s) \in \asset$
  onto the set of corresponding ground states
  $\set{s' \in \sset | \phi(s') = \phi(s)}$
  and where
  $\powerset(X)$ denotes the powerset of $X$.
\end{definition}

Since there are many possible state abstractions for any given state set,
it is useful to again rely on the notion of isomorphism
between state abstractions to describe state abstractions
that are the same in every relevant aspect.
\begin{definition}[isomorphic state abstractions]
  Let
  $\phia: \sset \mapsto \Phia$
  and
  $\phib: \sset \mapsto \Phib$
  be two state abstractions.
  We say that
  $\phia$
  and
  $\phib$
  are
  \emph{isomorphic}
  or
  \emph{equal up to an isomorphism},
  denoted
  $
    \phia \isomorph \phib
  $,
  if all sets of ground states
  that correspond to the abstract states
  of $\phia$
  are the same for $\phib$.
  That is,
  if
  $
  \Set{ \phiai(\astate_\alpha) | \astate_\alpha \in \Phia }
   =
  \Set{ \phibi(\astate_\beta) | \astate_\beta \in \Phib }
  $.
\end{definition}

\subsubsection{Loss of the Markov Property.}
It is tempting to see a state abstraction
as inducing a new decision process
based on the abstract state set.
This can however be far from straightforward.
It is possible
that this \emph{abstract decision process}
no longer maintains the Markov property.

\begin{example}
  \begin{figure}[ht]
    \centering
    \begin{subfigure}[t]{.49\textwidth}
      \centering
      \includegraphics[scale=.6]{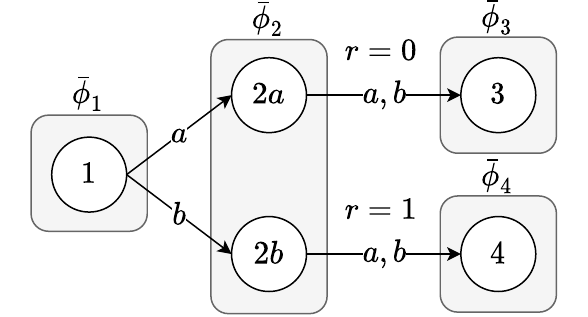}
      \caption{}
      \label{fig:nomdp_ground}
    \end{subfigure}
    \begin{subfigure}[t]{.49\textwidth}
      \centering
      \includegraphics[scale=.6]{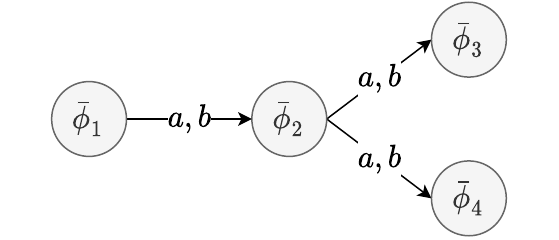}
      \caption{}
      \label{fig:nomdp_abstract}
    \end{subfigure}
    \caption{On the left hand side is depicted
    an MDP
    with 5 states and the action set $\{a,b\}$.
    Arcs between states denote
    deterministic transitions following the labelled action,
    receiving the annotated reward $r$.
    An abstraction is shown as grouping together states.
    The resulting abstract states are depicted on the right.}
    \label{fig:abstraction_nomdp}
  \end{figure}
  \figref{fig:nomdp_ground}
  depicts an MDP to which a state abstraction
  is applied.
  This simple state abstraction $\phi$
  groups states
  $2a$ and $2b$ together:
  $\phi(2a) = \phi(2b) = \astate_2$.
  The other states remain ungrouped under the abstraction,
  i.e.,
  each is mapped to a different abstract state.
  \figref{fig:nomdp_abstract} shows the resulting abstract
  states,
  along with transitions labelled with actions
  in an attempt at
  fidelity to the original MDP.
  There is a problem describing the
  resulting \emph{abstract}
  MDP however:
  it is impossible to express
  the transition probabilities from $\astate_2$
  unambiguously
  in such a way that it accurately reflects the transition
  probabilities of the ground MDP.
  The result of an action
  from $\astate_2$
  depends on the action
  that was taken before (at $\astate_1$).
  This is because $\astate_2$
  is the abstraction of the ground states
  $2a$ and $2b$.
  It is
  no longer possible to distinguish between
  $2a$ and $2b$,
  yet they lead into different states:
  $3$ and $4$ respectively,
  with abstract states
  $\astate_3$ and $\astate_4$ respectively.
  As a result,
  it is not clear whether a certain action
  from $\astate_2$ will lead to
  $\astate_3$ or $\astate_4$.
  Not without knowing what preceded.
  In short, the Markov property is violated.

  Consider for example
  $ p(\astate_3 \mid \astate_2, a) $
  which has two candidates.
  The resulting state of action $a$
  in $\astate_2$
  is either $\astate_3$
  or $\astate_4$ (in which case it is not $\astate_3$).
  The candidates depend on different histories over actions:
  \begin{align*}
      &\Pr(\astaterv_{t+1} = \astate_3 \mid \astaterv_t=\astate_2, A_t=a, \astaterv_{t-1}=\astate_1, A_{t-1}=a) = 1\\
    \neq\;
      &\Pr(\astaterv_{t+1} = \astate_3 \mid \astaterv_t=\astate_2, A_t=a, \astaterv_{t-1}=\astate_1, A_{t-1}=b) = 0
  \end{align*}
  That is,
  in order to know
  the outcome of action $a$
  from
  $\astate_2$,
  it is necessary to know
  which action
  was taken from state $\astate_1$.
  Similarly,
  it has also become impossible to describe the expected reward function
    $r(\astate_2, a)$:
  \begin{align*}
    &\expected\left[R_{t+1} \mid \astaterv_t=\astate_2, A_t=a, \astaterv_{t-1}=\astate_1, A_{t-1}=a\right] = 0 \\
    \neq\;
    &\expected\left[R_{t+1} \mid \astaterv_t=\astate_2, A_t=a, \astaterv_{t-1}=\astate_1, A_{t-1}=b\right] = 1
  \end{align*}
  Both transition and reward dynamics
  have become dependent on the history,
  making it impossible to construct
  an abstract Markov Decision Process
  from this state abstraction
  that is faithful to the original MDP.
\end{example}

\paragraph{}
As the example shows,
a state abstraction does not necessarily lead to an obvious choice
for an abstract MDP.
Instead, it
results in a range of possible abstract MDPs,
each a candidate that we can learn a policy for.
We will delve into ways of assessing different state abstractions
and abstract MDPs after we define some basic machinery.

\begin{definition}[abstract MDP]
  Given an MDP
  $\mdp = \langle \sset, \aset, p, r \rangle$
  and state abstraction
  $\phi: \sset \mapsto \acodomain$,
  we say that
  $\mdp_\phi = \langle \Phi, \aset, p_\phi, r_\phi \rangle$
  is an
  \emph{abstract MDP}
  for $\mdp$ and $\phi$
  if and only if
  there exists a function
  $w_\phi: \sset \mapsto [0, 1]$
  such that for all
  $\astate, \astate' \in \asset$,
  it holds that
  $\sum_{s\in\astate}w_{\phi}(s)=1$,
  and
\begin{align}
  p_\phi(\astate' | \astate, a) &=
  \sum_{s \in \phii(\astate)}
  \sum_{s' \in \phii(\astate')}
    w_{\phi}(s) p(s' | s, a)
  \label{eq:reformulated_transition}
  \\
  r_\phi(\astate, a) &=
  \sum_{s \in \phii(\astate)}
    w_{\phi}(s) r(s, a)
  \label{eq:reformulated_reward}%
\end{align}
By reuse of notation we may write the set of abstract MDPs
for $\mdp$ and $\phi$
as $\phi(\mdp)$, such that
$\mdp_\phi \in \phi(\mdp)$
and
$\phi^{-1}(\mdp_\phi) = \mdp$.
\end{definition}

\paragraph{}
As a result of the freedom permitted by $w_{\phi}(s)$,
many abstract MDPs are possible
for a particular MDP and state abstraction.
The abstract transition and reward dynamics
are constructed as weighted averages over the
original transition and reward dynamics.
The weighting $w_{\phi}(s)$
is the mechanism through which a new valid MDP
is created,
albeit only an approximation of the original.
The cost of state abstraction
is a loss in fidelity to the original MDP.
In this paper we will discuss particular classes
of state abstractions
that result in abstract MDPs
where the freedom permitted by $w_\phi$
does not play a role and
that are particularly useful for
transfer.

\subsubsection{Comparing Abstractions}
State abstractions operating on the same state set
can be compared to one another.
The first state abstraction property that allows comparison
is \emph{coarseness}
(or conversely \emph{fineness}),
where one state abstraction is finer than another
if it makes more distinctions between states.
The following definition of \emph{coarseness}
is adapted from
the work of
\citet{Li2006}
to take into account isomorphism.

\begin{definition}[coarseness]
  Let
  $\phi_1: \sset \mapsto \Phi$
  and
  $\phi_2: \sset \mapsto \Phi'$
  be two state abstractions.
  We say that
  $\phi_1$ is \emph{finer} than
  $\phi_2$,
  or equivalently that
  $\phi_2$ is \emph{coarser} than
  $\phi_1$,
  written $\phi_1 \finereq \phi2$,
  if
  all states that are equivalent under $\phi_1$
  are also equivalent under $\phi_2$.
  Formally:
  \[\phi_1 \finereq \phi_2 \quad \text{iff} \quad \forall s, s' \in \sset: \phi_1(s) = \phi_1(s') \Rightarrow \phi_2(s) = \phi_2(s') \]


  We say \emph{strictly finer} ($\finer$)
  or
  \emph{strictly coarser} ($\coarser$)
  if in addition
  $\phi_1$ and $\phi_2$
  are not isomorphic.
  We also refer to the
  \emph{coarseness}
  of an abstraction
  as its
  \emph{expressivity};
  an abstraction that is
  \emph{finer}
  is
  \emph{more expressive}.
\end{definition}
The
\emph{coarser} relation
imposes a partial ordering
over a set of state abstractions
$\abstractionsset_\mdp$
for an MDP $\mdp$.
If there is no state abstraction
$\phi' \in \abstractionsset_\mdp$
that is \emph{finer} (coarser)
than $\phi \in \abstractionsset_\mdp$,
we say that $\phi$ is a
\emph{maximally fine} (maximally coarse) state abstraction in $\abstractionsset_\mdp$.
If a state abstraction
$\phi \in \abstractionsset_\mdp$
is \emph{finer} (coarser)
than all other state abstractions
$\abstractionsset_\mdp$,
we say that
$\phi$
is
\emph{finest} (coarsest).
The finest abstraction
for $\abstractionsset_\mdp$
is isomorphic with the identity state abstraction.
The coarsest abstraction
for $\abstractionsset_\mdp$,
on the other hand,
maps every state
onto a single abstract state,
equating all states and
making no distinctions.
It follows
that there is only one isomorphically distinct
finest and coarsest abstraction
for any MDP.

\subsection{Transfer Through State Abstraction}%
\label{sub:transfer}
The power of state abstractions
becomes apparent when multiple ground states
are mapped to the same abstract state.
When a behavior is learned for a single abstract state,
it can apply automatically for
all ground states for this abstract state.
It now becomes useful
to see a state abstraction
as inducing an equivalence relation,
where two states are equivalent
when they
share an abstraction.
An equivalence relation
$\equi \subseteq \sset \times \sset$
is a binary relation
that is reflexive, symmetric and transitive.
If
a state pair
$(s,s')$
is in $\equi$,
we say that
$s$ and $s'$ are equivalent.
We employ the shorthand
$\equi(s,s')$
to denote
$(s,s') \in \equi$.
The view of
state abstractions as
inducing equivalence relations
is due to
\citet{Givan2003},
the language and what follows is our own.

\begin{definition}[abstraction-induced equivalence relation]
  A state abstraction
  $\phi$
  induces an equivalence relation
  $\equi_\phi$
  over a state set.
  If $\equi_\phi(s,s')$
  we write $s \aequiv s'$.
  For states $s, s' \in \sset$
  this relation is defined as
  \begin{equation}
    s \aequiv s'
    \iff
    \phi(s) = \phi(s')
    \label{eq:abstraction_equivalence_relation}
  \end{equation}
  If $s \aequiv s'$
  we say that
  \emph{$s$ is equivalent to $s'$ under the abstraction $\phi$}.
\end{definition}

An equivalence relation in turn
results in a partitioning of the state set.
States that
share the same abstraction
belong to the same
\emph{equivalence class}.
For an equivalence relation induced by
a state abstraction $\phi$,
we can denote the equivalence class
of a ground state $s$
under $\phi$
as
$[s]_\phi \defas \Set{ s' \in \sset | s \aequiv s' }$.
For convenience of notation
we overload $\phi(s)$,
so far used to denote an abstract state,
to also denote the equivalence class of $s$ under $\phi$,
i.e.
$\phi(s) \defas [s]_\phi$.

An equivalence relation
$\equi$ on $\sset$
induces a \emph{partition}
$P = \Set{ [s_1], \dots, [s_m] }$.
For an abstraction-induced equivalence relation
we can write the partition as
$\Phi = \Set{ \astate | \astate \defas \phi(s) \; \forall s \in \sset }$.
In addition,
we have
$\bigcup_{\astate} = \sset$
and
$\bigcap_{\astate} = \emptyset$.
It follows that
if there is at least one
equivalence class with more than one member,
i.e. if two states share an abstraction,
then
the partition will be smaller than the original state set.

\subsubsection{Single-Task Transfer.}
The view of a state abstraction
as inducing an equivalence relation
between ground states
is useful
because it gives us a way to transfer behavior
learned for an abstract state
to all ground states
in the corresponding equivalence class.
Where before
behavior had to be learned for each ground state separately,
now we only need to learn a policy
for the abstract states.
Since the abstract MDP has fewer states to contend with,
it can be easier to learn than the ground MDP.

\paragraph{}
If we have an MDP
$\mdp$
and a state abstraction $\phi$
which induces the abstract MDP $\mdp_\phi$,
then we can learn a policy
$\pi_\phi : \sset_\phi \mapsto \probset(\aset)$
for this abstract MDP.
To transfer
this \emph{abstract policy}
to the ground task,
we make use of
a process called
\emph{policy derivation}.

\begin{definition}[derived policy]
  Let
  $\mdp =\langle \sset, \aset, p, r \rangle$
  be an MDP,
  let
  $\phi: \sset \mapsto \Phi$
  a state abstraction,
  and
  let
  $\mdp_\phi=\langle \Phi, \aset, p_\phi, r_\phi \rangle$
  be
  an abstract MDP constructed
  from $\mdp$ and $\phi$.
  Given an abstract policy
  $\pi_\phi: \Phi \mapsto \probset(\aset)$,
  we define
  the \emph{policy derived from} $\pi_\phi$ as
  $\derived: \sset \mapsto \probset(\aset)$ s.t.
  $\forall s \in \sset, a \in \aset$:
  \begin{equation*}
    \derived (a \mid s) =
    \pi_\phi(a \mid \phi(s))
  \end{equation*}
  Note that we use the symbol $\delta$
  to denote a derived policy
  rather than the customary $\pi$
  to suggest that it is derived
  from an abstract policy.

  Let
  $\Pi_\phi$
  be the set of all abstract policies
  $\pi_\phi: \Phi \mapsto \probset(\aset)$.
  We write the
  \emph{set of derived policies}
  for state abstraction $\phi$
  and MDP $\mdp$
  as
  \begin{equation*}
    \Pi_\delta(\phi, \mdp)
    \defas
    \Set{
      \derived
      |
      \derived(s, a) = \pi_\phi(\phi(s), a)
      \quad \forall \pi_\phi \in \Pi_\phi
    }
  \end{equation*}
  If $\mdp$ is clear from the context,
  we also refer to $\Pi_\delta(\phi, \mdp)$
	as
  \emph{the derived policies of $\phi$}.
\end{definition}

\paragraph{}
Policy derivation gives us a simple construction
to translate an abstract policy to a policy for a ground MDP.
It is interesting to investigate
what manner of ground policies can be derived
given a particular state abstraction,
even before considering any abstract MDPs.
As we shall see shortly,
state abstraction is restrictive.
Only the most trivial of state abstractions
allow access to the whole set
of ground policies by means of derivation.
This means that different state abstractions
will result in different sets of derived policies.
Some sets of derived policies
might not even include the optimal policy
because of restriction by the state abstraction.
We start with one straightforward result
that describes how state abstraction
limits possible derived policies.
After that we will consider
whether this limitation is harmful
to the end goal:
attaining good policies.

\paragraph{}
When a state abstraction
assigns two ground states to the same abstract state,
any abstract policy can no longer distinguish between those two states
as it only sees the combined abstract state.
As a result, any derived policy
is also unable to distinguish between these two states,
effectively decreasing \emph{controllability}
or \emph{control granularity}.
Increased control is only possible
by distinguishing between states.

\begin{proposition}
  \label{thm:abstraction_control}
    Let
  $\mdp =\langle \sset, \aset, p, r \rangle$
  be an MDP
  and let
  $\phi: \sset \mapsto \Phi$
  be a state abstraction
  that maps two states
  $s, t \in \sset$
  to the same equivalence class $\astate$.
  Then all derived policies for $\phi$
  must have the same distributions over actions
  for both $s$ and $t$.
  Formally,
  $
      \forall s, t \in \sset:
  $
  \begin{equation}
      \phi(s) = \phi(t)
      \;\Leftrightarrow\;
      \forall \delta \in \Pi_\delta(\phi, \mdp):
      \delta(\cdot\mid s) = \delta(\cdot\mid t)
      \label{eq:thm_state_abstraction_control}
  \end{equation}

\end{proposition}

\seeproof{abstraction_control}

\paragraph{}
Since different state abstractions
result in different restrictions on possible derived policies,
the level of control or
\emph{control granularity}
is a useful property to characterise
and compare between state abstractions.

\begin{definition}[control granularity]
  Let
  $\mdp = \langle \sset, \aset, p, r \rangle$
  be an MDP
  and let
  $\phi: \sset \mapsto \Phi$
  and
  $\phi': \sset \mapsto \Phi'$
  be two state abstractions.
  We say
  that $\phi$
  is
  \emph{more control granular}
  than $\phi'$
  if all derived policies contained in $\Pi_\delta(\phi', \mdp)$
  are also contained in
  $\Pi_\delta(\phi', \mdp)$.
  Formally:
\begin{equation*}
\Pi_\delta(\phi', \mdp) \subseteq \Pi_\delta(\phi, \mdp)
\;\Longleftrightarrow\;
\forall \delta' \in \Pi_\delta(\phi', \mdp):
\exists \delta \in \Pi_\delta(\phi, \mdp):
\delta = \delta'
\end{equation*}
If in addition
we have that
$\Pi_\delta(\phi, \mdp) \neq \Pi_\delta(\phi', \mdp)$,
we say that
$\phi$
is
\emph{strictly more control granular}
than $\phi'$.
\end{definition}

State abstraction control granularity
sounds a lot like state abstraction fineness or coarseness.
It turns out both properties
are different sides of the same coin.
Coarser state abstractions give up more control;
finer state abstractions have more control, \emph{are more control granular}.
The reason we introduced the secondary notion of control granularity
is that it is more amenable to the context of control
that we are faced with.

\begin{proposition}[control granularity--coarseness trade-off]
  \label{thm:control_coarseness_tradeoff}
    Let
  $\mdp =\langle \sset, \aset, p, r \rangle$
  be an MDP
  and let
  $\phi: \sset \mapsto \Phi$
  and
  $\phi': \sset \mapsto \Phi'$
  be two state abstractions.
  Then
  $\phi$
  is
  \emph{finer}
  than $\phi$
  if and only if
  $\phi$
  is more control granular than $\phi'$.
  Conversely,
  if there exists a derived policy for $\phi$
  that cannot be derived for $\phi'$
  then $\phi$
  must be finer,
  meaning
  that there must be two states
  $s,t \in \sset$
  that are mapped to different abstract states
  by $\phi$
  but to a single abstract state
  by $\phi'$.
  \comment{}
  Formally:
  \begin{equation}
  \begin{gathered}
\phi \finereq \phi'
\;\Longleftrightarrow\;
\Pi_\delta(\phi, \mdp)
\supseteq
\Pi_\delta(\phi', \mdp)
  \end{gathered}
  \label{eq:thm_derivability}
  \end{equation}

\end{proposition}

\seeproof{control_coarseness_tradeoff}

\paragraph{}
Potential for transfer
through state abstraction
and policy derivation
does not necessarily mean
that it is beneficial
--
the resulting derived policy
is not necessarily a good one.
For example,
a policy derived from an optimal
abstract policy
is not necessarily optimal in the ground MDP.
Now that we know that the set of derived policies
$\Pi_\delta(\phi, \mdp)$
is a subset of the set
of all ground policies,
it pays to look at the value functions
of those derived policies.
This will tell us
what solution qualities
are excluded by a state abstraction,
impossible to be attained
regardless of the learning algorithm
that might eventually be used,
constrained as it must be by
the control granularity permitted by the state abstraction.
What this will come down to eventually
is the question of how
a state abstraction's
potential for transfer
and
the potential solution quality permitted by that state abstraction
are related.

\paragraph{}
Consider the value functions
associated with the derived policies
$\Pi_\delta(\phi, \aset)$.
These value functions
constitute a partial order without the guaranteed existence of
a greatest element
as there can be multiple policies that are optimal for the set of derived policies.
Still, this set of value functions can be useful
as we can compare sets of value functions of derived policies
between state abstractions.
If one state abstraction allows a derived policy
that trumps all derived policies of another state abstraction,
then that state abstraction
is potentially better.

\begin{definition}[transfer value]
  Let
  $\mdp = \langle \sset, \aset, p, r, \gamma \rangle$
  be an MDP
  and let
  $\phi: \sset \mapsto \Phi$
  and
  $\phi': \sset \mapsto \Phi'$
  be two state abstractions.
  We say that
  $\phi$
  has \emph{greater transfer value for $\mdp$}
  than
  $\phi'$,
  denoted
  $v_\phi(\mdp) \geq v_{\phi'}(\mdp)$,
  if
  there exists a derived policy for
  $\phi$
  with a value function greater than or equal
  all derived policies' value functions
  for $\phi'$.
  Formally:
  \begin{equation*}
    v_\phi(\mdp) \geq v_{\phi'}(\mdp)
    \;\Longleftrightarrow\;
    \exists \delta \in \Pi_\delta(\phi, \mdp):
    \forall \delta' \in \Pi_\delta(\phi', \mdp): v_\delta(\cdot; \mdp) \geq v_{\delta'}(\cdot; \mdp)
  \end{equation*}
  We say that
  $\phi$ has \emph{strictly greater transfer value}
  than $\phi'$,
written
$v_\phi(\mdp) > v_\phi'(\mdp)$,
if there is a derived policy in $\Pi_\delta(\phi, \mdp)$
that has a strictly greater value function
than any derived policy in $\Pi_\delta(\phi', \mdp)$.
\end{definition}

We should add that
comparing transfer values
can be conservative
for some applications
as the above definition
requires value dominance at all states.
In practice
there is often a set of start states for which
a higher value is far more important than for other states.
In those scenarios,
transfer values
could be compared for a subset of states
rather than for all states.
In this paper we limit ourselves to comparing
values for all states but limiting this comparison to a subset of states
is a straightforward extension.

\paragraph{}
The set of all possible ground policies
may contain multiple optimal policies
but they will all have the same greatest value function.
The same is not necessarily true for the set of derived policies.
Because policy derivation restricts the space of available policies,
the optimal policy might no longer be available
which will be reflected in the transfer value of
that state abstraction.
Instead of a single optimal derived policy,
it is possible to have
multiple derived policies
with maximal value,
all suboptimal,
all with value functions that are incomparable
(neither greater nor less).
Since transfer value is defined in terms of value functions,
it is also possible
for one state abstraction's
transfer value
to not be comparable to that of another state abstraction.

\paragraph{}
In effect,
we now have two partial orderings over
state abstractions:
coarseness and control granularity on one side and transfer value on the other.
The two sides are of course related.
To describe how they relate,
it is instructive
to understand when one state abstraction
has \emph{a higher transfer value than} another.
It is only possible for a state abstraction
to have a strictly higher transfer value than another
if it allows for more control where the other does not \
(and for that improved control to support a greater value at one or more states).
If two states are grouped together under an abstraction
that an optimal policy would distinguish between,
the value of the best available policy under this abstraction suffers
and this is reflected in the transfer value of this abstraction.
This leads to the following Theorem.

\begin{theorem}[{transfer value--control trade-off}]
  \label{thm:transfer_value_control_tradeoff}
    Let
  $\phi: \sset \mapsto \Phi$
  and
  $\phi': \sset \mapsto \Phi'$
  be two state abstractions
  such that
  $\phi$
  has
  \emph{strictly higher transfer value than}
  $\phi'$
  for some MDP
  $\mdp \defas \langle \sset, \aset, p, r \rangle, \gamma$.
  Then there must
  exist two states
  $s, t \in \sset$
  such that
  $\phi'$ does not distinguish between them
  while $\phi$ does,
  and there exists a derived policy for
  $\phi$
  that takes advantage of this distinction
  by controlling $s$ and $t$ differently
  in a way that no derived policy for $\phi'$ can do.
  By giving up representational capacity
  and increasing potential for transfer,
  $\phi'$ sacrifices transfer value.
  Formally:
  \begin{gather*}
    v_\phi(\mdp) > v_{\phi'}(\mdp)
    \\
    \Downarrow \\
      \exists s, t \in \sset:\\
    \begin{aligned}
      &1.\; \phi(s) \neq \phi(t) \\
      &2.\; \phi'(s) = \phi'(t) \\
      &3.\; \exists \delta \in \Pi_\delta(\phi, \mdp):
      \delta(\cdot\mid s) \neq \delta(\cdot\mid t)
      \\
      &4.\; \forall \delta' \in \Pi_\delta(\phi', \mdp):
      \delta'(\cdot\mid s) = \delta'(\cdot\mid t)
    \end{aligned}
  \end{gather*}

\end{theorem}

Note that $\phi$ need not be
\emph{finer}
than $\phi'$,
so there can still exist two states
$\bar{s}, \bar{t} \in \sset$
where
$\phi(\bar{s}) = \phi(\bar{t})$
and
$\phi'(\bar{s}) \neq \phi'(\bar{t})$
(to contrast with the consequent above).
This result just means that there is at least one
pair of states where the opposite is true.

\paragraph{}
\seeproof{transfer_value_control_tradeoff}

\paragraph{}
As a direct result of
Theorem~\ref{thm:transfer_value_control_tradeoff}
we also have the following corollary.
While it is weaker than the theorem
it does enjoy some brevity.

\begin{corollary}
  \label{thm:transfer_value_control_tradeoff_corollary}
    Let
  $\phi: \sset \mapsto \Phi$
  and
  $\phi': \sset \mapsto \Phi'$
  be two state abstractions
  such that
  $\phi$
  has
  \emph{strictly higher transfer value than}
  $\phi'$
  for some MDP $\mdp$.
  Then
  $\phi'$ can both be \emph{no finer}
  than $\phi$
  and \emph{no more control granular}
  than $\phi$.
  Formally:
  \begin{equation}
    v_\phi(\mdp) > v_{\phi'}(\mdp)
    \Longrightarrow
    \phi' \not \finereq \phi
    \;\wedge\;
    \Pi_\delta(\phi, \mdp)
    \not \subseteq \Pi_\delta(\phi', \mdp)
    \label{eq:transfer_value_control_tradeoff_corollary1}
  \end{equation}
  The following also holds.
  If $\phi$ is coarser
  and less control granular than $\phi'$,
  it is impossible for
  $\phi$
  to have a higher transfer value than
  $\phi'$.
  Formally:
  \begin{equation}
    \phi' \finereq \phi
    \;\wedge\;
    \Pi_\delta(\phi, \mdp)
    \subseteq \Pi_\delta(\phi', \mdp)
    \Longrightarrow
    v_\phi(\mdp) \not > v_{\phi'}(\mdp)
    \label{eq:transfer_value_control_tradeoff_corollary2}
  \end{equation}

\end{corollary}

\seeproof{transfer_value_control_tradeoff_corollary}

\paragraph{}
We have established a general trade-off
between transfer value on one hand
and coarseness or control granularity
on the other hand.
However,
it's not as simple
as every coarsening resulting
in decreased transfer value and vice versa.
It is sometimes possible to coarsen a state abstraction
without sacrificing transfer value.
For example,
the identity state abstraction
has maximal transfer value,
yet there can be state abstractions
with the same maximal transfer value
yet that are coarser
that the identity state abstraction.
This presents an interesting opportunity,
where potential for transfer can be increased
with no drawback.

\begin{definition}[minimal control state abstraction]
  A state abstraction
  $\phi: \sset \mapsto \Phi$
  is a
  \emph{minimal control state abstraction}
  for an MDP
  $\mdp=\langle \sset, \aset, p, r, \gamma\rangle$
  if $\phi$
  has \emph{maximal transfer value}
  and there exists \emph{no state abstraction}
  coarser than $\phi$
  with higher transfer
  value than $\phi$.
\end{definition}

In other words,
there is no coarsening
a \emph{minimal control state abstraction}
without negatively impacting transfer value
because that would prevent the optimal policy from being derived.
Still,
there is the potential for \emph{some} increased coarseness
compared to the identity state abstraction
where distinguishing between
states does not impact the solution quality.




\begin{center}
  $\ast$~$\ast$~$\ast$
\end{center}

\subsubsection{Multi-Task Transfer.}
Now that we have considered transfer within a single task,
a natural extension would be to consider two tasks
$\mdp_\alpha=\langle \sset_\alpha, \aset, p_\alpha, r_\alpha \rangle$
and
$\mdp_\beta=\langle \sset_\beta, \aset, p_\beta, r_\beta \rangle$
with a shared action set $\aset$.
We can also consider two state abstractions
$\phi_\alpha: \sset_\alpha \mapsto \asset_\alpha$
and
$\phi_\beta: \sset_\beta \mapsto \asset_\beta$
such that
$\asset_\alpha \cup \asset_\beta \subseteq \acodomain$.
While the images of
$\phi_\alpha$
and
$\phi_\beta$,
i.e.
$\asset_{\alpha}$
and
$\asset_{\beta}$,
are not necessarily the same,
they might have a non-empty intersection.
Both state abstractions
define equivalence relations on their respective domains,
as discussed in the previous single-task setting.
On top of that
we can also consider
an equivalence relation
relating states from the original sets
$\sset_\alpha$
and
$\sset_\beta$
wherever they share an abstraction.
We term this equivalence relation
$\equi_{\phi_\alpha,\phi_\beta}$
over
$\sseta \cup \ssetb$
\emph{the equivalence relation imposed by $\phia$ and $\phib$}
and,
if we take
$\phi(s)$ to be
$\phi_\alpha(s)$ if $s \in \sset_\alpha$
and $\phi_\beta(s)$ otherwise,
define it
for states
$
s, s' \in \sset_\alpha \cup \sset_\beta
$
as:
\begin{equation}
  E_{\phi_\alpha,\phi_\beta}(s, s')
  \iff
  \phi(s) = \phi(s')
\end{equation}
This is a generalization of the single-task setting
in Eq.~\eqref{eq:abstraction_equivalence_relation}.
If we now learn a policy
$\pi_{\phi_\alpha}: \sset_{\phi_\alpha} \mapsto~\probset(\aset)$
for
the abstract MDP $\mdp_{\phi_\alpha}$,
it can be readily transferred to the ground task
$\mdp_\alpha$
using policy derivation.
Moreover,
if
$\Phi_\alpha$
and
$\Phi_\beta$ intersect,
there will be states
$s_\alpha \in \sset_\alpha$
and
$s_\beta \in \sset_\beta$
so that
$\phi(s_\alpha) = \phi(s_\beta)$.
For these states
$s_\beta$,
it is possible to again
construct a policy
by setting
$\pi_\beta(\cdot|s_\beta) = \pi_{\phi_\alpha}(\cdot|\phib(s_\beta))$.
This may not cover
the whole state set
$\sset_\beta$%
\textemdash
a policy
$\pi_\beta$
so constructed
would not necessarily constitute a full solution
to $\mdp_\beta$%
\textemdash
but it could still improve
learning in $\mdp_\beta$
by bootstrapping with a partial solution.
In this setting we refer to
$\mdp_\beta$ as the \emph{target task},
and to $\mdp_\alpha$ as the \emph{source task}.
The part of the target task's state set
that is covered by the overlap in state abstractions
we refer to as the \emph{transfer cover}.

\begin{figure}[h]
  \centering
  \includegraphics[width=0.4\linewidth]{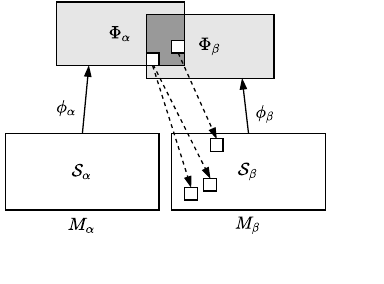}
  \caption{Diagram detailing the multi-task transfer setting.
  Rectangles depict state sets
  while squares denote states within those sets.
  The dark gray overlap denotes
  $\Phia \cap \Phib$
  and the dotted arrows
  starting from it point at ground states
  in the \emph{transfer cover}
  of $\phib$ given $\phia$.}
  \label{fig:multitask_setting}
\end{figure}

\begin{definition}[transfer cover]
  Let
  $\phi_\alpha: \sset_\alpha \mapsto \asset_\alpha$
  and
  $\phi_\beta: \sset_\beta \mapsto \asset_\beta$
  be two state abstractions.
  We refer to the set of
  all ground states
  $s_\beta \in \sset_\beta$
  with an abstraction in
  the intersection
  of the two abstractions
  $\Phi_\alpha \cap \Phi_\beta$
  as
  \emph{the transfer cover of $\phi_\beta$
  given
  $\phi_\alpha$}
  and denote it as
  \begin{align*}
    \phi^{-1}_\beta(\Phi_\alpha)
    &\defas
    \Set{
      s \in \phi^{-1}_\beta(\astate) | \astate \in \Phi_\alpha
    }
  \end{align*}
  This also implies a partial ordering
  over pairs of state abstractions.
  Let
  $\phi_\beta': \sset_\beta \mapsto \asset_\beta'$
  be another state abstraction.
  Given
  $\phi_\alpha$,
  $\phi_\beta$
  has greater transfer cover
  than
  $\phi_\beta'$
  if the transfer cover of $\phi_\beta'$
  is a subset of that of $\phi_\beta$.
  That is, if:
  \begin{equation*}
    {\phi_\beta'}^{-1}(\Phi_\alpha)
    \subseteq
    {\phi_\beta}^{-1}(\Phi_\alpha)
  \end{equation*}
\end{definition}

\paragraph{}
Now that we have identified the overlap between state abstractions
as an opportunity for transfer,
we can revisit \emph{policy derivation}
for the multi-task setting.

\begin{definition}[partially derived policy]
  Let
  $\mdp_\alpha=\langle \sset_\alpha, \aset, p_\alpha, r_\alpha \rangle$
  and
  $\mdp_\beta=\langle \sset_\beta, \aset, p_\beta, r_\beta \rangle$
  be two MDPs,
  let
  $\phi_\alpha: \sset_\alpha \mapsto \asset_\alpha$
  and
  $\phi_\beta: \sset_\beta \mapsto \asset_\beta$
  be two state abstractions
  such that
  $\asset_\alpha \cup \asset_\beta \subseteq$,
  and let
  $\pi_{\phi_\alpha}: \asset_\alpha \mapsto \probset(\aset)$
  be an abstract policy.
  A policy
  $\derived a\beta: \sset_\beta \mapsto \probset(\aset)$
  is
  \emph{partially derived}
  if it holds that
  $\forall s_\beta \in \phi_{\beta}^{-1}(\Phi_\alpha):$
  \begin{equation*}
    \derived a\beta(a | s_\beta) = \pi_{\phi_\alpha}(a | \phib(s_\beta))
  \end{equation*}
  Let
  $\Pi_{\phi_\alpha}$
  be the set of all abstract policies
  $\pi_{\phi_\alpha}: \Phi_\alpha \mapsto \probset(\aset)$.
  We write the
  \emph{set of partially derived policies
  from state abstraction
$\phi_\alpha$}
  for state abstraction $\phi_\beta$
  and MDP $\mdp_\beta$
  as
  \begin{equation*}
    \Pi_\delta(\phi_\beta, \mdp_\beta; \phi_\alpha)
    \defas
    \Set{
      \derived a\beta \in \Pi_\delta(\phi_\beta, M)
      |
        \exists \pi_{\phi_\alpha} \in \Pi_{\phi_\alpha}:
        \forall s_\beta \in \phi^{-1}_\beta(\Phi_\alpha):
        \derived a\beta(\cdot | s_\beta) = \pi_{\phi_\alpha}(\cdot | \phi_\beta(s_\beta))
    }
  \end{equation*}
  We can also consider the subset of partially derived policies
  for a single abstract policy.
  We write the
  \emph{set of partially derived policies from
    abstract policy
  $\pi_\phia$}
  for state abstraction $\phib$ and MDP $\mdpb$ as
  \begin{equation*}
    \Pi_\delta(\phi_\beta, \mdp_\beta; \pi_\phia)
    \defas
    \Set{
      \derived a\beta \in \Pi_\delta(\phi_\beta, M)
      |
        \forall s_\beta \in \phi^{-1}_\beta(\Phi_\alpha):
        \derived a\beta(\cdot | s_\beta) = \pi_{\phi_\alpha}(\cdot | \phi_\beta(s_\beta))
    }
  \end{equation*}
\end{definition}
It is straightforward to see that
single-task
\emph{derived policies}
form a special case
of cross-task
\emph{partially derived policies}
by setting
$\mdp_\alpha = \mdp_\beta$
and
$\phi_\alpha = \phi_\beta$.
The single-task setting is in effect a special case
of the multi-task setting
with full overlap between the
abstract codomains
resulting in the partially derived policy
making up a complete policy.

\paragraph{}
Whenever
$\Phib \not \subseteq \Phia$,
that is,
when $\phi_\beta$
does not have maximal transfer cover
given $\phi_\alpha$,
the above partially derived policy
construction does not constitute
a full policy
for all states in $\asset_\beta$,
hence the
\emph{partial}
in the name.
Since the region of the state space
described by the transfer cover
is the only place where transfer can occur,
we should want it to be as great as possible.
At the same time,
there exists an interesting interplay
between transfer cover
and concepts we have discussed before in the single-task setting:
control granularity and transfer value.
We will now re-visit
these concepts for the multi-task setting,
taking into account their interactions with transfer cover.

\begin{definition}[conditional coarseness]
  Let
  $\phi_\alpha: \sset_\alpha \mapsto \asset_\alpha$,
  $\phi_\beta: \sset_\beta \mapsto \asset_\beta$,
  $\phi_\beta': \sset_\beta \mapsto \asset'_\beta$
  be state abstractions
  such that
  $\asset_\alpha \cup \asset_\beta\cup \asset'_\beta$.
  We say that
  $\phi_\beta$
  is \emph{finer
  given
$\phi_\alpha$}
  than
  $\phi_\beta'$,
  denoted
  $
  \phi_\beta \underset{\phi_\alpha}{\finereq} \phi_\beta'
  $,
  if all states in the transfer cover of $\phi_\beta$ given $\phi_\alpha$
  that are equivalent under $\phi_\beta$
  are also in the transfer cover of $\phi_\beta'$ given $\phi_\alpha$
  and are also equivalent under $\phi_\beta'$.
 Formally:
  \begin{equation*}
    \begin{gathered}
      \phi_\beta \underset{\phi_\alpha}{\finereq} \phi_\beta'
      \\
      \Updownarrow
      \\
      \forall s, t \in \sset_\beta:
      \phi_\beta(s) = \phi_\beta(t) \wedge s, t \in \phi_\beta^{-1}(\Phi_\alpha)
      \;\Longrightarrow\;
      \phi_\beta'(s) = \phi_\beta'(t) \wedge s, t \in {\phi_\beta'}^{-1}(\Phi_\alpha)
    \end{gathered}
  \end{equation*}
  If, on the other hand,
  it holds that
  $
    \phi_\beta \underset{\phi_\alpha}{\coarsereq} \phi_\beta'
  $ (shorthand for $\phi'_\beta \underset{\phi_\alpha}{\finereq} \phi_\beta$),
  we say that $\phi_\beta$ is
  \emph{coarser given $\phi_\alpha$} than $\phi_\beta'$.
  If, in addition,
  $\phi_\beta$ is not isomorphic to $\phi_\beta'$,
  we say
  that $\phi_\beta$ is \emph{strictly finer}
  (resp. \emph{strictly coarser})
  than $\phi_\beta'$.
\end{definition}

The difference with the single state abstraction version
we used before
is that this definition of coarseness
takes into account
the transfer cover.
It is possible for one state abstraction to be strictly finer than another
without being strictly finer with respect to
a given state abstraction $\phi_\alpha$,
simply because only states in the transfer cover are relevant
for \emph{conditional coarseness}.
In effect,
in our transfer setting,
part of the abstract state set is left outside of consideration.
Specifically, the part of $\Phi_\beta$ that does not overlap with
$\Phi_\alpha$,
i.e. $\Phi_\beta - \Phi_\alpha$.

\begin{figure}[h]
  \centering
    \begin{subfigure}[t]{.49\textwidth}
      \centering
      \includegraphics[scale=.8]{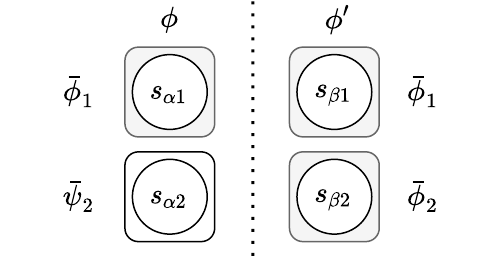}
      \caption{}
      \label{fig:conditional_coarseness_vs_transfer1}
    \end{subfigure}
    \begin{subfigure}[t]{.49\textwidth}
      \centering
      \includegraphics[scale=.8]{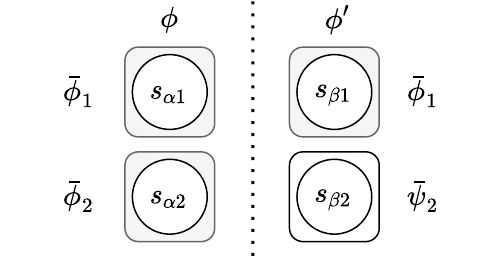}
      \caption{}
      \label{fig:conditional_coarseness_vs_transfer2}
    \end{subfigure}
  \caption{Two transfer settings,
  each showing the same two states with two different state abstractions.
  In each setting the abstract state $\bar\psi_2$ falls outside of the transfer cover.}
  \label{fig:conditional_coarseness_vs_transfer}
\end{figure}

Since transfer cover requires exact equality
between abstract states in $\Phi_\alpha$ and $\Phi_\beta$
and coarseness includes equality under isomorphism,
the interplay between transfer cover and coarseness
can be non-obvious.
It is for example possible
for a state abstraction to have different transfer cover
but the same level of conditional coarseness,
by means of a simple
relabelling of abstract states.
This relabelling is a bijection,
i.e. an isomorphism,
and preserves conditional coarseness.
\figref{fig:conditional_coarseness_vs_transfer}
shows two transfer settings
where shaded abstract states depict states in the transfer cover.
In each setting,
$\phi$ is \emph{equally fine} as $\phi'$.
However,
in the first setting
(\figref{fig:conditional_coarseness_vs_transfer1})
$\phi$
has less transfer cover
than $\phi'$,
and in the second setting
(\figref{fig:conditional_coarseness_vs_transfer2})
$\phi$
has more transfer cover.
The difference in each case is that one abstract state
is simply relabelled
as $\bar\psi_2$,
i.e. a state that is not in the transfer cover.
While conditional coarseness is preserved,
transfer cover can be affected as transfer cover requires exact equality.
From this perspective
transfer cover is a more stringent property,
one that is necessary for transfer.
This also means that there is no straightforward relation
between a change in coarseness and a change in transfer cover.
Just like in the single-task setting however,
there does exist
a direct relation between conditional coarseness and
and
\emph{conditional control granularity}.

\begin{definition}[conditional control granularity]
  Let
  $\phi_\alpha: \sset_\alpha \mapsto \asset_\alpha$,
  $\phi_\beta: \sset_\beta \mapsto \asset_\beta$,
  and
  $\phi_\beta': \sset_\beta \mapsto \asset_\beta'$
  be state abstractions
  such that
  $\asset_\alpha \cup \asset_\beta \cup \asset'_\beta$,
  and
  let
  $\mdp_\beta=\langle \sset_\beta, \aset, p_\beta, r_\beta \rangle$
  be an MDP.
  We say
  that $\phi_\beta$
  is
  \emph{more control granular
  for $\mdp_\beta$
given $\phi_\alpha$}
  than $\phi_\beta'$
  if all derived policies contained in
  $\Pi_\delta(\phi_\beta', \mdp_\beta; \phi_\alpha)$
  are also contained in
  $\Pi_\delta(\phi_\beta', \mdp_\beta; \phi_\alpha)$.
\begin{equation*}
\Pi_\delta(\phi_\beta', \mdp; \phi_\alpha)
\subseteq
\Pi_\delta(\phi_\beta, \mdp; \phi_\alpha)
\;\Longleftrightarrow\;
\forall \delta' \in\Pi_\delta(\phi_\beta', \mdp; \phi_\alpha):
\exists \delta \in\Pi_\delta(\phi_\beta, \mdp; \phi_\alpha):
\delta = \delta'
\end{equation*}
If in addition
$
\Pi_\delta(\phi_\beta', \mdp; \phi_\alpha)
\neq
\Pi_\delta(\phi_\beta, \mdp; \phi_\alpha)
$,
we say that
$\phi_\beta$
is
\emph{strictly more control granular}
than $\phi_\beta'$ given $\phi_\alpha$.
\end{definition}

\begin{proposition}[control granularity--coarseness trade-off]
  \label{thm:mt_control_coarseness_tradeoff}
  Let
  $\phi_\alpha: \sset_\alpha \mapsto \asset_\alpha$,
  $\phi_\beta: \sset_\beta \mapsto \asset_\beta$,
  and
  $\phi_\beta': \sset_\beta \mapsto \asset'_\beta$
  be state abstractions
  such that
  $\asset_\alpha \cup \asset_\beta \cup \asset'_\beta$
  and
  let
  $\mdp_\beta=\langle \sset_\beta, \aset, p_\beta, r_\beta \rangle$
  be an MDP.
  Then
  $\phi_\beta$
  is \emph{finer given $\phi_\alpha$}
  than $\phi_\beta'$
  if and only if
  $\phi_\beta$
  is more \emph{control granular
  given $\mdp_\beta$ and $\phi_\alpha$}
  than $\phi_\beta'$.
  Formally:
  \begin{equation*}
    \phi_\beta \underset{\phi_\alpha}{\finereq} \phi_\beta'
    \;\Longleftrightarrow\;
    \Pi_\delta(\phi_\beta, \mdp_\beta; \phi_\alpha)
    \supseteq
    \Pi_\delta(\phi_\beta', \mdp_\beta; \phi_\alpha)
  \end{equation*}
\end{proposition}
The proof is a straightforward extension of the
proof of the
single-task state abstraction Proposition~\ref{thm:control_coarseness_tradeoff}.
We omit it here for the sake of brevity.

Just like in the single-task setting,
state abstractions
can impact attainable value functions
in the multi-task setting.
We can compare attainable value functions
between two state abstractions for an MDP,
conditional on another state abstraction
(based on transfer cover).

\begin{definition}[conditional transfer value]
  Let
  $\phi_\alpha: \sset_\alpha \mapsto \asset_\alpha$,
  $\phi_\beta: \sset_\beta \mapsto \asset_\beta$,
  $\phi_\beta': \sset_\beta \mapsto \asset'_\beta$
  be state abstractions
  such that
  $\asset_\alpha \cup \asset_\beta \cup \asset'_\beta$
  and
  let
  $\mdp_\beta=\langle \sset_\beta, \aset, p_\beta, r_\beta, \gamma \rangle$
  be an MDP.
  We say that
  $\phi_\beta$
  has \emph{greater transfer value for $\mdp_\beta$ given $\phi_\alpha$}
  than
  $\phi_{\beta'}$,
  denoted
  $v_{\phi_\beta}(\mdp_\beta; \phi_\alpha) > v_{\phi_\beta'}(\mdp_\beta; \phi_\alpha)$,
  if
  there exists a derived policy for
  $\phi_\beta$
  with a value function greater than or equal to
  all derived policies' value functions
  for $\phi_\beta'$:
  \begin{equation*}
    \begin{gathered}
      v_{\phi_\beta}(\mdp_\beta; \phi_\alpha) \geq v_{\phi_\beta'}(\mdp_\beta; \phi_\alpha)
      \\
      \Updownarrow
      \\
      \exists \delta_\beta \in \Pi_\delta(\phi_\beta, \mdp_\beta; \phi_\alpha):
      \forall \delta_\beta' \in \Pi_\delta(\phi_\beta', \mdp_\beta; \phi_\alpha):
      v_{\delta_\beta}(\cdot; \mdp_\beta) \geq v_{\delta_\beta'}(\cdot; \mdp_\beta)
    \end{gathered}
  \end{equation*}
  If the above inequality is strict,
  we say that
  $\phi_\beta$ has \emph{strictly greater transfer value}
  than $\phi_\beta'$.
\end{definition}

As is the case in the single-task setting,
so too in the multi-task setting are
control granularity and coarseness on one side
and transfer value on the other side
related.
This time it takes into account the transfer cover.

\begin{theorem}[{transfer value--control trade-off}]
  \label{thm:mt_transfer_value_control_tradeoff}
  Let
  $\phi_\alpha: \sset_\alpha \mapsto \asset_\alpha$,
  $\phi_\beta: \sset_\beta \mapsto \asset_\beta$,
  and
  $\phi_\beta': \sset_\beta \mapsto \asset'_\beta$
  be state abstractions
  such that
  $\asset_\alpha \cup \asset_\beta \cup \asset'_\beta$
  and such that
  $\phi_\beta$
  has
  \emph{strictly higher transfer value given $\phi_\alpha$}
  for an MDP
  $\mdp_\beta=\langle \sset_\beta, \aset, p_\beta, r_\beta \rangle$
  than
  $\phi_\beta'$.
  Then there must
  exist two states
  in the transfer cover of
  $\phi_\beta'$ given $\phi_\alpha$
  such that
  $\phi_\beta'$ does not distinguish between them
  while either $\phi_\beta$ does,
  or not both states are in the transfer cover of $\phi_\beta$.
  In addition,
  there exists a partially derived policy for
  $\phi_\beta$
  that takes advantage of this
  by controlling $s$ and $t$ differently
  in a way that no partially derived policy of $\phi_\beta'$ can do.
  Formally:
  \begin{gather*}
    v_{\phi_\beta}(\mdp_\beta; \phi_\alpha) > v_{\phi_\beta'}(\mdp_\beta; \phi_\alpha)
    \\
    \Downarrow \\
      \exists s, t \in \sset_\beta:\\
    \begin{aligned}
      &1.\; \phi_\beta(s) \neq \phi_\beta(t) \vee s, t \not \in {\phi_\beta}^{-1}(\Phi_\alpha) \\
      &2.\; \phi_\beta'(s) = \phi_\beta'(t) \wedge s, t \in {\phi_\beta'}^{-1}(\Phi_\alpha) \\
      &3.\; \exists \delta_{\beta} \in \Pi_\delta(\phi_\beta, \mdp_\beta; \phi_\alpha):
      \delta_\beta(\cdot\mid s) \neq \delta_\beta(\cdot\mid t)
      \\
      &4.\; \forall \delta_{\beta'} \in \Pi_\delta(\phi_\beta', \mdp_\beta; \phi_\alpha):
      \delta_{\beta'}(\cdot\mid s) = \delta_{\beta'}(\cdot\mid t)
    \end{aligned}
  \end{gather*}
\end{theorem}

The proof of this theorem proceeds exactly like
that of
Theorem~\ref{thm:transfer_value_control_tradeoff}
with the added limitation imposed by transfer cover.
Analogous to the single task setting we derive the following corollary.

\begin{corollary}
  \label{thm:mt_transfer_value_control_tradeoff_corollary}
  Let
  $\phi_\alpha: \sset_\alpha \mapsto \asset_\alpha$,
  $\phi_\beta: \sset_\beta \mapsto \asset_\beta$,
  and
  $\phi_\beta': \sset_\beta \mapsto \asset_\beta$
  be state abstractions
  such that
  $\asset_\alpha \cup \asset_\beta$
  and such that
  $\phi_\beta$
  has
  \emph{strictly higher transfer value given $\phi_\alpha$}
  for an MDP
  $\mdp_\beta=\langle \sset_\beta, \aset, p_\beta, r_\beta \rangle$
  than
  $\phi_\beta'$.
  Then
  $\phi_\beta'$ can both be \emph{no finer given $\phi_\alpha$}
  than $\phi_\beta$
  and \emph{no more control granular given $\phi_\alpha$}
  than $\phi_\beta$.
  Formally:
  \begin{equation}
    v_{\phi_\beta}(\mdp_\beta; \phi_\alpha) > v_{\phi_\beta'}(\mdp_\beta; \phi_\alpha)
    \Longrightarrow
    \phi_\beta' \not \underset{\phi_\alpha}{\finereq} \phi_\beta
    \;\wedge\;
    \Pi_\delta(\phi_\beta, \mdp_\beta; \phi_\alpha)
    \not \subseteq
    \Pi_\delta(\phi_\beta', \mdp_\beta; \phi_\alpha)
  \end{equation}
  The following also holds.
  If $\phi_\beta$ is coarser
  and less control granular given $\phi_\alpha$
  than $\phi_\beta'$,
  it is impossible for
  $\phi_\beta$
  to have a higher transfer value
  given $\phi_\alpha$
  than
  $\phi_\beta'$.
  Formally:
  \begin{equation}
    \phi_\beta' \underset{\phi_\alpha}{\finereq} \phi_\beta
    \;\wedge\;
    \Pi_\delta(\phi_\beta, \mdp_\beta; \phi_\alpha)
    \subseteq
    \Pi_\delta(\phi_\beta', \mdp_\beta; \phi_\alpha)
    \Longrightarrow
    v_{\phi_\beta}(\mdp_\beta; \phi_\alpha) \not> v_{\phi_\beta'}(\mdp_\beta; \phi_\alpha)
  \end{equation}
\end{corollary}

\paragraph{}
Ideally we'd have maximal transfer cover
and maximal transfer value.
Maximal transfer cover
means that a partially derived policy is
\emph{complete}:
behavior is transferred
for every single target state.
Nothing more needs to be learned for the target task.
Maximum transfer value
in addition means
that it is possible for
a derived policy to be optimal everywhere.
This is no guarantee that it will be
but it does guarantee that the optimal policy is not excluded
by means of policy derivation.
Having both maximal transfer cover and transfer value
is an ideal situation,
though impossible in all but the most trivial cases.
For some intuition behind this,
consider when transfer cover
is highest.
If one task is a subtask of another
then behavior might be transferred readily
and beneficially to the subtask.
In less trivial settings
where this is not the case,
there simply are no source states
from which we could substitute the behavior for every single target state.
It is always possible to arbitrarily increase transfer cover
by increasing overlap between the abstract state sets
$\Phi_\alpha$
and $\Phi_\beta$,
arbitrarily equating states across tasks.
However, this overlap is not necessarily beneficial,
as would be reflected in the transfer value.
In general, there is no free lunch;
there is no way for a given transfer setting
to maximise both transfer cover and transfer value.
That summarised,
we end this section with a definition of those state abstractions
that do in fact not sacrifice on transfer value
while still being as coarse as possible.

\begin{definition}[minimal control state abstraction]
  A state abstraction
  $\phi_\beta: \sset_\beta \mapsto \asset_\beta$
  in a \emph{minimal control state abstraction}
  for an MDP
  $\mdp_\beta=\langle \sset_\beta, \aset, p_\beta, r_\beta \rangle$
  given a state abstraction
  $\phi_\alpha: \sset_\alpha \mapsto \asset_\alpha$
  if $\phi_\beta$
  has \emph{maximal transfer value}
  given $\phi_\alpha$
  and there exists no state abstraction coarser given $\phi_\alpha$
  than $\phi_\beta$
  that has higher transfer value for $\mdpb$
  than $\phi_\beta$.
\end{definition}

As in the single task setting,
a \emph{minimal control state abstraction}
cannot be coarsened without negatively impacting transfer value.
Since multi-task coarsening can result from increasing transfer cover,
this means that a minimal control state abstraction
can also not increase its transfer cover
without affecting transfer value.

\paragraph{}
We have described here a framework to assess
state abstractions
and the opportunities for transfer
that they provide.
What we have not gone into detail about
is how to actually arrive at such state abstractions.
It is this topic that forms
the heart of
Section~\ref{sec:predictive_representations}
and this paper as a whole.
Before we get there however,
we will briefly introduce the basic setting of
\emph{hierarchical reinforcement learning},
which will be the framework in which we will leverage
our newly developed and assessed state abstractions.

\subsection{Skills as Options and Hierarchical Reinforcement Learning}%
\label{sec:general_skills}
Now that we have discussed state abstraction
at length,
it is time we discuss its counterpart:
\emph{action abstraction}.
A complete discussion
building on the state abstractions we develop in the rest of this paper
can be found in
Section~\ref{sec:skills_that_transfer}.
This section will be limited
to laying the foundations.

\paragraph{}
In this paper we
rely on the well-known formalism of
\emph{options}
\citep{Sutton1999},
useful for its versatility and lack of assumptions.
Options can be seen as macro-actions
that encapsulate a particular behavior.
In particular,
we focus on options that
describe
\emph{local}
and
\emph{reusable}
behavior.
To capture these notions
we refer to such options
as \emph{skills},
as skills in common parlance
describe a very specific
(and repeatedly applicable) behavior.

\paragraph{}
To formalize our notion of \emph{skills},
we extend the traditional
definition of
\emph{options}
with a
state set that the option operates on.
\begin{definition}[option]
  An option
  $\omega$
  is a tuple
  $\langle \sset_\omega, \pi_\omega, \beta_\omega \rangle$,
  with
  $\sset_\omega$
  the option's
  \emph{domain},
  and
  \begin{align*}
    &\pi_\omega: \sset_\omega \mapsto \probset(\aset) \\
    &\beta_\omega: \sset_\omega \mapsto \probset(\{0,1\})
  \end{align*}
  the option control policy and termination policy respectively.
\end{definition}
\paragraph{}
Given an MDP
$\mdp$
with action
set
$\aset$
and an option set
$\Omega = \left\lbrace \omega \mid \omega = \langle \sset_\omega, \pi_\omega, \beta_\omega \rangle \right\rbrace$,
we construct
an option-enabled MDP
(also known as a Semi-MDP; \citep{Sutton1999})
$\mdp_\Omega$ with action set
$\aset \cup \Omega$.
Intuitively,
the agent expands its action set with the set of options
and can use them interchangeably,
making no distinction between primitive action ($a \in \aset$)
or option.
Upon invocation,
an option's control policy
selects primitive actions
until a $1$ is sampled from its termination policy
indicating termination.
Upon termination,
control is passed back to the agent,
for whom only a single step has passed.

There is a form of action abstraction at play here.
While multiple primitive steps may have been taken
by the option,
to the agent it appears as only one action.
This action abstraction leads to a hierarchy,
with the option-enabled agent
initiating options on the higher level
and the option operating on the lower level.
This view of options inducing a hierarchy
has led to the term
\emph{hierarchical reinforcement learning}
(HRL)
\citep{Barto2003HRL}.

A skill is defined for a particular state set
-- its domain.
When its domain is chosen to be task-independent
the skill itself becomes usable across tasks.
The hard part is of course to find such a task-independent domain.
Skills go hand in hand with state abstraction,
the latter providing the abstract state set for a skill to operate over.
We will return to combining
action abstraction and state abstraction
in Section~\ref{sec:skills_that_transfer}
when we define options over our state abstractions
that we introduce in Section~\ref{sec:outcome_state_abstraction}.

\section{State Abstractions for Transfer}
\label{sec:predictive_representations}
Whereas Section~\ref{sec:general} dealt mainly with
state abstractions and different ways to describe and compare them,
this section deals mainly with
the construction of useful state abstractions.
`Useful', as discussed in the previous section,
generally means high transfer cover
and high transfer value,
but that too will receive more nuance in this section
as we explore how a state abstraction construction
impacts transfer value.
In particular,
in this section we propose
a new method for state abstraction.
Based on predictions of \emph{outcomes},
i.e. common characteristics between tasks,
(Section~\ref{sub:outcomes_that_signal_reward})
outcome equivalent state abstractions
come with strong guarantees
of optimality
when used with a particular class of MDPs
(Section~\ref{sec:outcome_state_abstraction}).
We finish with an illustration
of what our outcome-based state abstraction affords us
in Section~\ref{sec:illustration_of_a_domain},
showing that outcome equivalent state abstractions
can be coarsened to increase opportunity for transfer
while incurring a cost in optimality.
This neatly leads into Section~\ref{sec:skills_that_transfer}
which illustrates that this cost can be overcome,
meaning that you \emph{can}
have your cake and eat it too.

\subsection{Outcomes as Reward Signals}%
\label{sec:outcomes}
So far we have not formally
defined the notion of a
\emph{domain} of related tasks.
We want to be able to describe tasks
that on the surface have something in common,
where transitions across tasks
follow \emph{the same rules}.
In order to do so we first need a way to describe transitions.
We start off with \emph{remarkable transitions}:
transitions that are especially \emph{meaningful} to the task at hand.
It will then be possible
to describe transitions that are remarkable
in the same way across tasks
---
a step closer to describing related tasks.

We first introduce
the notion of
an \emph{outcome}.
Our notion of \emph{outcomes}
is inspired by the
homonymous \emph{outcomes}
of \citet{Sherstov2005}
that describe differences in state,
tracing back to outcomes in planning
\citep{Boutilier2001},
and more recently
by the reward decomposition
of \citet{Barreto2017}.

A transition can be described in terms of multiple outcomes,
each denoting the occurrence of something remarkable.
In particular,
remarkable here is taken to mean
`relevant to reward':
the occurrence of an outcome
signals the presence of reward.
Or, from a different perspective,
a transition with no outcomes
is a transition with no reward.

Just like rewards can be randomly distributed for a given transition,
so are outcomes.
Where we use the random variable
$R_{t+1}$ to describe reward for a transition,
we use
$\Sigma_{t+1}$
to do the same for the outcome distribution.
The difference is that while
$R_{t+1}$ ranges over real numbers,
$\Sigma_{t+1}$
describes a distribution over vectors of outcomes.
The motivation for this is that multiple things may be relevant for reward,
each in a different way.
We rarely use $R_{t+1}$ directly
and usually use the expected reward for a transition
$(s, a, s')$ instead,
denoted $r(s, a, s')$.
In the same vein we use
$\sigma(s, a, s')$
to denote the expected outcome
for that transition.
In addition,
for simplicity we assume
that the expected outcome is itself an outcome.
We will discuss later
how this outcome function can be obtained.

\begin{definition}[outcome]
  The random variable describing
  \emph{outcomes} at time ${t+1}$
  is denoted $\Sigma_{t+1}$
  and is defined over the \emph{outcome domain} $\Sigma$.
Given a transition $(s, a, s') \in \sset \times \aset \times \sset$,
its
\emph{expected outcomes}
are given by
$\sigma(s, a, s')~\in~\Sigma$,
with
$\sigma(s, a, s')$
defined as
$\expected[\Sigma_{t+1} \mid S_t=s, A_t=a, S_{t+1}=s']$.
The
\emph{outcome domain}
$\Sigma$
is
defined as
$\Sigma \subseteq \realnumbers^d$.
Finally,
the expected reward
of a transition
$(s, a, s') \in \sset \times \aset \times \sset$
can be decomposed as
a product of the expected outcomes of that transition
and a task-specific
\emph{reward weights vector}
$\rw \in \realnumbers^d$
which describes how
different outcomes affect the reward
for that task:
\begin{equation}
  r(s, a, s') = \sigma(s, a, s')^\top \rw%
  ,
  \label{eq:reward_factored}
\end{equation}
\end{definition}

If this holds for all transitions we can also write
$r = \sigma^\top \rw$
which we define to only hold when the above equality
holds for all transitions in the domain $\sset \times \aset \times \sset$.
Since we do not work with MDPs that have randomly distributed reward
given an exact transition $(s, a, s')$,
we instead refer to
\emph{expected outcomes}
as \emph{outcomes}.

\paragraph{}
An outcome's presence is indicated
by a non-zero element in
$\sigma(s, a, s')$.
Similarly,
a zero element means the outcome did not occur.
If all elements are $0$
it means the transition is not
\emph{remarkable}
in any way,
i.e. it is not relevant for reward.

\begin{example}
\label{sub:outcomes_that_signal_reward}
\begin{figure}[h]
  \centering
  \begin{minipage}[t]{.6\linewidth}
    \includegraphics[height=0.3\linewidth]{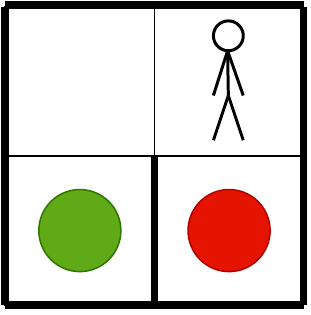}
    \hfill
    \includegraphics[height=0.3\linewidth]{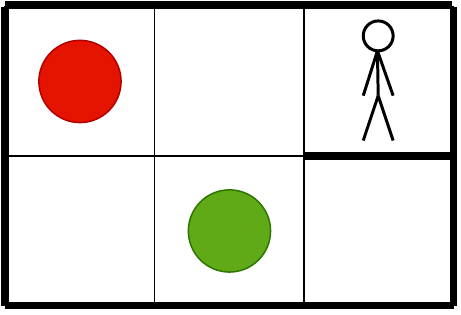}
    \hfill
  \end{minipage}%
    \captionof{figure}{Two gridworld tasks with different layouts
    that belong to the same domain.}
    \label{fig:gridworlds}
\end{figure}
For example,
\figref{fig:gridworlds}
describes two Gridworlds.
In both tasks
transitioning onto a green tile
yields a positive reward
of $10$ and terminates the episode.
Transitioning onto a red tile
incurs a negative reward
depending on the task at hand:
$-10$
in the first task
and $-1$
in the second task.
Both tasks seem to be composed of the same
\emph{elements};
red and green tiles
are present in both tasks
and both incur rewards,
albeit that the specific amount differs.
These elements are part of the language
of the domain that both tasks belong to
and outcomes are our tool
to describe them.

To continue the example,
we can describe the outcomes
for both tasks
as a binary-valued vector
\begin{equation*}
  \sigma(s, a, s') = (\texttt{green}(s'), \texttt{red}(s'))^\top,
\end{equation*}
where
$\texttt{green}(s)$
and
$\texttt{red}(s)$
take on a value of $1$
only if
the agent in state $s$
occupies a tile
of the corresponding colour.
That is,
there are two possible outcomes
and they describe where the agent ends up
after a transition.
In this example,
we have
$d=2$
because there are two outcomes
corresponding to the two different colors.
The reward weights vector for the first task
is
$w_{r_\alpha}=(10,-10)^\top$
and for the second task we have
$w_{r_\beta}=(10,-1)^\top$.
The difference between the reward weights
for both tasks
highlights the aforementioned difference in rewards;
the second task penalizes transitioning onto red cells
less heavily.
%

Transitions
with rewards
are not the only remarkable transitions
in this example.
Both tasks
contain impassable walls
and attempted movement in the direction of a wall will fail
in both tasks.
This, just like transitioning onto red or green cells,
is
also a kind of transition
that exists in both tasks
and is part of the
\emph{elements}
that make up the domain.
It may therefore be useful to expand the
outcome function
$\sigma$
to include other outcomes that are not relevant
to reward.
This irrelevance will show itself
in corresponding elements of $\rw$
which are set to $0$.
As a result,
the outcome function $\sigma$
can always be expanded
while maintaining correctness.
We will revisit this idea
of \emph{non-necessary outcomes}
later.
\end{example}

To summarize the above,
outcomes allow
describing reward functions of different tasks
in terms of invariant
commonalities.
These commonalities make up the language
of what we have so far informally referred to as
\emph{domains}.
We formalize a domain
for the purposes of our work as follows.

\begin{definition}[domain]
  A domain
  $\domain$
  is a set of tasks
  with the same
  \emph{outcome domain}
  $\Sigma$:
\begin{equation*}
  \domain \defas \{ \mdp |
  \mdp =\langle \sset, \aset, p, r \rangle
  \ \text{with} \ %
  r(s,a,s') = \sigma(s,a, s')^\top \rw
\}
\end{equation*}
with
$\sigma: \sset \times \aset \times \sset \mapsto \Sigma$
and $\rw \in \realnumbers^d$
specific to each task
$\mdp$.
\end{definition}

We can now rely on outcomes
to describe domains
in terms of domain-invariant commonalities.
However, they only help us to describe transitions
relevant for the one-step reward function.
Many transitions are unremarkable
and have no non-zero outcomes associated with them.
To overcome this
lack of descriptive power
we can talk about a sequence of outcomes
allowing us to go beyond a single time step:
specifically,
a sequence of ouctomes
given a sequence of actions
and a starting state.
Since such an outcome sequence is conditional
on future behavior (an action sequence),
this gives us a way to describe
the future of a state in terms of outcomes.
We formalize the following.

\begin{definition}[outcome sequence--test]
  Let
  $\mdp = \langle \sset, \aset, p, r \rangle$
  be an MDP
  and
  let
  $\sigma(s, a, S')$
  be the random variable describing the distribution over outcomes
  associated with the transition from state $s$ following action $a$,
  with $S'$ the random variable describing the resulting state.
  For a state $s_t \in \sset$
  and action sequence
  $\langle a \rangle = (a_1,\dots,a_n) \in \aset^n$,
  we term the resulting sequence of outcomes
  an
  \emph{outcome sequence--test}
  and denote it as
  \begin{equation}
    \sigmaseqrv (s_t, \langle a \rangle) \defas (\sigma(s_t, a_1, S_{t+1}),\dots, \sigma(S_{t+n-1}, a_n, S_{t+n}))
    \label{eq:sigma_sequence}
  \end{equation}
  Since the future states
  $S_{t+1},\dots,S_{t+n}$
  are random variables,
  the outcomes derived from those transitions
  are also random variables,
  and so is the resulting outcome sequence.
  Its distribution describes possible outcome sequences given an action sequence.
  To stress that it is a random variable,
  we use the blackboard character $\bbsigma$.
  Note that an \emph{outcome sequence--test}
  is effectively an \emph{outcome sequence} random variable
  that is conditioned on a state and action sequence.
  The probability that it takes on a value
  $(\sigma_1,...,\sigma_n) \in \Sigma$,
  denoted
  $p_\sigma((\sigma_1,...,\sigma_n) \mid s_t, \aseq)$
  to mirror the notation for the transition dynamics $p$,
  is defined as:
  \begin{equation*}
    p_\sigma((\sigma_1,...,\sigma_n) \mid s_t, \aseq)
    \defas
    \Pr_{p}\left(
      \sigmaseqrv(s_t, \aseq) = (\sigma_1,...,\sigma_n)
    \right)
  \end{equation*}
  Its computation can be expressed with the help of the shorthand
  $
  p_\sigma(\sigma_i \mid s, a, s') \defas
  \Pr(\Sigma_{t+1} = \sigma_i \mid S_t=s, A_t=a, S_{t+1}=s'))
      $,
      as:
  \begin{align*}
    \begin{split}
    p_\sigma((\sigma_1,...,\sigma_n) \mid s_t, \aseq)
      &=
      \sum_{s_{t+1} \in \sset} p(s_{t+1} \mid s_t, a_1)
      p_\sigma(\sigma_1 \mid s_{t}, a_1, s_{t+1})
      \\
      &\qquad \qquad
      \dots
      \sum_{s_{t+n} \in \sset}
      p(s_{t+n} \mid s_{t+n-1}, a_{n})
      p_\sigma(\sigma_n \mid s_{t+n-1}, a_n, s_{t+n})
    \end{split}
    \\
    \begin{split}
      &=
      \sum_{s_{t+1},...,s_{t+n} \in \sset^{n}}
      p(s_{t+1} \mid s_t, a_1)
      \dots
      p(s_{t+n} \mid s_{t+n-1}, a_{n})
      \\
      &\qquad \qquad
      p_\sigma(\sigma_1 \mid s_{t}, a_1, s_{t+1})
      \dots
      p_\sigma(\sigma_n \mid s_{t+n-1}, a_n, s_{t+n})
      % ...
    \end{split}
  \end{align*}
\end{definition}

\begin{example}
\begin{figure}[h]
    \centering
    \includegraphics[width=.15\linewidth]{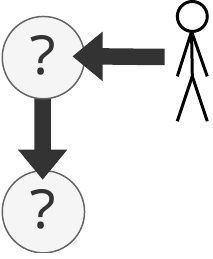}
    \caption{%
      An outcome sequence--test for the action sequence
      $(\texttt{left}, \texttt{down})$%
    .}
    \label{fig:gridworld_test}
\end{figure}
  We have already used
  outcomes
  to describe commonalities
  across tasks from the same domain.
  Consider now the outcome sequence--test
  $\sigmaseqrv(s, (\texttt{left}, \texttt{down}))$
  (depicted in
  \figref{fig:gridworld_test}),
  for the two states shown
  in our previous example in
  \figref{fig:gridworlds}.
  For both states
  the resulting outcome sequence--test
  assigns a probability of $1$
  to the outcome sequence
  $((0, 0), (1, 0))$.
  This sequence describes two transitions:
  the first transition
  resulting from the action
  \texttt{left}
  is unremarkable,
  denoted by $(0, 0)$,
  and the second transition
  resulting from the action
  \texttt{down}
  has a
  \texttt{green}
  outcome, denoted by $(1, 0)$.
\end{example}

The interesting thing about
the
outcome sequence--test
in the above example
is that it has the same distribution
over outcome sequences
for states
from different tasks.
In both states the same remarkable outcomes occur
after the same sequence of actions.
Building on outcomes,
outcome sequence--tests
give us a powerful tool to describe states
across tasks in a unified language.



\subsection{Outcome-Based State Abstractions That Transfer}
\label{sec:outcome_state_abstraction}
This section
develops state abstractions
based on our new notion of
\emph{outcomes}.
We will see that these state abstractions
lend themselves well for transfer,
though not all is straightforward.
By the end of this section we will have developed
a firm understanding of
outcome-based state abstractions
and their implications for transfer value,
as well as their applicability.

\paragraph{}
To start constructing new state abstractions,
consider first how outcome sequence--tests relate between states.
In particular,
when an outcome sequence--test
is identically distributed
for two different states and given action sequence,
we know that these two states
share the same distribution
over possible futures
for that particular action sequence
and
for all relevant purposes
-- that is, for reward.
For a collection of outcome sequence tests to all be identical for two different states, indicates that those states share the same distribution over relevant futures for any action sequence in the collection of tests.
In fact, we often only care about \emph{expected} reward
and need only focus on expected outcomes
and expected outcome sequences rather than the full distributions.

\begin{definition}[expected outcome sequence]
  Let
  $\mdp = \langle \sset, \aset, p, r \rangle$
  be an MDP
  with outcome function
  $\sigma$.
  The
  \emph{expected outcome sequence}
  for state $s \in \sset$
  and action sequence
  $\aseq = (a_1\dots a_n) \in \aset^n$,
  denoted
  $\sigmaseq(s, \aseq)$,
  is defined as the expected value of
  the outcome sequence--test
  $\sigmaseqrv(s, \aseq)$.
  Formally:
  \begin{align}
    \sigmaseq(s, \aseq)
    &\defas
    \expected_{p_\sigma} \left[
      \sigmaseqrv(s, a)
      \right]
    \\
  \intertext{
    For a continuous domain $\Sigma$,
    this can be expressed as
    the following n-dimensional volume differential:
  }
    \sigmaseq(s, \aseq)
    &\defas
    \int_{\Sigma}
    p_\sigma((\sigma_1,...,\sigma_n) \mid s_t, \aseq)
    (\sigma_1,...,\sigma_n)
    d^n(\sigma_1,...,\sigma_n)
    \label{eq:expected_outcome_seq_integral}
  \end{align}
  Since the outcome sequence--test $\sigmaseqrv(s, \aseq)$
  is a random variable describing a distribution over sequences,
  the \emph{expected outcome sequence}
  $\sigmaseq(s, \aseq)$
  is simply a sequence.
  Note the difference in notation to highlight this difference in nature.
\end{definition}

As stated before,
the interesting thing about expected outcome sequences
is that states can be described in an
entirely task-agnostic way
yet in a way that is sufficiently rich.
Two states
can now be equated based on
their expected outcome sequences,
despite being from different tasks
with entirely different transition and reward dynamics, and will respond in the same way (up to expectation) as one another in response to the action sequences included in the collection of expected outcome sequences. If the collection of outcome sequences includes all possible sequences of actions from the states, then equality of the expected outcome sequences corresponds to identical, action dependent, reward relevant futures for the pair of states.

\begin{definition}[outcome equivalence]
  Given two
  MDPs
  $\mdp_\alpha=\langle \sset_\alpha, \aset, p_\alpha, r_\alpha \rangle$
  and
  $\mdp_\beta = \langle \sset_\beta, \aset, p_\beta, r_\beta \rangle$
  with the same action set
  $\aset$,
  two states $s_\alpha \in \sseta$
  and $s_\beta \in \ssetb$
  are
  \emph{outcome equivalent}
  if and only if for every possible sequence of actions
  $\langle a \rangle = a_1,\dots,a_n$,
  the expected outcome sequences are equal.
  That is, if and only if
  $
  \forall n \in \naturalnumbers,
  \;
  \forall \langle a \rangle \in \aset^n,
  \forall s_\alpha \in \sseta, s_\beta \in \ssetb
  $:
  \begin{gather}
    \sigmaseqa(s_\alpha, \aseq) = \sigmaseqb(s_\beta, \aseq)
    \label{eq:outcome equivalence}
    \\
    \Updownarrow
    \nonumber
  \end{gather}
  \begin{equation*}
    \expected_{p_{\alpha}}\left[\sigmaseqrv_\alpha(s_\alpha, \aseq)\right]
    =
    \expected_{p_{\beta}}[\sigmaseqrv_\beta(s_\beta, \aseq)]
  \end{equation*}
\end{definition}

\paragraph{}
We will now develop the tools necessary
to compare expected outcome sequences
which will eventually allow us to construct
state abstractions based on them.
Note that
it is not in fact
necessary to consider
all action sequences
as there is a lot of duplicate
information between expected outcome sequences.
Proposition~\ref{thm:expected_outcome_sequence_sequence_expectations}
describes a basic simplification.
Proposition~\ref{thm:outcome_equivalent_length}
takes this further and allows us to simplify notation.

\begin{proposition}
  \label{thm:expected_outcome_sequence_sequence_expectations}
    Let
  $\mdp = \langle \sset, \aset, p, r \rangle$
  be an MDP
  with outcome function
  $\sigma$.
  The
  \emph{expected outcome sequence}
  for state $s \in \sset$
  and action sequence
  $\aseq = (a_1\dots a_n) \in \aset^n$
  can also be written as the sequence of
  expected outcomes after action sequences
  $(a_1\dots a_i)$ for
  $i$ ranging from $1$ to $n$.
  Formally:
  \begin{multline}
    \sigmaseq(s, \aseq)
    =
    \big(
    \expected_{p}
    \left[
      \sigma(S_t, A_t, S_{t+1}) \mid S_t=s, A_t=a_1
    \right]
    ,\dots,
    \\
    \expected_{p}
    \left[
      \sigma(S_{t+n-1}, A_{t+n-1}, S_{t+n}) \mid
      S_t=s, A_t=a_1,\dots,
      A_{t+n-1}=a_n
    \right]
  \big)
  \label{eq:expected_outcome_sequence_sequence_expectations}
  \end{multline}

\end{proposition}

\seeproof{expected_outcome_sequence_sequence_expectations}

\begin{proposition}
  \label{thm:outcome_equivalent_length}
    Given two
  MDPs
  $\mdp_\alpha=\langle \sset_\alpha, \aset, p_\alpha, r_\alpha \rangle$
  and
  $\mdp_\beta = \langle \sset_\beta, \aset, p_\beta, r_\beta \rangle$,
  two states
  $s_\alpha \in \sseta, s_\beta \in \ssetb$,
  and an action sequence
  $(a_1,\dots,a_{n-1}, a_n)$
  such that
  \begin{equation*}
    \sigmaseqa (s_\alpha, (a_1,\dots,a_{n-1}, a_n)) =
    \sigmaseqb (s_\beta, (a_1,\dots,a_{n-1}, a_n))
  \end{equation*}
  Then it also holds that
  \begin{equation*}
    \sigmaseqa (s_\alpha, (a_1,\dots,a_{n-1})) =
    \sigmaseqb (s_\beta, (a_1,\dots,a_{n-1}))
  \end{equation*}
  In other words,
  if two outcome sequence tests
  for $s_\alpha$ and $s_\beta$
  are the same for some
  length-$n$ action sequence,
  then the two outcome sequence tests
  for $s_\alpha$ and $s_\beta$
  constructed from
  the action subsequence of length $n-1$
  starting from $a_1$
  are also equal.

\end{proposition}

\seeproof{outcome_equivalent_length}

\paragraph{}
  Proposition~\ref{thm:outcome_equivalent_length}
  allows for some simplification
  as we do not need to consider all possible action sequences
  of all possible lengths
  but instead only the longest action sequences
  of which all shorter action sequences are prefixes.
  Let
  $\aset^n$
  be the labelled set
  of all $n$-length action sequences,
  each uniquely labelled so as to distinguish between action sequences.
  We rely on the notion of labelled sets
  because equality between labelled sets is evaluated
  pairwise on the elements based on labels,
  to contrast with equality between regular sets.
  If we now overload
  $\langle \sigma \rangle (s, \langle \aset^n \rangle)$
  to denote the {labelled set}
  of outcome sequence tests,
  we can simplify
  the outcome equivalence condition
  in
  Eq.~\eqref{eq:outcome equivalence},
  making use of
  Proposition~\ref{thm:outcome_equivalent_length},
  to
  \begin{equation}
    \lim_{n \rightarrow \infty} \bigg( \langle \sigma \rangle (s, \aset^n) = \langle \sigma \rangle (s', \aset^n) \bigg)
  \end{equation}
  By abuse of notation
  we might again simplify this to
  \begin{equation}
  \langle \sigma \rangle (s, \aset^\infty) = \langle \sigma \rangle (s', \aset^\infty)
  \end{equation}

What this means is that two states are outcome equivalent
if their outcome sequence tests
constructed from all infinite action sequences
are \emph{outcome equivalent}.
Thanks to
Proposition~\ref{thm:outcome_equivalent_length}
this is the same as
relying on all outcome sequence tests
\emph{of all lengths}.

\paragraph{}
Now that we have a definition of outcome equivalence between states
and some simplified notation to go with it,
it is time to extend this in a principled way
to more states across tasks.
The method for this bridging of tasks
is of course state abstraction.

\begin{definition}
  Let
  $\mdp = \langle \sset, \aset, p, r \rangle$
  be an MDP
  and let
  $\equi$
  be a binary relation over
  $\sset$
  such that
  for all
  $s, s' \in \sset$,
  $\equi(s, s')$
  holds if
  $s$
  and
  $s'$
  are
  \emph{outcome equivalent}.
  Then
  we say that
  $\equi$
  is
  \emph{outcome equivalent}.
  If in addition
  $\equi$
  is induced by a single state abstraction $\phi$,
  we also say that $\phi$ is
  \emph{outcome equivalent}.
\end{definition}

\begin{proposition}
  A binary relation that is \emph{outcome equivalent} is an \emph{equivalence relation}.
\end{proposition}

\begin{proof}
  Since outcome equivalence is based on equality between sequences
  which is itself an eqivalence relation,
  it too is an equivalence relation.
  \qed
\end{proof}

\paragraph{}
An outcome equivalent state abstraction
is based on
infinite action sequences
which makes it impractical for any real problem.
Nonetheless,
we continue with the proposed construction
because of its theoretical merit,
which we will discuss shortly.
Later in this paper
in Section~\ref{sec:skills_that_transfer}
we will consider again finite action sequences
to make outcome equivalent state abstractions more tractable.
Even with a limit on the length of the action sequences in consideration
however,
having access to future expected rewards
is not a small assumption.
Because the main goal of this paper is to investigate
the usefulness of outcome equivalent state abstractions,
we continue under the assumption that we have access
to expected outcomes.
We resume this discussion
on how to actually attain such a state abstraction
in practice
in Section~\ref{sec:conclusion}.

\paragraph{}
Aside from single-task state abstractions,
we can also extend the notion of outcome equivalence to apply to
relations
that span multiple tasks.
This is a particularly useful view for transfer.
On one side there is the inter ground task view:
states from different tasks with the same expected outcome sequences
are outcome equivalent,
despite being from different tasks.
Outcome equivalence relates states
in a meaningful way
that before had nothing obviously in common.
The other view is between an abstract task
and ground tasks.
An abstract state from an abstract MDP
is outcome equivalent
with its ground states
and all states from other MDPs
that are outcome equivalent with its ground states.
This is a powerful property:
MDP abstraction based on outcome equivalent state abstractions
preserves outcome equivalence.
We will see later why this is such an interesting property.
In short though,
in particular kinds of MDPs
an optimal policy for an abstract MDP
based on outcome equivalent state abstractions
will also result in optimal derived policies,
which is a great guarantee to have for transfer.

We formalize the findings for the two views of inter ground tasks
and abstract to ground tasks.

\begin{proposition}
  \label{thm:outcome_equivalent_relation_from_abstractions}
  Let
  $\mdp_\alpha=\langle \sset_\alpha, \aset, p_\alpha, r_\alpha, \gamma \rangle$
  and
  $\mdp_\beta = \langle \sset_\beta, \aset, p_\beta, r_\beta, \gamma \rangle$
  be two MDPs
  with outcome functions
  $\sigma_\alpha$ and $\sigma_\beta$ respectively
  and
  let
  $\phi_\alpha: \sset_\alpha \mapsto \asset_\alpha$
  and
  $\phi_\beta: \sset_\beta \mapsto \asset_\beta$
  be two
  \emph{outcome equivalent} state abstractions.
  Then the equivalence relation
  $\equi_{\phia,\phib}$
  over
  $\sseta \cup \ssetb$
  implied by
  $\phia$ and $\phib$
  is itself
  \emph{outcome equivalent}.
\end{proposition}

\begin{proof}[\ref{thm:outcome_equivalent_relation_from_abstractions}]
  This follows immediately from the definition of
  outcome equivalence between states of different tasks.
  \qed
\end{proof}

\begin{proposition}
  \label{thm:outcome_equivalent_abstract_ground}
    Let
  $\mdp_\alpha=\langle \sset_\alpha, \aset, p_\alpha, r_\alpha, \gamma \rangle$
  and
  $\mdp_\beta = \langle \sset_\beta, \aset, p_\beta, r_\beta, \gamma \rangle$
  be two MDPs
  with respective outcome functions
  $\sigma_\alpha$ and $\sigma_\beta$
  and reward weightings
  $\rwa$ and $\rwb$
  such that
  $r_\alpha = \sigma_\alpha^\top \rwa$
  and
  $r_\beta = \sigma_\beta^\top \rwb$,
  and
  let
  $\phi_\alpha: \sset_\alpha \mapsto \asset_\alpha$
  and
  $\phi_\beta: \sset_\beta \mapsto \asset_\beta$
  be two
  \emph{outcome equivalent} state abstractions.
  Then all abstract states
  in $\Phia$
  are outcome equivalent with their ground states
  with respect to
  any abstract MDP
  $\mdp_\phia=\langle \Phia, \aset, p_\phia, r_\phib, \gamma \rangle$
  created from $\mdpa$ and $\phia$
  with outcome function
  $\sigma_\phia$
  and reward weighting
  $w_{r_\phia}$
  such that
  $r_\phia = \sigma_\phia^\top w_{r_\phia}$,
  given that the reward weighting
  for that abstract MDP is the same
  as those of the ground MDPs,
  i.e.
  $w_{r_\phia} = \rwb = \rwa$.
  Formally,
  $\forall \astate \in \Phia, \forall s_\beta \in \phib^{-1}(\astate)$:
  \begin{equation}
    \sigmaseq_\phia (\astate, \aset^\infty) = \sigmaseqb (s_\beta, \aset^\infty)
  \end{equation}

\end{proposition}

\seeproof{outcome_equivalent_abstract_ground}

\subsubsection*{Outcome equivalent construction}
For a given MDP
there are many outcome equivalent state abstractions,
the identity abstraction being the finest of them.
There is one straightforward way to construct
the coarsest outcome equivalent state abstraction,
which we will from here on refer to as the
\emph{minimal outcome equivalent state abstraction},
which maps together all states that have the same expected outcome sequences.
This construction assigns to each state
a labelled set of expected outcome sequences
as that state's abstraction.
The labels of this set are the action sequences,
the elements corresponding to those labels
are the expected outcome sequences.
Algorithm~\ref{alg:outcome_equivalent_abstraction}
lists this simple iteration.

\begin{algorithm}[ht]
  \caption{minimal outcome equivalent state abstraction construction}
  \label{alg:outcome_equivalent_abstraction}
  \begin{algorithmic}
    \State \textbf{input}
    MDP $\mdp = \langle \sset, \aset, p, r \rangle$
    \State initialize $\phi: \sset \mapsto \emptyset$
    \For{$s \in \sset$}
      \State $\phi(s) \gets \langle \sigma \rangle (s, \aset^\infty)$
    \EndFor
  \end{algorithmic}
\end{algorithm}

\paragraph{}
Given state abstractions
that are constructed by labelling states
with their outcome sequence tests
(according to Algorithm~\ref{alg:outcome_equivalent_abstraction}),
an equivalence relation that
equates these states across
tasks naturally follows.
This equivalence relation
is again outcome equivalent
(Proposition~\ref{thm:outcome_equivalent_relation_from_abstractions}).
This construction is part of what makes
outcome equivalence such a useful property:
a task-specific construction that is agnostic of other tasks
still results in a meaningful equivalence relation across tasks,
one that equates states with the same future outcomes.

\subsection{Transfer Optimality}
\label{sec:transfer_optimality}
Given an outcome equivalent state abstraction,
we need to be able to construct an abstract MDP
that we can then solve.
Many constructions are possible
as each choice for the abstract
transition and reward dynamics
$p_\phi$ and $r_\phi$
results in a different abstract MDP.
However,
not all abstract MDPs are created equally.
While the preceding propositions have shown that some properties
are not affected by the choice of abstract MDP
when using an outcome equivalent abstraction,
this does not go for all properties.
Eventually
we want to learn an policy for the abstract MDP
and translate it to the ground MDP.
Ideally, a good abstract policy leads to a good ground policy.
What follows is an example
to illustrate the problems we face
when employing an outcome equivalent state abstraction.
This will naturally lead to a new definition
of when transfer is really \emph{optimal}.

\begin{example}
  \figref{fig:mt_outcome_not_value_equivalent}
  depicts two MDPs
  $\mdp_\alpha$
  and
  $\mdp_\beta$.
  The respective abstractions $\phi_\alpha: \sset_\alpha \to \Phi_\alpha$ and $\phi_\beta: \sset_\beta \to \Phi_\beta$ are such that $\phi_\alpha(\alpha_0) = \phi_\beta(\beta_0) = \bar{\phi}_0$ (in other words states $\alpha_0$ and $\beta_0$  are mapped to the same abstract states) and all other states are mapped to unique abstract states (not depicted).
  It is then the case that $\Phi_\alpha \cap \Phi_\beta = \{\bar{\phi}_0\}$ and the transfer cover $\phi^{-1}_\beta(\Phi_\alpha) = \{\beta_0\}$.

  The resulting equivalence relation
  $\equi_{\phi_\alpha,\phi_\beta}$
  is outcome equivalent.
  To confirm that this is the case,
  we need only ensure that the two states $\alpha_0$ (in $M_\alpha$) and $\beta_0$ (in $M_\beta$) are outcome equivalent. For this,
  it is sufficient
  to inspect
  the resulting expected outcome sequences
  for action sequences
  $(b)$,
  $(a, b)$,
  and
  $(a, a)$
  for both states
  $\alpha_0$
  and
  $\beta_0$,
  and see that they are the same for both.
  As before,
  we will consider transfer
  from an abstract MDP
  $\mdp_{\phi_\alpha}$
  to
  $\mdp_\beta$.
  We will also see that this transfer can be sub-optimal.

\begin{figure}[h]
  \centering
  \begin{subfigure}[t]{.49\textwidth}
    \centering
    \includegraphics[height=4cm]{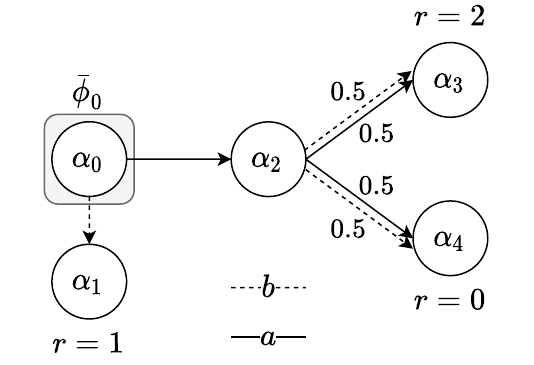}
    \caption{}
    \label{fig:mt_outcome_not_value_equivalent_a}
  \end{subfigure}
  \hfill
  \begin{subfigure}[t]{.49\textwidth}
    \centering
    \includegraphics[height=4cm]{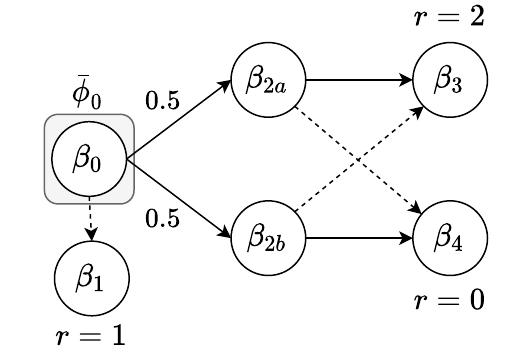}
    \caption{}
    \label{fig:mt_outcome_not_value_equivalent_b}
  \end{subfigure}
  \caption{}
  \label{fig:mt_outcome_not_value_equivalent}
\end{figure}

  Since $\phia$ is no coarser than the identity state abstraction,
  the abstract MDP
  $\mdp_{\phi_\alpha}$
  looks exactly the same as
  $\mdp_\alpha$
  with its state set replaced
  by $\Phia$,
  the image of the injective state abstraction $\phi_\alpha$.
  Interestingly,
  there are an inifite number of optimal policies
  as for all policies
  $\pi_{\phi_\alpha}$
  it holds that
  $v_{\pi_{\phi_\alpha}}(\bar\phi_0)=1$,
  regardless of the choice of distribution over actions $a$ and $b$
  at state $\astate_0$.
  Let us then consider the optimal policy
  $\pi^*_{\phi_\alpha}$
  such that
  $\pi^*_{\phi_\alpha}(b \mid \bar\phi_0) = 1$
  --
  that is,
  the policy
  that takes action $b$
  in state $\bar\phi_0$.

  We can now construct a partially derived policy
  $\pidb: \sset_\beta \mapsto \probset(\aset)$
  from $\phi_\alpha$
  for $\mdp_\beta$.
  Since the overlap
  between
  $\Phi_\alpha$
  and
  $\Phi_\beta$
  is the singleton
  $\{ \astate_0 \}$,
  the only constraint for this partially derived policy
  is
  \begin{equation*}
    \pidb(b \mid \beta_0) =
    \pi^*_{\phi_\alpha}(b \mid \astate_0)
    = 1%
    .
  \end{equation*}
  There are a few partially derived policies
  that all satisfy this constraint
  yet all differ only trivially
  as the MDP halts after an initial
  action $b$.
  Again,
  for all such derived policies
  $\pidb$
  it holds that
  $v_{\delta}(\beta_0)=1$.
  And herein lies \emph{the first problem}.
  There exists a policy
  for $\mdp_\beta$
  that always manages to reach state
  $\beta_3$
  and gain a reward of $2$.
  This policy is optimal
  with a value of $2$
  at state $\beta_0$.
  This is a real problem:
  we have reached an inferior policy on $M_\beta$
  despite deriving this from an optimal abstract policy
  based on an abstraction $\phia$
where the transfer cover was purely on outcome equivalent states.
  Therefore, although it might seem that outcome equivalence places quite sensible (and strict) conditions that should lead naturally to transfer (identical sequences of actions from equivalent states have the same expected outcome), they are not sufficient to guarantee optimal transfer. In this example, it is a consequence of the fact that the expected reward from $\alpha_2$ in $M_\alpha$ is the same for either action $a$ or $b$, is also the same as the expected reward from the uncertain superposition of states $\beta_{2a}$ and $\beta_{2b}$ induced by moving from $\beta_0$ in $M_\beta$, again for either action $a$ or $b$, and that both situations are reached by performing $a$ from equivalence class $\bar{\phi}$. However, in $M_\beta$ this uncertain superposition of states is resolved by the agent observing the true state, thus having a greater opportunity to control its future, than an agent in $M_\alpha$. Note that this is quite a carefully constructed pathological case and it is by no means clear that this situation arises frequently in any domain we may wish to investigate.

  \paragraph{}
  There is also a second problem.
  We could instead start off with the abstract policy
  $\pi^*_{\phi_\alpha}$
  such that
  $\pi^*_{\phi_\alpha}(a \mid \cdot) = 1$.
  That is,
  it always picks action $a$,
  which is also optimal for
  $\mdp_\phia$.
  Again,
  any partially derived policy
  $\pidb$
  for $\mdp_\beta$
  is only constrained
  for state $\beta_0$.
  This time around
  the optimal policy
  with
  $v_{\pi^*}(\beta_0) = 2$
  can indeed be derived
  from
  $\mdp_\phia$.
  Even so,
  its values can not be easily transferred
  from the abstract policy,
  as
  \begin{equation*}
  v_{\pi^*}(\beta_0) = 2
  \neq
  v_{\pi^*}(\astate_0) = 1%
  .
  \end{equation*}
  Even when the optimal policy can be
  derived from an optimal abstract policy
  in the case of an
  outcome equivalent relation,
  there is no straightforward way
  to transfer the value function
  of the abstract policy.
  The value function of the derived policy
  would need to be
  recomputed completely
  which further hinders transfer.
\end{example}

This example motivates the property
that we will call
\emph{transfer optimality}.
It requires two practical concepts that
we will define first,
the need for which has been demonstrated by the example.
First,
we need to extend the definition of
\emph{transfer value}
of a state abstraction to
transfer value of
an abstract MDP.
Second,
we require a notion of
\emph{value compatibility}
between abstract policies and their derived policies.

\begin{definition}[abstract MDP guaranteed transfer value]
  Let
  $\phi_\alpha: \sset_\alpha \mapsto \asset_\alpha$
  and
  $\phi_\beta: \sset_\beta \mapsto \asset_\beta$
  be two state abstractions
  such that
  $\asset_\alpha \cup \asset_\beta$.
  In addition,
  let
  $\mdp_\beta=\langle \sset_\beta, \aset, p_\beta, r_\beta \rangle$
  be an MDP
  and
  let
  $\mdp_{\phi_\alpha}=\langle \Phi_\alpha, \aset, p_{\phi_\alpha}, r_{\phi_\alpha}, \gamma \rangle$
  and
  $\mdp_{\phi_\alpha}'=\langle \Phi_\alpha, \aset, p_{\phi_\alpha}', r_{\phi_\alpha}', \gamma \rangle$
  be two abstract MDPs.
  We say that
  $\mdp_{\phi_\alpha}$
  has \emph{greater guaranteed transfer value for $\mdp_\beta$}
  than
  $\mdp_{\phi_\alpha}'$,
  denoted
  $
    v_{\mdp_{\phi_\alpha}}
    \underset{\mdpb,\phib}{\geq}
    v_{\mdp_{\phi_\alpha}'}
  $,
  if each of the optimal policies of
  $\mdp_{\phi_\alpha}$
  has a partially derived policy
  for $\mdp_\beta$
  that is at least as good as some policy
  derived from the optimal policies of
  $\mdp_{\phi_\alpha}'$.
  Formally,
  $
    v_{\mdp_{\phi_\alpha}}
    \underset{\mdpb,\phib}{\geq}
    v_{\mdp_{\phi_\alpha}'}
  $
  if
  the derived policies
  of all optimal abstract policies
  $\pi^*_\phia$
  have a value greater than
  the derived policies of any non-optimal abstract policies
  $\bar{\pi}_\phia$:
  \begin{equation*}
    \forall \delta \in \Pi_\delta(\phib, \mdpb; \pi^*_\phia):
    \forall \delta' \in \Pi_\delta(\phib, \mdpb; \bar{\pi}_\phia):
    v_\delta(\cdot; \mdpb) \geq v_{\delta'}(\cdot; \mdpb)
  \end{equation*}
  If the above inequality is strict
  for all $\pi_\phia$,
  we say that
  $\mdp_{\phi_\alpha}$
  has \emph{strictly greater guaranteed transfer value}
  than
  $\mdp_{\phi_\alpha}'$.
\end{definition}

What \emph{guaranteed transfer value}
really captures is how useful the abstract policies
for a given abstract MDP are in terms of transfer.
Importantly,
it takes into account all optimal abstract policies.
This is important because
different policies,
while identically valued in the abstract MDP,
need not result in derived policies with the same value functions.
It is perfectly possible for one optimal abstract policy
to result in a derived policy that is optimal
and another optimal abstract policy to result
in a derived policy that is sub-optimal.

\begin{definition}[value compatibility]
  Let
  $\phi_\alpha: \sset_\alpha \mapsto \asset_\alpha$
  and
  $\phi_\beta: \sset_\beta \mapsto \asset_\beta$
  be two state abstractions
  such that
  $\asset_\alpha \cup \asset_\beta$.
  In addition,
  let
  $\mdp_\beta=\langle \sset_\beta, \aset, p_\beta, r_\beta, \gamma \rangle$
  be an MDP
  and
  let
  $\mdp_{\phi_\alpha}=\langle \Phi_\alpha, \aset, p_{\phi_\alpha}, r_{\phi_\alpha}, \gamma \rangle$
  be an abstract MDP.
  Given an abstract policy
  $\pi_{\phi_\alpha}: \Phi_\alpha \mapsto \probset(\aset)$
  and its partially derived policy
  $\derived a\beta: \sset_\beta \mapsto \probset(\aset)$,
  their value functions
  $v_{\pi_{\phi_\alpha}}$
  and
  $v_{\delta_\beta}$
  are defined to be \emph{compatible}
  with respect to
  $\mdp_{ \phi_\alpha }$ and $\mdp_\beta$
  if it holds
  for the transfer cover
  $\forall s_\beta \in \phi_\beta^{-1}(\Phi_\alpha)$
  that:
  \begin{equation*}
    v_{\delta_\beta}(s_\beta; \mdp_\beta) = v_{\pi_{\phi_\alpha}}(\phi_\beta(s_\beta); \mdp_{\phi_\alpha})
  \end{equation*}
  If this holds
  for all partially derived policies
  $\delta_\beta$,
  we also say that
  $\pi_{\phi_\alpha}$
  is
  \emph{derived value--compatible}
  with respect to
  $\mdp_{ \phi_\alpha }$ and $\mdp_\beta$.
\end{definition}

\begin{definition}[transfer optimality]
  Let
  $\mdp_\alpha=\langle \sset_\alpha, \aset, p_\alpha, r_\alpha, \gamma \rangle$
  and
  $\mdp_\beta = \langle \sset_\beta, \aset, p_\beta, r_\beta, \gamma \rangle$
  be two MDPs,
  let
  $\phi_\alpha: \sset_\alpha \mapsto \asset_\alpha$
  and
  $\phi_\beta: \sset_\beta \mapsto \asset_\beta$
  be two state abstractions,
  and let
  $\mdp_{\phi_\alpha}=\langle \Phi_\alpha, \aset, p_{\phi_\alpha}, r_{\phi_\alpha}, \gamma \rangle$
  be an abstract MDP constructed
  from $\mdp_\alpha$ and $\phi$.
  We say that
  $\mdp_{\phi_\alpha}$
  is \emph{transfer optimal}
  for $\mdp_\beta$
  if
  \begin{enumerate}
    \item
      all optimal abstract policies
      $\pi^*_{\phi_\alpha}: \Phi_\alpha \mapsto \probset(\aset)$
      of
      $\mdp_{\phi_\alpha}$
      are \emph{derived value--compatible}
      with respect to
      $\mdp_\beta$; and
    \item
      the \emph{guaranteed transfer value}
      of
      $\mdp_{\phi_\alpha}$
      for $\mdp_\beta$
      is maximal
      (i.e.
      the optimal policy for $\mdp_\beta$ can be partially derived from
      every optimal abstract policy for $\mdp_{\phi_\alpha}$).
  \end{enumerate}
\end{definition}

Transfer optimality captures
the two things that were missing in the preceding example,
both important for transfer.
It might seem that the second property
is subsumed by the first,
for if all optimal abstract policies
are \emph{derived value--compatible},
every derived policy must have the same value function.
While that is the case,
there is no guarantee that the optimal value function for an abstract MDP
matches the optimal value function of the ground MDP.
The ground MDP could have a greater value function still.
The second property covers that scenario.

\paragraph{}
\comment{post snip}
Still,
for all its practicality of construction,
an outcome equivalence relation
does not guarantee transfer optimality.
This is a matter of \emph{control granularity}
--
an outcome equivalent state abstraction
can be overly coarse,
relinquishing the benefit of added control
and sacrificing transfer value as a result
(such as the example already seen in \figref{fig:mt_outcome_not_value_equivalent}).
Still, it is possible to find scenarios where
outcome equivalence is sufficient for transfer optimality.
This occurs when expected outcomes are sufficient
to base optimal behavior on,
rather than having to react to each state in turn
(as in the MDP in \figref{fig:mt_outcome_not_value_equivalent_a}).
This would correspond to a
\emph{planning perspective}:
an outcome equivalence relation
will be optimal in those cases where
all behavior can be planned in advance,
i.e. when a sequence of
actions, at the start decided on and then executed,
is as good as a reactive policy
which considers each state before acting.
A great number of tasks fall into this category,
excepting those where uncertainty in transitions
play a role,
and even then many would be
\emph{plannable}.

\begin{definition}[plannable]
  An MDP
  $\mdp$
  is
  \emph{plannable}
  if
  any policy can be approximated
  by a single sequence of actions.
  That is, if for any policy
  $\pi$
  it holds that
  for all states
  $s \in \sset$
  and all positive integers $n$
  there exists an action sequence
  $(a_1,...,a_n) \in \aset^{n}$
  such that the following holds:
  \begin{align}
    v_{\pi}(s)
        &=
    \expected_{p} \left[ \sum^{n}_{k=1} \gamma^{k-1} R_{t+k} \relmiddle| S_t=s, A_t=a_1, ..., A_{t+n-1}=a_n\right]
  \end{align}
\end{definition}

In other words,
an MDP
is \emph{plannable}
if the added reactivity of a policy
is not required to attain optimality.
For example,
all deterministic MDPs are plannable.
An example of a task that is
\emph{not plannable}
is depicted in
\figref{fig:mt_outcome_not_value_equivalent_b}.

Outcome equivalent state abstractions
are particularly useful
when it comes to plannable MDPs.
When transferring between plannable MDPs
by means of an outcome equivalence relation,
there exists a strong theoretical guarantee
of \emph{transfer optimality}.

\begin{theorem}
  \label{thm:plannable_outcome_equivalent_transfer_optimal}
    Let
  $\mdpa=\langle \sset_\alpha, \aset, p_\alpha, r_\alpha, \gamma \rangle$
  and
  $\mdpb= \langle \sset_\beta, \aset, p_\beta, r_\beta, \gamma \rangle$,
  be two \emph{plannable} MDPs
  with respective outcome functions
  $\sigma_\alpha$ and $\sigma_\beta$
  and reward weightings
  $\rwa$ and $\rwb$
  such that
  $r_\alpha = \sigma_\alpha^\top \rwa$
  and
  $r_\beta = \sigma_\beta^\top \rwb$,
  and let
  $\phi_\alpha: \sset_\alpha \mapsto \asset_\alpha$
  and
  $\phi_\beta: \sset_\beta \mapsto \asset_\beta$
  be two outcome equivalent state abstractions state abstractions.
  Then any abstract MDP
  $\mdp_\phia=\langle \Phia, \aset, p_\phia, r_\phib, \gamma \rangle$
  created from $\mdpa$ and $\phia$
  with outcome function
  $\sigma_\phia$
  and reward weighting
  $w_{r_\phia}$
  such that
  $r_\phia = \sigma_\phia^\top w_{r_\phia}$
  and
  such that
  $w_{r_\phia}=\rwa=\rwb$,
  is \emph{transfer optimal} for $\mdpb$.

\end{theorem}

\seeproof{plannable_outcome_equivalent_transfer_optimal}

\paragraph{}
Theorem~\ref{thm:plannable_outcome_equivalent_transfer_optimal}
indicates that for
plannable MDPs,
outcome equivalence relations are sufficient
for transfer optimality.
Where plannability does not hold on the other hand,
it is possible that the outcome equivalent state abstractions
are overly coarse,
relinquishing control
which leads to sub-optimal transfer
between states where the optimal policy is different.

\subsubsection*{Summary}
To summarize this section,
there are two fundamental properties that make
outcome equivalent state abstractions useful:
\begin{enumerate}
  \item
    Outcome equivalence relations
    are constructed in an inherently \emph{bottom-up} fashion
    (Algorithm~\ref{alg:outcome_equivalent_abstraction}).
    For each MDP
    an outcome equivalent state abstraction
    can be created separately
    from the ground up
    without having to consider other MDPs.
    Thanks to outcome equivalence,
    these state abstractions naturally lead to an outcome equivalence relation
    between MDPs.
    This makes it very easy to consider
    new MDPs without having to start from scratch
    as one just needs a new state abstraction for that MDP.
  \item
    Outcome equivalence relations
    guarantee \emph{transfer optimality}
    for the entire transfer cover
    when the ground MDPs are \emph{plannable}.
\end{enumerate}

While outcome equivalence relations are versatile
with strong guarantees,
these guarantees break down in the face
of highly stochastic,
non-plannable MDPs.
Still,
in practical problems this may not prove that
much of a problem,
especially
when we introduce our hierarchical framework
in Section~\ref{sec:skills_that_transfer}.
Outcome equivalent state abstractions
shine especially
in this setting.
The question of how
to actually compute the expected outcomes
necessary for the construction
of outcome equivalent state abstractions
is an important one that has come up
earlier in this section
but one that is left to the discussion
in Section~\ref{sec:conclusion}.

\subsection{Illustration of a Domain}%
\label{sec:illustration_of_a_domain}

\begin{figure}[ht]
  \renewcommand{\thesubfigure}{\Alph{subfigure}}
  \centering
  \hfill
  \begin{subfigure}[t]{.19\textwidth}
    \centering
    \includegraphics[scale=.3]{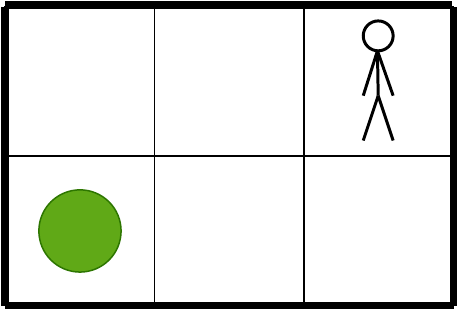}
    \caption{}
    \label{fig:domain_a}
  \end{subfigure}
  \hfill
  \begin{subfigure}[t]{.19\textwidth}
    \centering
    \includegraphics[scale=.3]{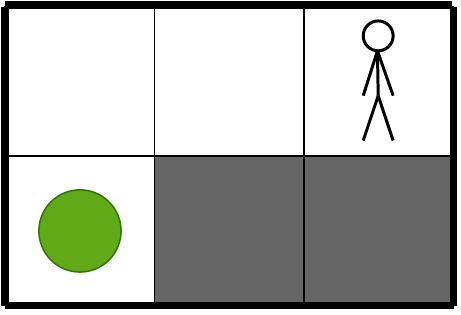}
    \caption{}
    \label{fig:domain_b}
  \end{subfigure}
  \hfill
  \begin{subfigure}[t]{.19\textwidth}
    \centering
    \includegraphics[scale=.3]{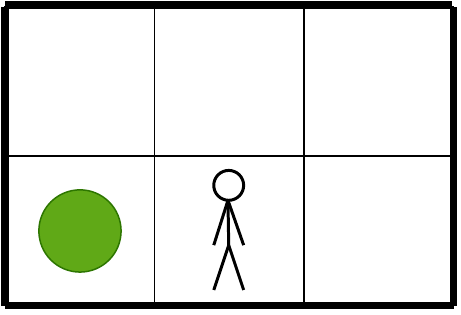}
    \caption{}
    \label{fig:domain_c}
  \end{subfigure}
  \hfill
  \begin{subfigure}[t]{.19\textwidth}
    \centering
    \includegraphics[scale=.3]{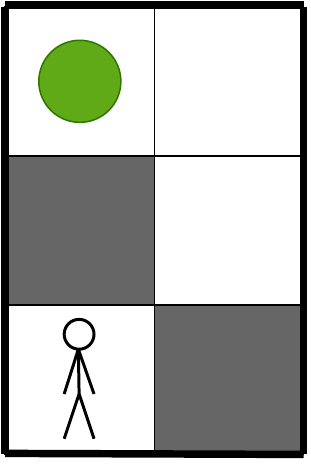}
    \caption{}
    \label{fig:domain_d}
  \end{subfigure}
  \hfill
  \begin{subfigure}[t]{.19\textwidth}
    \centering
    \includegraphics[scale=.3]{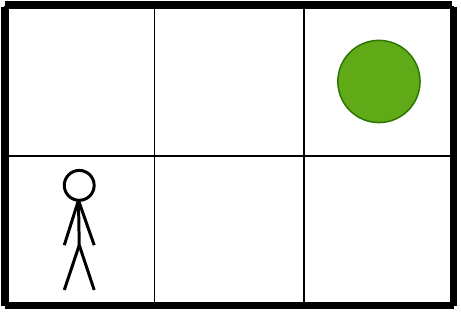}
    \caption{}
    \label{fig:domain_e}
  \end{subfigure}
  \hfill
  \caption{Five tasks
  of a small Gridworld domain.}
  \label{fig:mini_domain}
\end{figure}

In this section
we illustrate the usefulness of
outcome equivalent state abstractions
for relating tasks.
We consider a domain of tasks
with Gridworld semantics:
4 actions move the agent in the 4 cardinal directions,
unless something is blocking its path.
There is only one goal location
which incurs a positive reward:
there is only one outcome and
the outcome function is simply
$\sigma(s, a, s') = (\texttt{goal}(s'))$.
Since the domain is deterministic,
all tasks are guaranteed to be plannable.

We want to consider different tasks
and the relations between them,
so we generated all tasks
exhaustively according to the following specification.
Each task consisted of 6 cells
laid out in a 2-by-3 or 3-by-2 grid.
Up to two cells in the grid were occupied by obstacles,
which the agent could not move through,
nor could it move off the grid.
The agent starts in one particular cell
and the goal is in a different location.
A different start location describes a different task,
so some tasks are identical
excepting agent starting location.
It is entirely possible that a task is unsolvable
in case the agent is blocked from reaching the goal location,
as is the case for instance in
\figref{fig:domain_d}.
This specification results in 660 distinct tasks.

Each task contains at most six distinct states.
For each state we calculated
all expected future outcome sequences.
Since we knew a priori that no action sequence beyond
six steps would yield new information
(as the grid contains only six cells),
this was a tractable problem resulting
in 4096 distinct action sequences.
This results in
an expected outcome-sequence state abstraction
$\idealphi$
for each task,
with a shared abstract space $\Phi$
across all tasks.
While there were a total of 2544 states
across the 660 tasks,
there were only 370 distinct abstract states
in terms of expected outcome-sequences.
The majority of these
non-singleton equivalence classes
are due to
tasks that only differ in terms of starting location,
as is the case in Figures
\ref{fig:domain_a}
and
\ref{fig:domain_c}.
We might expect more distinct abstract states
given the amount of different action sequences
but our domain is of limited size.
The amount of distinct abstract states
is in a way
a description of the true size of the domain;
there are only 370 states that
differ in any meaningful way,
giving us a compact characterization
of the complexity of the domain.

Some tasks may look alike visually
yet their states have nothing in common in terms of
expected outcome sequences.
Consider the tasks
in Figures~\ref{fig:domain_a} and \ref{fig:domain_b}.
The general layout looks the same --
goal locations and agent starting locations are the same,
and we might even say that
the optimal policy in the second task is optimal in the first.
Still, because of the obstacles in the second task
each state gets a different outcome equivalent abstraction,
as the obstacles influence expected outcome sequences.
Yet surely the states in these tasks are more alike
than the states in some of the other tasks?
To gain insight into this question,
we computed
$\phi(s)$
as a vector of binary variables,
each describing the expected outcome
at the end of
a single action sequence.
Limiting ourselves to action sequences of up to length 6,
$\vecphi(s)$
is now a vector of
$\sum_{i=1}^{6}4^6=5460$
binary features.
The goal states are represented simply as
$\bm{0}$.
This is an outcome equivalent state abstraction.
But more than simply equating states
with fully identical expected outcome sequences,
this also gives us a way to describe states
with similar (non-identical) futures.
For example,
the start states
in Figures~\ref{fig:domain_a} and \ref{fig:domain_b}
have a trajectory in common that leads straight to the goal state
and so a part of their $\vecphi$ abstraction
will be identical.

\begin{figure}[t]
  \centering
  \hfill
  \begin{subfigure}[t]{.49\textwidth}
    \includegraphics[width=1\linewidth]{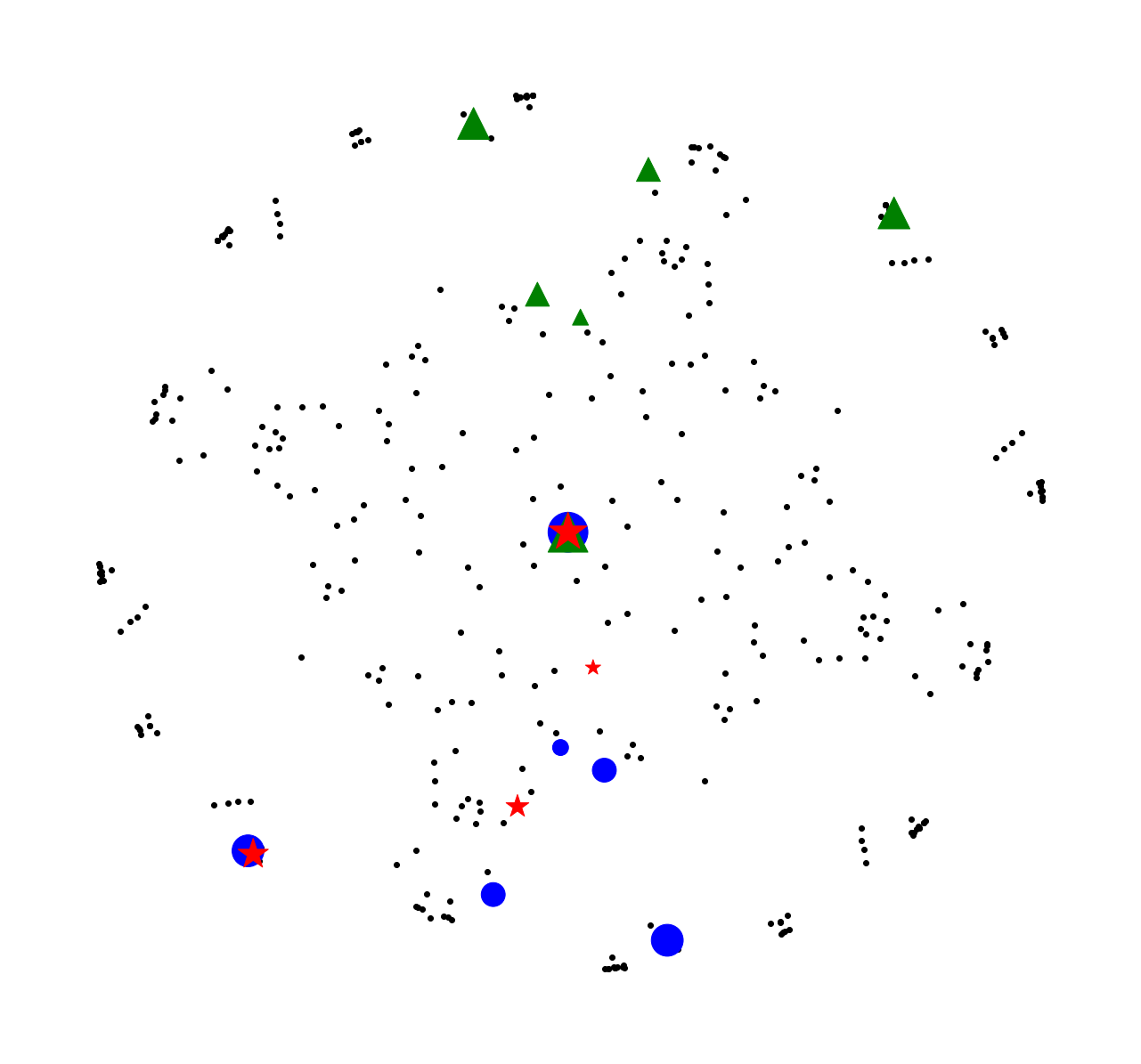}
    \caption{}
    \label{fig:mini_domain_embedding}
  \end{subfigure}
  \hfill
  \begin{subfigure}[t]{.49\textwidth}
    \includegraphics[width=\linewidth]{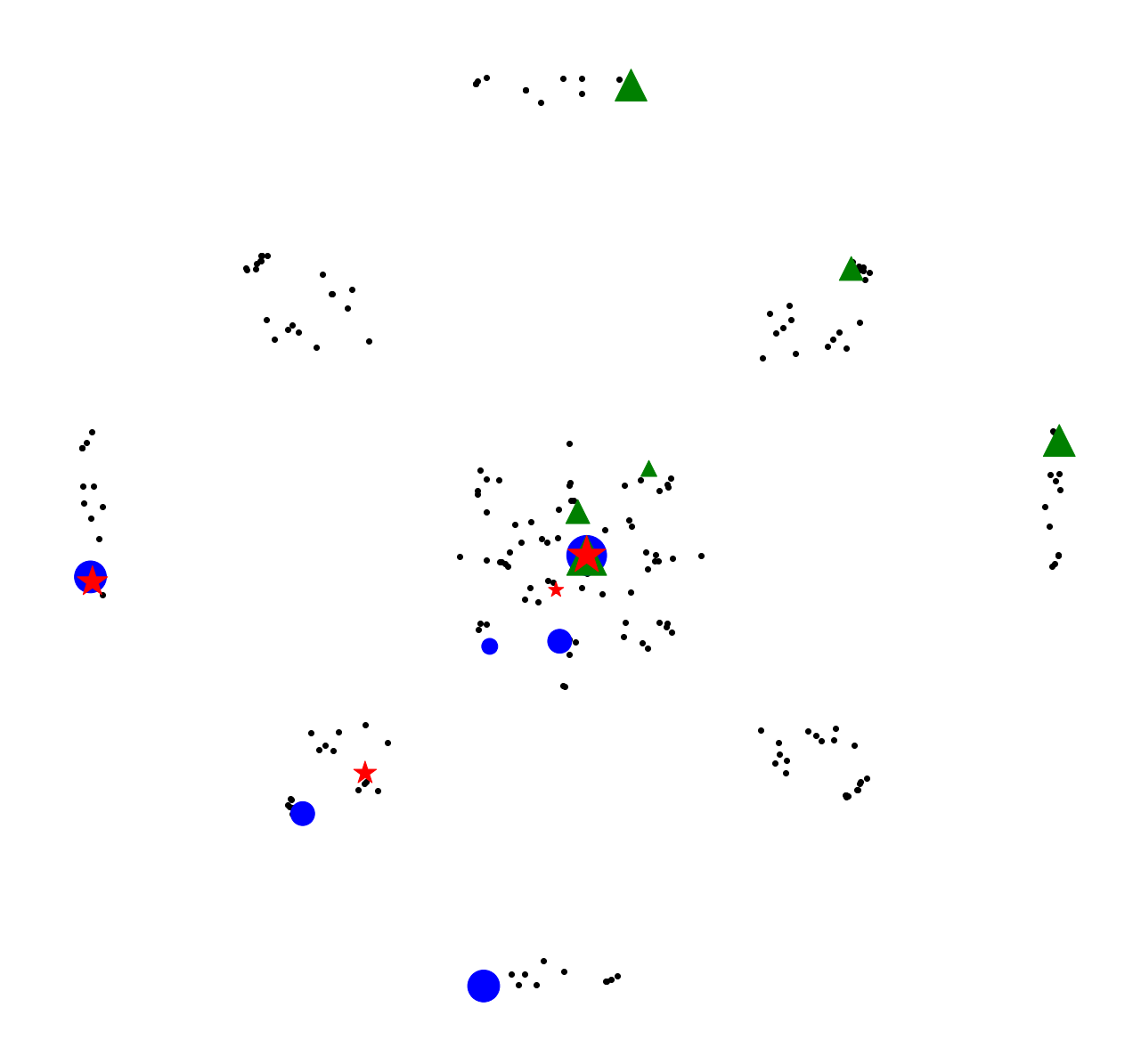}
    \caption{}
    \label{fig:domain_embedding_coarse}
  \end{subfigure}
  \caption{}
  \hfill
  \hfill
\end{figure}

Representing states this way
even allows us to look at distance measures between states,
even between states from different tasks.
\figref{fig:mini_domain_embedding}
depicts a
2-D Laplacian Eigenmap
\citep{Belkin2002eigenmaps}
of all 370 abstract states.
States that are shown close together
carry similar spectral information
in the graph
constructed based on K-nearest neighbors:
they are similar in terms of expected outcome sequences.
States from
Figures~\ref{fig:domain_a} and \ref{fig:domain_b}
are indicated
with blue circles and red stars respectively,
while states from
Figures~\ref{fig:domain_e}
are indicated by green triangles.
Marker size is correlated with closeness to the goal,
with the largest marker indicating goal location
and the smallest marker indicating starting location
(which happens to be as far as possible from the goal in each task).

Interestingly,
states from tasks A and B,
both drawn close to each other,
seem dissimilar to states from task E
which are drawn on the other side of the image.
In fact,
the near-goal states from tasks A and B
are drawn almost identically,
pointing at their having most future outcome sequences in common.
Even though this is an imperfect low-dimensional
visualization of a high-dimensional problem,
the observations it lends itself to are surprisingly insightful.

\paragraph{}
We used action sequences of length 6
because we knew that
for this particular domain
no information would be added
by sequences of any greater length.
We empirically verified this
and indeed,
the same distinct 370 abstract states emerge.
Ground states that share the same length-6
outcome equivalent abstraction share
the same
length-9 abstraction, for example.
In the language of abstractions,
this length-6 abstraction
is \emph{maximally fine}.
Adding information
to the state descriptions of $\vecphi$
--for example, by using longer action sequences--
makes the abstraction no finer
and does not gain anything.
We can also see
that at the domain level
this is far from a \emph{minimal control state abstraction}.
If only $370$
distinct abstract states emerge
despite using $5460$ features,
there must be a lot of duplicate information between features.
That is,
there are many state abstractions
that are coarser than this engineered abstraction
that still preserve all information
relevant to expected return.

This example does beg the question of the viability
of this kind of abstraction.
There are $\aset^k$ action sequences
of length $k$
and already for this tiny toy domain
that results in $5460$
action sequences.
No matter how one comes by the
corresponding outcome sequences,
be it by prediction or by model,
this is a large amount that
will only get combinatorially larger
as one considers less trivial domains.
This naturally leads us to think about
shorter action sequences
to make up the outcome sequence abstraction,
even if that does mean losing information.
This corresponds to a
\emph{coarsening}
of the abstraction
which leads to
a kind of partial observability of the environment:
without everything relevant to decision-making available
in the state description,
it may no longer be possible
to express the optimal policy.
In our framework this means
constructing an abstract MDP
and accepting that it will not accurately reflect the original dynamics,
losing transfer optimality and
possibly resulting in a sub-optimal derived solution.
In the domain at hand
we can for example only consider action sequences of length 3,
which results in only 217 abstract states,
less than the $370$ distinct abstract states we know this domain has.
While we sacrifice transfer optiamlity,
the decrease in size may just make the difference in feasability.
The corresponding visualization
of this coarser state abstraction
is shown in
\figref{fig:domain_embedding_coarse}.

Opting for a coarser state abstraction
is a trade-off:
computationally more viable,
it entails a loss of information.
Yet this loss of information
is not entirely without potential for gain.
If our aim were to learn a policy that could transfer across the domain,
a coarser state abstraction leads to less states to learn behavior for
and potentially more states to transfer behavior to.
At the same time
the behavior could not be globally optimal.
Despite the difficulty,
it is exactly this trade-off that we wish to exploit
in the following section.

\section{Skills That Transfer}%
\label{sec:skills_that_transfer}
Section~\ref{sec:predictive_representations}
introduced outcome equivalent state abstractions,
which for the class of \emph{plannable} MDPs
guarantee \emph{transfer optimality}.
Section~\ref{sec:illustration_of_a_domain}
showed how outcome equivalent state abstractions
can be coarsened
and illustrated
the trade-off between
coarseness (i.e. opportunity for transfer)
and transfer value
for a collection of similar tasks.
We concluded that past the threshold
of a \emph{minimnal control state abstraction},
coarseness
-- and with it, the transfer cover
and the opportunity for transfer --
can only be increased
at a loss of transfer value.
In this section
we posit
that it is possible to increase
the opportunity for transfer
at no cost to transfer value
by introducing some new machinery relating to action abstraction.
This machinery falls under the general category of
\emph{hierarchical reinforcement learning} (HRL).

Many realistic tasks have some structure to them,
a sense of consisting of discrete and partially isolated \emph{subtasks}.
Navigating one subtask may be done independently
of another subtask.
That is not to say that one subtask cannot
depend on the outcome of another,
or that some other complicated relation exists.
Rather,
once one subtask has started
only local information is relevant
--
the other subtasks are put on hold.
Moreover,
related tasks may share subtasks,
and a solution to a shared subtask
naturally applies in both related tasks,
an idea that has driven HRL research for decades
\citep{Barto2003HRL}.
Even where this repeated structure
is not immediately apparent,
we think that many realistic tasks exhibit
this pattern.
First
we will describe
subtasks,
their potential benefits,
and how to solve them.
After that we discuss
how to actually uncover the different subtasks
that make up a bigger task.

\paragraph{}
There is one immediate advantage to considering
a task as made up of subtasks.
Each subtask will by definition
be at most as complex as the original task
and ideally far less complex than that.
This decrease in complexity translates into
a subtask-specific
\emph{minimal control state abstraction}
that can be coarser than the
minimal control state abstraction
of the greater task.
That is,
the subtask
can have
a simpler, coarser, state abstraction
(as compared to the greater task)
that still guarantees
that a globally optimal policy is attainable,
simply because locally
(in the subtask)
that globally optimal policy
requires less information,
as reflected in the coarser state abstraction.
As discussed before,
a coarser state abstraction
naturally leads to an increase in transfer potential.
This transfer potential is available in
a single-task setting,
yet shines especially
when considering multiple tasks.

\figref{fig:common_subtasks}
depicts three gridworld tasks.
Under an outcome-equivalent state abstraction
these three tasks have no overlap in abstract states
---
that is, no state
in any of the tasks is outcome equivalent
to any other state.
Consider the first and third tasks
in
\figref{fig:common_subtasks_a}
and
\figref{fig:common_subtasks_c}
respectively.
In each task
we have depicted
with an arrow
a `common' subtask
of navigating up the vertical corridor.
If we now consider a
minimal
outcome equivalent state abstraction
for each task's
\emph{navigate-up-the-corridor} subtask,
it should become clear that it is
coarser than the minimal outcome equivalent state abstraction
for the original tasks.
What's more,
the subtasks
now have full overlap in abstract states:
each state has the same expected outcomes as a
state in the other task's subtask.
Indeed the transfer cover is 1.
This means that a policy learned
for an abstraction on one of the subtasks
will transfer to the other completely.
What's more,
the outcome-equivalent abstraction
guarantees that this transfer is optimal;
the optimal abstract policy in the source task
results in an optimal policy for the target task.

While not formally defined,
there appears to exist a bijection between the abstract MDPs
created from each subtask.
It is this bijection that justifies
us talking of a \emph{common subtask}
that exists between the two tasks.
In fact,
instead of deriving a partial policy
from the abstract subtask,
we could have gotten the same result
if we had simply followed the same actions
in the ground task as
we would have in the abstract subtask,
effectively emulating the abstract subtask
until it was finished.
This view will come in handy shortly.

\begin{figure}[ht]
  \centering
  \hfill
  \begin{subfigure}[t]{.3\textwidth}
    \centering
    \includegraphics[scale=.3]{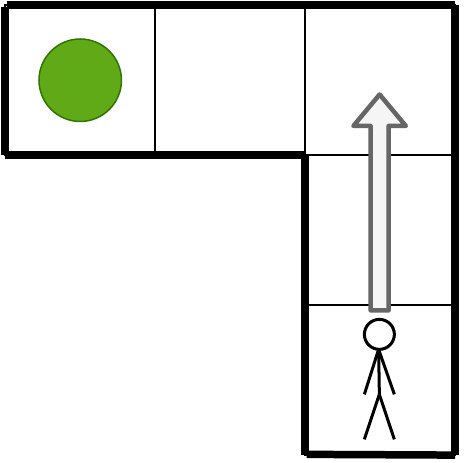}
    \caption{}
    \label{fig:common_subtasks_a}
  \end{subfigure}
  \hfill
  \begin{subfigure}[t]{.3\textwidth}
    \centering
    \includegraphics[scale=.3]{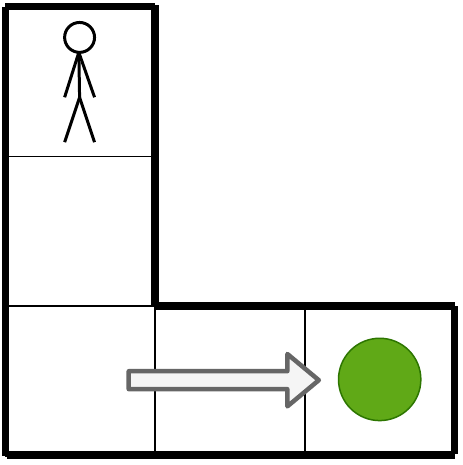}
    \caption{}
    \label{fig:common_subtasks_b}
  \end{subfigure}
  \hfill
  \begin{subfigure}[t]{.3\textwidth}
    \centering
    \includegraphics[scale=.3]{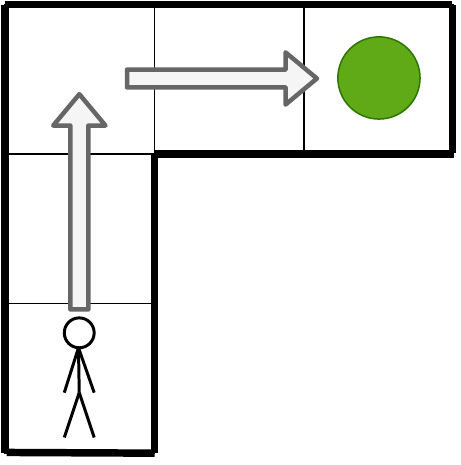}
    \caption{}
    \label{fig:common_subtasks_c}
  \end{subfigure}
  \hfill
  \caption{
    Two tasks combined contain all of the subtasks
    that make up a third task.
  }
  \label{fig:common_subtasks}
\end{figure}

Of course this only deals with part of the entire task at hand so far
--
we are only dealing with a single subtask after all.
We can take this one step further
and consider multiple subtasks.
If we could just piece together policies
for subtasks
--
each a partial solution for the whole task
--
we could recover a complete policy.
The problem is that any two tasks
are unlikely to have all of their subtasks in common,
in which case transfer still only yields a partial policy.
We can resolve this by considering multiple
source tasks,
each contributing their share of subtasks.
If in addition to the task
in \figref{fig:common_subtasks_a}
we consider the task in
in \figref{fig:common_subtasks_b},
the two together contain all the subtasks that make up
the task
in \figref{fig:common_subtasks_c}.
Two subtask policies,
one gathered from
each of the first two tasks,
combined
make up a complete policy for the third task.
The overall picture now looks something like
depicted in
the diagram in
\figref{fig:reusable_subtasks}.

\paragraph{}
If subtasks are so effective at
enabling transfer,
what prevents us from considering as many subtasks
as possible?
After all,
the smaller the subtask,
the easier it is to achieve an overlap
and so
the easier it is to transfer policies.
Indeed, why not as many subtasks
as there are different states,
one each,
with the added bonus
of only needing such a coarse minimal state abstraction
that subtasks will even start repeating within the source tasks?
This question is an important one
and we will return to it shortly.

\begin{figure}[ht]
  \centering
  \includegraphics[width=0.5\linewidth]{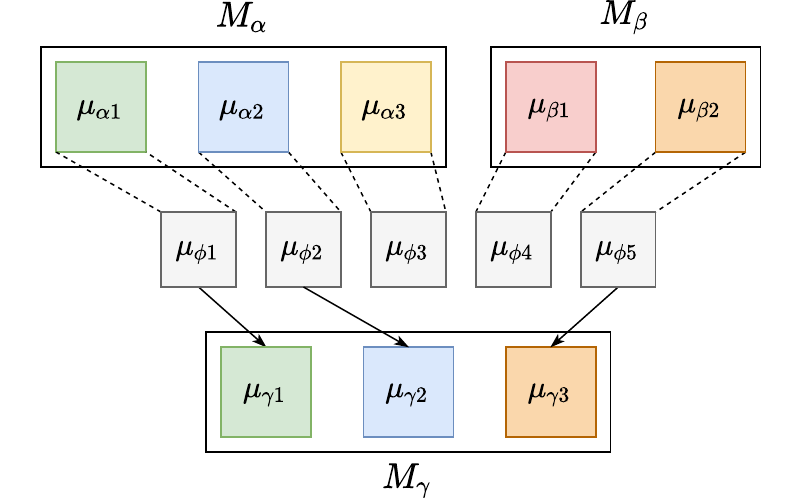}
  \caption{Three tasks are depicted.
    Each task consists of subtasks,
    the abstractions of which are shown
    in the middle.
    The third task,
  $\mdp_\gamma$, is made up of some of the subtasks that recur in
  $\mdp_\alpha$ and $\mdp_\beta$.
  The mapping between subtasks of different tasks
  is effected by state abstractions.}
  \label{fig:reusable_subtasks}
\end{figure}

\subsection{Subtasks as Skills}%
\label{sec:subtasks_as_skills}
We have conceptually sketched
how a task can be seen as being made up of
subtasks,
and how those subtasks can each contribute
a partial policy
for a target task.
We now dig a little deeper to understand
how these subtasks can be used in practice.

\paragraph{}
We approach our
subtasks with
\emph{skills}
(Section~\ref{sec:general_skills}).
Formally,
skills are
just options.
We have opted for the term
\emph{skill}
to add semantics
to the usual notion of an option;
in our eyes,
a skill represents \emph{local behavior}
that is \emph{reusable}.
First,
the \emph{local-ness} of behavior
goes hand in hand with using coarse
outcome equivalent state abstractions:
if a state abstraction
captures only
\emph{local} information,
it will be coarser
than a state abstraction
that captures more global information.
Subtasks should require only
local information.
Second,
we focus on subtasks that are repeated across tasks,
their solutions are \emph{reusable}.
This corresponds to preferring reusable options
rather than options that are useful in only
one specific setting.

\paragraph{}
Consider again the definition
we introduced previously
(Section~\ref{sec:general_skills})
of an option as a triple
$\omega~\defas~\langle \sset_\omega, \pi_\omega, \beta_\omega \rangle$,
respectively denoting
the option's domain,
control--, and termination policy.
This time however,
we choose
$\sset_\omega$
to be abstract
so that a state abstraction
maps into an option's domain.
The hierarchy now has options
operating over an abstract state set
on the lower level
and an option-enabled agent
initiating options
(themselves abstract actions)
on the higher level.
This may be contrary to intuition
demanding that abstract states are used on a `higher' level
--
indeed here it is the lower level
(the option)
which employs state abstraction.
An example trace is depicted in
\figref{fig:hrl_overview_example}
of an option-enabled agent
with two options operating
over different state abstractions.
To interpret this,
note that each horizontal position corresponds to a time step,
circles correspond to ground states
and rectangles to abstract states.
Horizontal arrows indicate state transitions and
are labelled with the primitive action that induces it.
Vertical arrows indicate initiating (down) and terminating (up) an option.
As explained by \citet{Sutton1999},
the higher level can be modelled as a semi-MDP
with variable numbers of time-steps between state observations.

\begin{figure}[ht]
  \centering
  \includegraphics[width=0.9\linewidth]{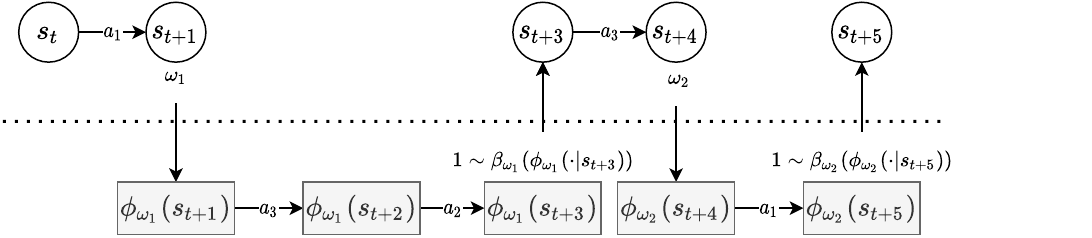}
  \caption{An example hierarchical trace is depicted,
    produced by
    an option-enabled agent
    with primitive actions
    $\{a_1, a_2, a_3\}$
    and options
    $\{
    \omega_1=(\Phi_{\omega_1}, \pi_{\omega_1}, \beta_{\omega_1}),
    \omega_2=(\Phi_{\omega_2}, \pi_{\omega_2}, \beta_{\omega_2})
    \}$
    with corresponding state abstractions
    $\phi_{\omega_1}: \sset \mapsto \Phi_{\omega_1}$
    and
    $\phi_{\omega_2}: \sset \mapsto \Phi_{\omega_2}$
    respectively.
  }
  \label{fig:hrl_overview_example}
\end{figure}

\paragraph{}
We can now relate options to the conceptual notion of subtasks.
First,
an option describes a behavior and not a task.
We can see an option as encoding the behavior for a subtask.
The option's domain,
the state abstraction's image,
corresponds to the subtask's state set.
Second,
a subtask ends at some point,
as governed by the option's termination policy.
Finally,
a subtask presents a task
which the option's policy presents the solution to.
One thing remains:
it is as of yet unclear how the start of a subtask is encoded.
Some approaches enrich an option with an
\emph{initiation policy}
which dictates when an option can be started
(when a subtask begins).
We employ the more common approach
of giving an option-enabled agent control over when to start an option.
This control puts the onus on the option-enabled agent
as the option is now available everywhere
but not necessarily universally useful.

\paragraph{}
It is now worth returning
to the question of
what complexity
this approach adds.
Consider a set of options
that each encode a solution to a smallest
conceivable subtask,
one which only concerns itself with a single state
and so only consists of a single step.
Realistically
there would be as many
different options
as there would be unique outcome-equivalent state abstractions
for one-step tasks,
a number that is combinatorial
in the number of actions.
The option-enabled agent now has control
over all primitive actions,
plus this combinatorial amount of options.
Its action set has increased greatly.
Since the agent will still need to learn
a policy to use each of these actions effectively,
it has gained effectively nothing
--
instead its task has only gotten more complicated.
At the other extreme
is the task consisting of a single subtask
with little to no potential for transfer.
There is need for a balancing act:
the ideal number of subtasks
and their sizes
should provide a benefit that outweighs
the cost of the added complexity.
There is a trade-off between a larger number of
shorter options and fewer longer options.
Short options can transfer behaviors to a large number of
states of a target task (high coverage),
but the higher level agent must still learn when to initiate each option.
A longer option transfers more behavioral information,
but is not as widely applicable.
This trade-off poses a key problem that we will delve
into in Section~\ref{sec:experimental_evaluation}.

\paragraph{}
So far we have left unanswered
the question of how subtasks or options
are specified.
We do not want to take the approach
of manually specifying subtasks
as that would require a large amount of
knowledge and engineering
without a guarantee to be helpful.
This is a question in two complementary parts:
state abstraction and action abstraction.
That is,
each option both
relies on a state abstraction
and represents
an abstract action
through
its control and termination policy.
These then make up the two interdependent great variables of this approach:
which state abstractions and which action abstractions.
We will solve for both in our experiments.

In what is to follow we describe our approach to this problem,
starting with
how we can uncover options
(or indeed -- action abstractions)
when given a set of state abstractions.
We finish with our construction
of coarse outcome equivalent state abstractions.

\subsection{Option Discovery}%
\label{sec:option_discovery}
In order to find a set of suitable options
we park for the moment the problem
of how to find suitable state abstractions.
We assume for now that
we have access to state abstractions
and that they are
both useful
and suitably coarse
to allow for transfer to happen.
The options that we would like to find
are those that are as repeatedly useful across tasks as possible.
Ideally,
we want to find behaviors
for subtasks
that are common
in a domain
--- \emph{reusable skills}.
This would make it all the more likely
that an option is useful
in a new and unseen task of the same domain.

At this point we should clarify our setting.
We assume that we have access to
a small subset of source tasks of a domain
such as in
\figref{fig:reusable_subtasks},
except
the target tasks are unknown beforehand.
Ideally these source tasks are representative for the domain,
for we aim to find options that
are as useful as possible in all of these tasks.
Recall that it is only possible to consider
options that operate in different tasks
thanks to a useful state abstraction,
one that is coarse
and one that ideally equates states
that benefit from the same policy,
i.e. one with a high transfer value such as an outcome equivalent state abstraction.

\paragraph{}
There are a few ways in which option discovery could be approached.
For example,
imagine an agent learning how to navigate
multiple tasks
while continuously updating an option set
that helps replicate its learned behavior.
A certain pattern of action
might arise as the agent is learning:
a natural candidate for an option.
This can be seen as a form of online compression
of the learned behavior,
constantly bootstrapping the agent's learning
and improving replicability of behavior.
We do not adopt this
\emph{online} approach
but we do let it guide us.
We still see option discovery as a compression
of learned behavior,
with one option explaining as much of the demonstrated behavior as possible,
but instead of an online approach
we take a more stratified
and offline approach.
As option discovery is not the main topic of this paper,
we describe our version and
rephrasing of the work of
\citet{Daniel2016}
in detail in
Appendix~\hyperref[app:option_discovery]{B}.
We summarize it here.

Our approach uses a collection of demonstration traces
$\{ \tau_i \}^N_{i=1}$,
where each trace (state-action-reward sequence)
is derived experience
from policies trained independently on one of
a collection of $J$ \emph{source} (training) tasks
$\mdp_j=\langle \sset_j, \aset, p_j, r_j, \gamma \rangle$.
Although these traces are derived from a non-hierarchical agent,
we then model them as if they were derived from a single hierarchical agent
acting across all $J$ tasks.
This hierarchical agent has access to $K$ options
$\Omega_K=\{ \omega_k \}^K_{k=1}$,
each with its own abstraction $\phi_k$
\footnote{Although each option can have its own abstraction,
in our experiments we use the same abstract space for all policies,
suck that $\Phi_k=\Phi$.},
and for each task $\mdp_j$
a separate high-level policy
$\pi_j:\sset_j \mapsto \probset(\aset \cup \Omega_K)$.
Equivalently,
this can be seen as $J$ hierarchical agents
sharing options $\Omega_K$.
Given a parameterization
$\theta_j$
for the control and termination policy of each option $\omega_j$,
the discovery mechanism seeks to find the maximum likelihood
of high-level policies
$\{\pi_j\}^J_{j=1}$
and option parameters
$\{\theta_k\}^K_{k=1}$
to explain traces
$\{ \tau_i \}^N_{i=1}$
as if they were generated by the hierarchical agent,
i.e.
\begin{equation*}
  \argmax_{\{\theta_k\}_{k},\{\pi_j\}_{j}}
  \prod_i^N \Pr(\tau_i \mid K, \{\theta_k\}_{k}, \{\pi_j\}_{j})
\end{equation*}

We always pick the state abstraction
to be strictly coarser
than the ground state representation,
ensuring opportunity for transfer
so that each option can be useful in multiple tasks,
as discussed previously.
As a consequence,
each of the resulting options
may not be able to reconstruct
the demonstrated trajectories exactly,
so that an optimal solution
would typically only be attained when options are combined.
This is exactly what we want:
a task made up of subtasks,
with each subtask reoccurring in many tasks.
We discuss the exact state abstraction
used in the next section.

\subsection{Implementing Coarse Outcome Equivalent State Abstractions}%
\label{sub:implementing_coarse_outcome_equivalent_state_abstractions}
The example in
Section~\ref{sec:illustration_of_a_domain}
showed that
\emph{truncating}
the action sequences employed
in the construction of
outcome equivalent state abstractions
naturally made for coarser state abstractions.
Consider shorter action sequences
and expected outcome sequences
as
resulting in short-sightedness,
leaving local information only.
Truncated outcome-equivalent state abstractions
are perfectly suited
for our options:
each option
has access to a local window of information,
ignoring global and unnecessary detail,
ensuring that states across tasks
that have an identical local window
(the same abstract state)
have access to the same behavior.

In this section
we describe how exactly we construct these truncated
outcome-equivalent state abstractions
for our experiments.
We assume for the experiments
that we have
access to an oracle
which gives perfect information
about future outcomes.
In practice,
outcomes would not be known without some prior exploration
or access to the model,
but for clarity we omit this complicated from our empirical studies.
This omission allows us to assess the
benefits of our approach in isolation,
without in addition needing to learn state abstractions.
We discuss the learning of outcomes in section~\ref{sec:related_work}.

\paragraph{}
Consider again the expected outcome sequence, Equation \eqref{eq:sigma_sequence},
constructed from
an action sequence
$\langle a \rangle = (a_1,\dots,a_n)$.
It contains
the expected outcome
for each transition
$(S_{t+k-1}, a_k, S_{t+k})$
following the actions
$(a_1,...,a_{k-1})$
up to that point (repeated here for convenience):
  \begin{equation*}
    \langle \sigma \rangle (s_t, \langle a \rangle) \defas (\sigma(s_t, a_1, S_{t+1}),\dots, \sigma(S_{t+n-1}, a_n, S_{t+n}))
  \end{equation*}
Given that $\Sigma \subseteq \realnumbers$,
this is a sequence of real values.
To construct a \emph{vector}
of outcomes,
we simply stack
the individual entries.
We employ boldface to denote vectors
(to contrast with the notation for a sequence of expected outcomes).
\begin{equation}
\testvec(s_t, \langle a \rangle) \defas
\begin{bmatrix}
  \sigma(s_t, a_1, S_{t+1})\\
  \vdots \\
  \sigma(S_{t+n-1}, a_n, S_{t+n})
\end{bmatrix}
  \in \realnumbers^{nd}
  \label{eq:test_vec_single}
\end{equation}

The vector containing
all expected outcomes
for action sequences of length $n$
is constructed by stacking
the individual expected outcome vectors.
So for action sequences
$\langle a \rangle_i \in \aset^n$
and $|\aset^n|=N$:
\begin{equation}
\testvec(s_t, \aset^n) \defas
\begin{bmatrix}
  \testvec(s_t, \langle a \rangle_1)\\
  \vdots \\
  \testvec(s_t, \langle a \rangle_{N})
\end{bmatrix}
  \in \realnumbers^{Nnd}
  \label{eq:test_vec_length_n}
\end{equation}
%
We can now
define a state abstraction
based on this vector construction.

\begin{definition}[Outcome-Predictive State Representation]
  Given an MDP
$\mdp = \langle \sset, \aset, p, r \rangle$,
  and outcome function
  $\sigma: \sset \mapsto \Sigma$,
  we define the \emph{Outcome-Predictive State Representation} (OPSR)
  of length $k$
  as the vector of all expected outcome sequences
  of length $k$.
  It is a function
  $\simplephi_{\sigma,k}: \sset \mapsto \Phi_{\sigma,k}$:
  \begin{equation}
    \bm{\simplephi}_{\sigma,k}(s) \defas
      \testvec(s, \aset^k)
    \label{eq:phi_sigma}
  \end{equation}
  We refer to $k$ as the \emph{horizon}.
\end{definition}
Since an OPSR
is parameterized
by the \emph{horizon} $k$,
it is natural to consider
the family of functions
arrived at by varying the
horizon.
For a given state set $\sset$
and outcome function $\sigma: \sset \mapsto \Sigma$,
we get the set
\begin{equation}
  \left\lbrace
    \phi_{\sigma, k}: \sset \mapsto \Phi_{\sigma, k}
    \mid
    k \in \naturalnumbers
  \right\rbrace%
  .
\end{equation}
This set is totally ordered
under the
\emph{finer} and \emph{coarser} relations.
That is,
for $k \in \naturalnumbers$,
$\phi_{\sigma,k+1}$
will be finer than
$\phi_{\sigma,k}$.
If
$\phi_{\sigma,k+1}$
is \emph{not strictly finer than}
$\phi_{\sigma,k}$,
we term
$\phi_{\sigma,k}$
the \emph{finest OPSR}
for $\mdp$ and $\sigma$.

\section{Experimental Evaluation}%
\label{sec:experimental_evaluation}
This section
puts to the test
the framework
that we have developed.
In particular,
we will evaluate
and compare
OPSR-based options
in different settings
with focus on one aspect:
\emph{transfer}.
We consider transfer to be beneficial
if training on a task
after transfer
outperforms
training on the same task
without transfer.
This means that for our experiments
we set aside the issue of
\emph{source task time}
and consider it a sunk cost
\citep{Taylor2009}.
This can be justified
if we see source task time as an
\emph{amortized cost}
that decreases upon each reuse
of the learned material
--
something we rely heavily on in our experiments.
That said,
we will however also keep an eye
on jumpstart and asymptotic performance,
which we will explain in the accompanying discussions.

We evaluate our methods
on two domains.
The Craftworld domain is mainly used to explore different parameters
and gain insights into
the performance of OPSR-based options.
The Lightworld domain was originally introduced by
\citet{Konidaris2007},
who proposed \emph{agent-spaces}:
manually engineered abstract state spaces
amenable to transfer.
We compare our general method directly
agent-space--based options.
Finally,
we inspect the potential
for transfer
between different subdomains.

\subsection{Experimental Structure}%
\label{sub:experimental_structure}
%
The previous section
already discussed
our Outcome-Predictive State Representation
construction
and our option discovery process
that builds on this
method of state abstraction.
In particular,
we constrain all options for a particular experiment
to share an abstract space
$\Phi_{\sigma,k}$.
A few more parameters remain for a given experiment:
the choice of outcome function $\sigma$,
the number of options,
the OPSR horizon $k$,
and a set of tasks used to discover options from
(training tasks)
as well as as a set of tasks used for evaluation
(target tasks).

\paragraph{}
An experiment is organized as a
straightforward two-step process.
First,
a set of training tasks is learned to completion.
From the demonstrated trajectories,
a set of options is learned that
maximizes their likelihood
(i.e.,
that is most likely to give rise
to the demonstrated action sequences
-- see Appendix~\ref{app:option_discovery}
or a brief summary in Section~\ref{sec:option_discovery}).
Finally, the resulting options
are evaluated on a set of  tasks
by inspecting the learning speed
of the hiararchical agent that has access to these options
as compared to a non-hierarchical agent.
While this setup resembles
may reminds the reader of multi-task RL,
this is simply for demonstration purposes
and our method is by no means limited to
the multi-task setting
where a set of source tasks is available.
The procedure is in fact designed
to closely imitate
the online learning agent
sketched in
Section~\ref{sec:option_discovery},
which no longer requires source or target tasks.
The entire procedure is outlined in
Procedure~\ref{alg:evaluation_protocol}.

\floatname{algorithm}{Procedure}
\begin{algorithm}[htb]
\caption{Option Discovery and Evaluation Protocol}
\label{alg:evaluation_protocol}
\begin{algorithmic}
  \State Sample $N_\textsc{train}$ training tasks
    $\mathcal{M}_\textsc{train} =
    \mdp_1,\dots \text{~from domain~} \domain$
  \State Sample $N_\textsc{test}$ target tasks
    $\mathcal{M}_\textsc{test} =
    \mdp_1,\dots \text{~from domain~} \domain$
  \State $T \gets \{\}$
  \For{$\mdp\ \text{in}\ \mathcal{M}_{\textsc{train}}$}
    \State Gather trajectory $\tau \gets \textsc{SolveMDP}(\mdp)$
    \State $T \gets T \cup \tau$
  \EndFor
  \State $\{\omega_1,\dots,\omega_o\} \gets
    \textsc{DiscoverOptions}(\mathcal{M}_\textsc{train}, T, \phi_{\sigma,k}, N_\textsc{options}=o)$
  \State Augment action sets of target tasks $\aset =
    \aset \cup \{\omega_1,\dots,\omega_o\}$
    \For{$\mdp\ \text{in}\ \mathcal{M}_{\textsc{test}}$}
    \State $\textsc{LearnMDP}(\mdp)$
  \EndFor
\end{algorithmic}
\end{algorithm}

\subsection{The Craftworld Domain}%
\label{sub:environment_craft_world}

\begin{figure}[ht]
  \centering
  \hfill
  \includegraphics[width=0.3\linewidth]{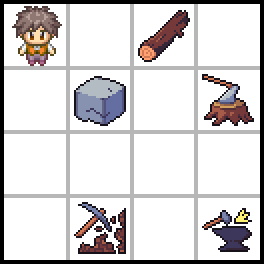}
  \hfill
  \raisebox{.5\height}{%
  \includegraphics[width=0.35\linewidth]{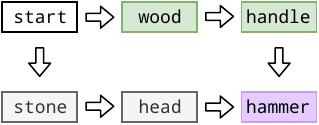}}
  \hfill
  \hfill
  \caption{Left: depiction of a Craftworld task.
  The diagram on the right illustrates crafting dependencies.}
  \label{fig:craftworld}
\end{figure}

The first environment is based on the
Craftworld introduced by
\citet{Andreas2017}.
It is implemented in VGDL%
\footnote{VGDL 2.0: a Video Game Description Language
(\url{https://github.com/rubenvereecken/py-vgdl})}
as the framework allows easy generation of new tasks
due to simple plain-text task specifications
that can be automatically generated.
An agent is randomly placed in a grid
where it needs to craft an end product.
To do so, it must gather resources
and refine these at their corresponding workstations.
It can pick up resources by walking over them
and interact with an appropriate workstation
by facing it
and performing a \texttt{use} action.

We restrict ourselves to one Craftworld domain
where the agent's
end goal is to craft a \emph{hammer},
first creating the intermediate
\emph{handle} and hammer \emph{head}.
One can see in the example
depicted in
\figref{fig:craftworld}
that each task has at least two optimal sequences of actions
since subtasks can be completed in any order.

An agent is to complete its objective in as few steps as possible.
To this end it receives a penalty of -1 on each action taken,
except its final action which is rewarded with 1000.
We restrict ourselves to 4 by 4 grids.

\subsubsection{Types of Agent}
We use two types of reinforcement learning agents:
a regular agent that has access to primitive actions only,
referred to as the \emph{option-less agent},
and an \emph{option-enabled, hierarchical agent}
that in addition to primitive actions
can also initiate options.
Agents use SARSA($\lambda$)
with $\epsilon$-greedy policies.
We fixed $\epsilon=0.05$,
$\lambda=0.99$,
$\gamma=0.999$
and set the learning rate to $0.2$,
all as results of a hyper-parameter sweep.

\subsubsection{State Representation}
The base state representation
is a fully Markov state representation
where each state is encoded separately:
a one-hot representation
which we will denote $x(s)$.
This is used by the option-less agent
as well as the high-level policy of the hierarchical agent.

The option's state representation
is an OPSR
$\phi_{\sigma, k}(s)$
as defined in
Equation~\ref{eq:phi_sigma}.
This leaves us to define
the outcome function $\sigma$
and a maximal action sequence length $k$.
We will define outcomes for all changes in resource counts
in order to capture every time the agent
gains or loses a resource
(by crafting it into another one).
Using
$\texttt{resource}(s)$ to denote the number of resources
of type \texttt{resource} the agent has in state $s$,
and defining the short-hand
$\Delta \texttt{resource} \defas \texttt{resource}(s')~-~\texttt{resource}(s)$%
,
we can write the outcome function
for our Craftworld domain as
\begin{equation*}
  \sigma(s,a,s') = (
  \Delta \texttt{wood},
  \Delta \texttt{stone},
  \Delta \texttt{handle},
  \Delta \texttt{head},
  \Delta \texttt{hammer}
  )
\end{equation*}
This will capture what is relevant to reward
as well as express information that we think
could be relevant to options.

\subsubsection{Option Discovery}
We derive traces from training tasks using Policy Iteration.
All control policies
and option termination policies
are parameterised by single layer neural networks
with a softmax activation.
The high-level policy is denoted
$\pi_H$.
\begin{align*}
  \pi_H(s) &= \textit{softmax}(Wx(s)) \\
  \pi_\omega(s) &= \textit{softmax}(W \phi_{\sigma,1:k}(s) + b) \\
  \beta_\omega(s) &= \textit{softmax}(W \phi_{\sigma,1:k}(s) + b)
\end{align*}
Calculation of the gradient
is outlined in
Appendix~\hyperref[app:option_discovery]{B}
while stepping of the gradient
was performed with
the Adam algorithm \citep{Adam}
using a learning rate of $0.3$.
Trajectories were sampled randomly from the training set,
one epoch seeing each trajectory once.
Learning continued until
the joint likelihood of the traces
reached $0.99$
or until convergence.

To reduce the size of the sizeable
prediction-based state spaces,
we first apply principal component analysis
and keep only those features that together
explain 99\% of variance
in the training set.

\subsubsection{Experimental Structure}
We generated 100 Craftworld tasks
of size 4 by 4,
all with the same structure as outlined in
\figref{fig:craftworld}.
To evaluate performance of our option-enabled agent
on a task,
we ran option discovery 10 times
on a single trace each of the remaining 99 tasks.
For each option discovery outcome,
we train the option-enabled agent 10 times for 100 episodes each,
totalling 100 samples per target task.
Results are averaged over all 100 target tasks,
totalling 10000 samples per result.

\subsubsection{Evaluation: Expressivity of Abstraction}
\begin{figure}[t]
  \centering
  \includegraphics[width=.7\linewidth]{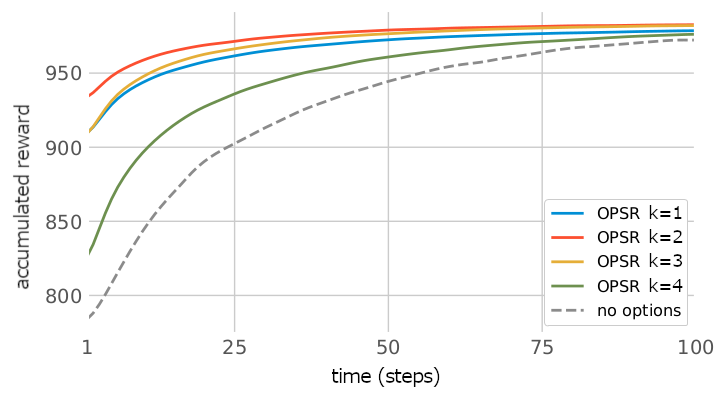}
  \caption{Learning curves for
  an agent without options
  and agents with three discovered
  OPSR $\phi_{\sigma,k}$ options,
  with $k$ ranging from 1 to 4.
  A higher accumulated reward is better.
  }
  \label{fig:craftworld_3options}
\end{figure}
What is the impact of the expressivity
of our chosen state abstraction?
In other words,
how much information should we include?
The fewer expected outcomes we include in our state abstraction,
the coarser it is,
the more states get aggregated and the more we can reuse behavior,
yet the less information is available to base the behavior on.
We investigate this question by varying the parameter
$k$ in $\phi_{\sigma, k}$.

\figref{fig:craftworld_3options} shows
accumulated reward per epoch for the primitive agent
as well as for 4 settings of the option-enabled agent,
corresponding to $k \in \{1,2,3,4\}$.
It also shows the learning curve of an agent without access to options.
The option-less agent is outperformed by each instantiation
of our option-enabled agent,
with a learning curve that is shallower at each epoch.
Even options based on an OPSR with $k=1$
do better,
from which we can conclude that
even a very restricted state abstraction
can be expressive and general enough to learn transferable behavior on.

The learning curve of the setting with $k=3$
outperforms any other one at each epoch.
Interestingly,
the setting $k=4$ shows decreased performance.
There appears to be a trade-off between complexity and performance,
where past a certain threshold more complex representations
incur penalties.
Of course,
considering action sequences of length $k=4$
adds four times as many features.
We think this increase in features resulted in more distinct states,
which require more learning.
At a given point, there simply is a need for more data,
i.e.\ demonstrated trajectories.

\subsubsection{Evaluation: Number of Options}
\begin{figure}[ht]
  \centering
  \includegraphics[width=.7\linewidth]{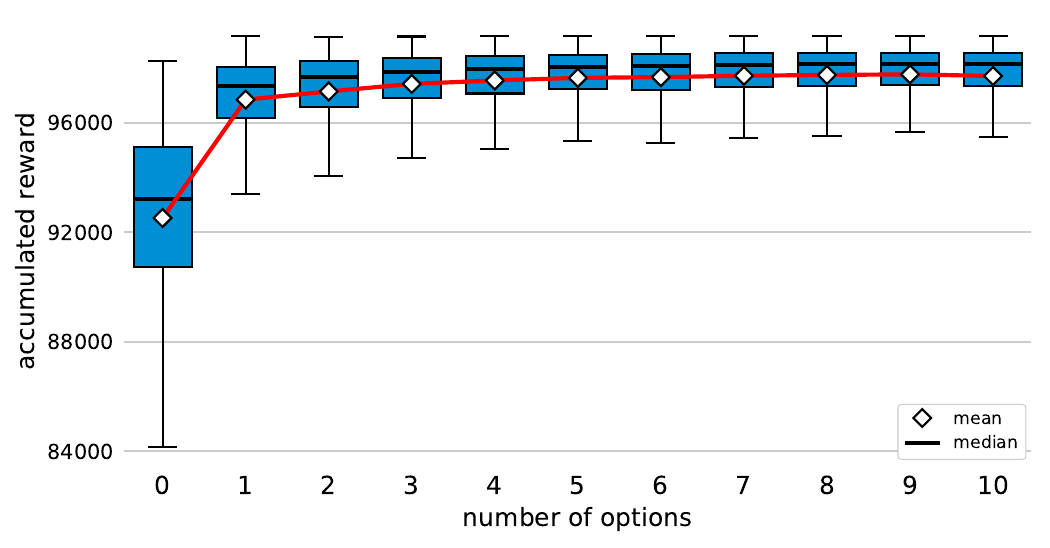}
  \caption{Craftworld for OPSR 1-3.}
  \label{fig:craftworld_pro3}
\end{figure}
Since we pose no restrictions on the behavior
of options aside from the
bias implicit in parameterisation,
it is interesting to investigate the different behaviors
emerging with different amounts of options in play.
To describe the performance of an agent
we can condense and summarize a single learning curve
by simply looking at the area under that curve.
We can then look at the distribution of
areas under the reward curve (AURC) across runs.

\figref{fig:craftworld_pro3}
shows such distributions of AURC,
one for each setting,
each setting corresponding to a different number of options.
It shows that even a single discovered option
can boost learning as compared to learning without options.
Increasing the number of options further helps learning
up to a point where it saturates.
There exists a trade-off in using options:
options encode useful and reusable behavior
yet come at the cost of a larger action space
for the high-level agent,
which in turn requires more learning
to handle.
This is exactly the trade-off we discussed earlier in
section~\ref{sec:skills_that_transfer}.

\subsubsection{Option Trace Inspection}
We have shown that our options
help speed up learning,
yet it is unclear how exactly they are helping.
The outcome-based abstraction is complex and hard to interpret
and an option's behavior is learned entirely from scratch.
In order to gain some insight in how exactly options helped,
we ran option discovery
once
for $4$ options with an OPSR
with $k=3$.
We then want to investigate how this set of learned options
tends to be used
across multiple runs
of learning an option-enabled policy.
In particular,
we want to investigate whether these options implemented generic
all-purpose behavior that would be useful anywhere,
or whether they implemented more specific behavior
for a particular subtask
that the option-enabled agent could rely on at the right time.
In order to test this,
we trained an option-enabled agent
to completion
on the Craftworld task
depicted in
\figref{fig:craftworld},
resulting in one final, ideally optimal trace.
We repeated this a number of times
to account for randomness,
seeing if different behaviors emerged given the same set of options.

In order to investigate the resulting trajectories we devised a method
to extract trajectories that are somehow \emph{similar}
so we do not compare option use between wildly different trajectories.
To this end we define a trajectory's \emph{story}
as the sequence of
non-zero outcomes in that trajectory.
For our investigation of option use
it is sufficient that we look at trajectories with the same \emph{story}.
We extracted the stories for all final trajectories,
then randomly selected $100$ trajectories that converged to the dominant story,
each from a different option-enabled agent that was trained from scratch.
The dominant story for this set of options
was
$(\texttt{stone}, \texttt{head}, \texttt{wood}, \texttt{handle}, \texttt{hammer})$,
i.e.
crafting the hammer head first,
then the handle from scratch.

For each trajectory,
we looked at which option was in use at which time step.
\figref{fig:craftworld_option_stats}
combines all $100$ trajectories
by showing for each time step
how many out of the 100 trajectories
spent that time step in which option.
We refer to the resulting graph as
the \emph{option occupancy graph}.
The optimal trajectory consists of $14$ steps,
which is why the count drops significantly from time step $15$ onwards,
with the longest trajectory taking 24 time steps.
\begin{figure}[h]
  \centering
  \includegraphics[width=0.7\linewidth]{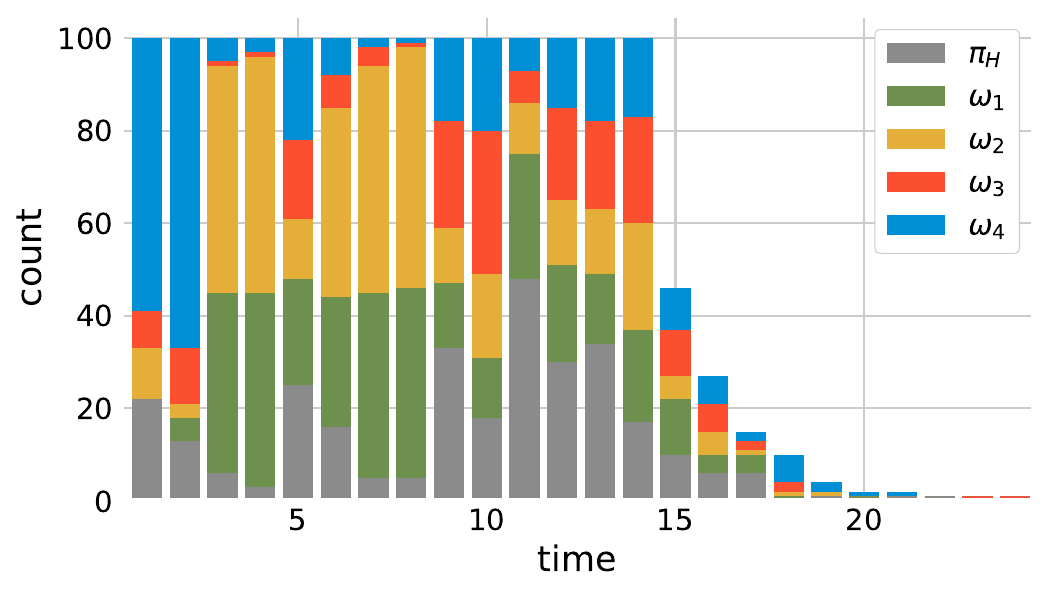}
  \caption{Option occupancy for $100$ traces following the same story.
  $\pi_H$ denotes when the high-level controller
  picks a primitive actions instead of an option.}
  \label{fig:craftworld_option_stats}
\end{figure}

There are some clear patterns in
\figref{fig:craftworld_option_stats}.
Option $\omega_4$
seems to be clearly useful at the start of the task,
with nearly 60 trajectories
spending the first time step in option $\omega_4$.
It is followed by options
$\omega_1$ or $\omega_2$
(and, on closer manual inspection of individual trajectories, sometimes both).
This interesting splitting of behaviors
turned out
to be due to
the multiple equally optimal ways
the agent can navigate from one point to the next,
despite following the same story.
Upon inspection,
each of $\omega_1$ and $\omega_2$
appeared to deal with one such an optimal route.
Option $\omega_3$ was used more towards the end
of the trajectory,
though so was the primitive behavior
which appeared to be preferable to options
during that period of an agent's lifetime.

If options did not specialize behavior and instead
were useful everywhere,
we would expect to see a very uniform option occupancy graph
with each color roughly equally present at each time step.
The option occupancy graph looks far from uniform.
Seeing the same patterns reoccur
in differently initialized
option-enabled agents
demonstrates that each option implemented
a certain behavior
that was useful for different parts of the greater task.

When we re-ran option discovery with different seeds,
quite different option behaviors emerged.
This was apparent from the dominant story
being different (yet, again, optimal),
along with very different option occupancy graphs.
This is because option discovery is an under-constrained problem;
many discovered option sets fit the demonstrated trajectories equally well.
This is necessary by design;
by limiting expressiveness,
we prevent options from overfitting on particular traces,
ensuring that they are reusable across tasks.

\subsection{The Lightworld Domain}%
\label{sub:domain_lightworld}
The Lightworld domain was conceived by \citet{Konidaris2007}
who tailored it to demonstrate
their theory of
\emph{agent-space} representations.
We treat it here because of the similarity
with our work.
An \emph{agent-space} representation
is quite loosely defined
as an abstraction that captures local information
centered on the agent.
An agent-space abstraction
should be engineered to be coarser than a fully Markov representation
for the same reason we employ coarse OPSRs:
to enable reusable behavior across a task.
In fact,
agent-spaces shine when the space is the same across tasks
so behavior can be reused across tasks.
From this description
outcome-predictive representations
{are} themselves
agent-space representations,
as they are by definition centered on the agent;
expected outcome sequences are based on action sequences
that always start from the agent's current state.
Since agent-spaces are fully engineered,
predictive representations can be seen as a generalisation
that allows for a more principled
and even automatic construction
of agent-spaces.

\paragraph{}
A Lightworld task exists of
a number of rooms in succession.
Exiting the final room means success.
Aside from doors
a room can contain keys and locks.
Some rooms do not contain keys,
in which case the lock can be opened without them.
Keys are picked up by navigating over them
and employing a
\texttt{pickup} action.
Locks require an
\texttt{unlock} action.
\begin{figure}[ht]
  \centering
  \includegraphics[width=0.35\linewidth]{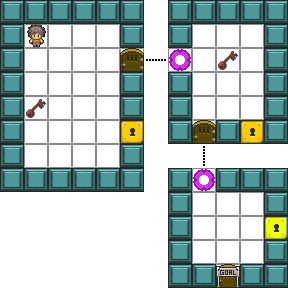}
  \caption{A Lightworld task with three rooms,
    implemented in VGDL 2.0.
    The same task description can be found in \citep{Konidaris2007}.%
  }
  \label{fig:lightworld}
\end{figure}
\subsubsection{State Representations}
We again employ a one-hot Markov state representation
for both the option-less and the option-enabled learner.
For options we consider two different state abstraction constructions:
the agent-centered representation described by \citet{Konidaris2007}
and our prediction-based OPSR.

\citet{Konidaris2007}
consider the agent as having sensors which can sense notable things
in its environment.
The three objects that can be sensed
are key, lock, and door.
Sensors only work in a straight line
and encode their findings as a value ranging from 0 to 1;
0 when 20 tiles away and 1 when on top.
Given the four cardinal directions this results in 12 values.

Our Lightworld OPSR is based on the following outcome function:
\begin{equation*}
  \sigma(s,a,s') = (
  \Delta \texttt{room},
  \Delta \texttt{key},
  \Delta \texttt{door},
  \Delta \texttt{goal}
  )
\end{equation*}
\subsubsection{Experimental Structure}
We closely follow
the experimental protocol where Lightworld was first conceived
to allow for a fair comparison between approaches.
We randomly generated 100 Lightworld tasks
with 2 to 5 rooms each,
each room 5 to 15 tiles wide and high,
and each containing a key with probability $1/3$.
The protocol to generate samples is the same as before.

\subsubsection{Evaluation: Agent-Space and OPSR Options}
We investigate discovered options
for agent-space and outcome-based state abstractions.
In both cases
we only show results for 4 options
as that was the optimal number of options
according to a parameter sweep
(different from the 3 predefined options in \citep{Konidaris2007}).
\figref{fig:lightworld_by_nrooms_boxplots}
presents results
for target tasks with differing numbers of rooms.
Note that even though performance is grouped by the number of rooms
in the target task,
options were still learned from all other tasks
regardless of the number of rooms.

In smaller tasks with only two rooms,
using either type of option does not improve learning.
It is only as tasks grow more difficult
(with more rooms)
that the benefit of options outweights their cost,
with tasks consisting of five rooms benefiting the most.
Interestingly,
our OPSR-based options outperformed
options based on the manually engineered agent-space,
despite the agent-space describing longer action sequences
and so having information about more steps into the future.
It appears that OPSRs capture
essential local information that
benefits reusable options,
more so than a state abstraction designed by an expert.
Looking at the setting with 5 rooms,
OPSR-based discovered options
also constitute the only approach that nets a positive return on average,
with the option-less agent evidently struggling.
These reusable options appear
crucial for tackling a task of only moderate complexity.

\begin{figure}[ht]
  \centering
  \includegraphics[width=.7\linewidth]{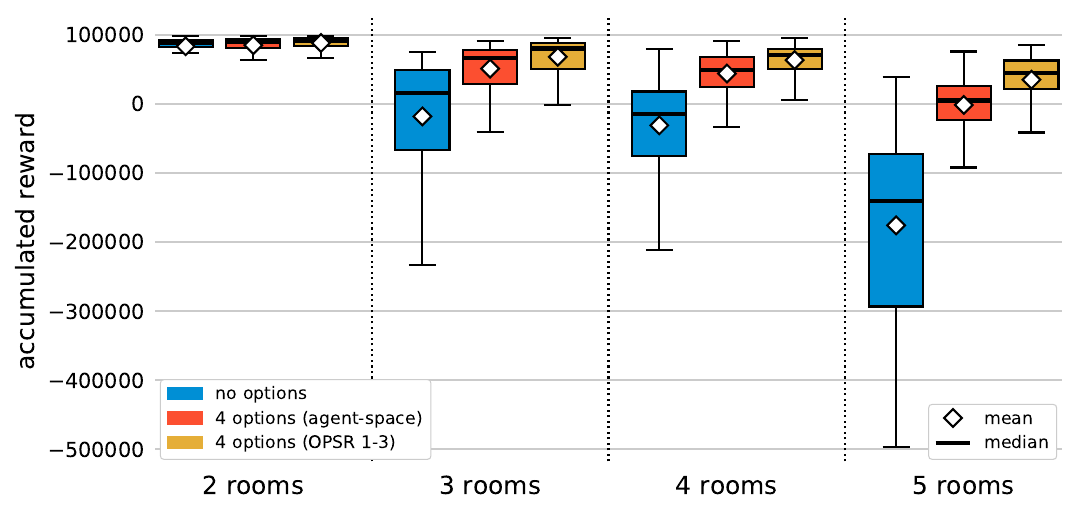}
  \caption{%
  Distribution of areas under the reward curve for target Lightworld tasks
with different numbers of rooms. Three settings are shown:
without options, with discovered OPSR options and with discovered agent-space options.}
  \label{fig:lightworld_by_nrooms_boxplots}
\end{figure}
\subsubsection{Evaluation: Transfer Between Subdomains}
\begin{figure}[ht]
  \centering
  \includegraphics[width=0.7\linewidth]{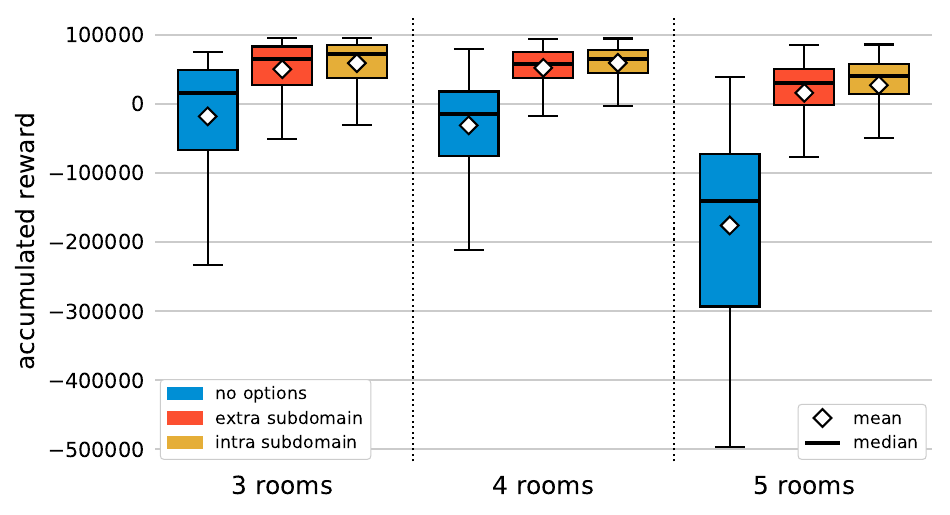}
  \caption{Distribution of areas under the reward curve in the Lightworld domain
    with transfer between and within subdomains.
    Results labeled `extra subdomain'
    have always been learned from tasks with only 2 rooms,
  whereas results labeled `intra subdomain'
have been learned from tasks with the same number of rooms as indicated.}
  \label{fig:lightworld_subdomains_transfer_boxplots}
\end{figure}
The Lightworld domain consists of tasks
made up of 2 to 5 rooms.
We could say that all tasks with the same number of rooms
belong to the same
\emph{subdomain}.
This raises the question:
how well does transfer work between different subdomains?
That is,
given demonstrated trajectories from one subdomain,
how well do the options learned from these trajectories
perform on tasks from another subdomain?

We have decided to look into this question
using options learned from tasks with 2 rooms,
evaluated on tasks with 3 to 5 rooms,
(\emph{extra-subdomain} transfer).
This is more than simple transfer between subdomains;
a task with more rooms is decidedly more complicated
than its room-deficient cousins,
as can be seen from the results of the previous experiment
in \figref{fig:lightworld_by_nrooms_boxplots}.
This allows us to investigate
whether options learned in simpler tasks
perform in more complicated tasks.

Figure~\ref{fig:lightworld_subdomains_transfer_boxplots}
shows the performances of
\emph{extra-} against \emph{intra-subdomain} transfer.
The latter refers to the case where options are learned from tasks
with the same number of rooms.
It is clear that extra-subdomain transfer works:
options learned in simpler tasks are useful in more complicated tasks,
despite never having encountered one before.
What is more,
options learned from a simpler subdomain almost match the performance
of options learned within the same, more complicated subdomain,
as can be seen from comparing extra-subdomain with intra-subdomain transfer boxplots.

These results point at the interesting possibility of \emph{curriculum learning}.
To solve a complicated task where one has access to simpler ones,
one could learn options in progressively harder tasks
in a \emph{curriculum} fashion,
each set of options making the next subdomain easier to tackle.

\section{Discussion of Related Work}%
\label{sec:related_work}

\subsection{Transfer Through State Abstraction}%
\paragraph{Transfer Setting}
Transfer is hard
and generally not very well understood
in the context of reinforcement learning.
One indicator of this
is the lack of a common and concise language
for describing a transfer setting,
for example in terms of the allowed differences
between tasks.
Some papers
allow differences only in the reward dynamics
(e.g. goal location;
\citep{Barreto2017}),
others only
changes in what might be referred to as
\emph{layout}
while keeping
the (loosely termed) \emph{dynamics}
the same,
and others again
only slight changes in
the \emph{dynamics}
(e.g. actuator strength;
\citep{Yang2020, Barekatain2020, Lee2020}).
All of the above examples
are cast under the same nomer of
\emph{transfer}
and it can be a considerable
undertaking to figure out the transfer
setting in a particular work.
We preface with this
because our area of transfer is a fairly general one,
with OPSRs working across
different transition dynamics, layout,
and even reward dynamics.

\paragraph{Underlying Transfer Method}
All transfer
is by definition
the result of some mapping,
implicit or otherwise.
This paper has focused on transfer
through state abstraction,
where the mapping
is an equivalence relation
resulting from states sharing an abstraction
or between abstraction and ground states.
This mapping can be relaxed
--
many papers
achieve partial transfer
through similarity (rather than equivalence)
based on some
(often implicit) measure
(see also discussion in Section~\ref{sub:representation_and_abstraction}).
We found that many papers that we will discuss shortly
rely on state abstraction
(or representation)
to achieve transfer
so we will not discuss
in detail
less common transfer
approaches based on other types of mapping
(e.g. the task mappings of \citep{Taylor2007b}).
We also largely focus on transfer in model-free RL
rather than model-based RL
\citep{Tirinzoni2020, Winder2020, Perez2020}.

\paragraph{State Abstraction per se}
Knowledge transfer in reinforcement learning
through state abstraction
is not a particularly new notion
\citep{Andre2002},
relying on a long tradition of abstraction in planning
\citep{Amarel1981}.
State abstraction tends to be seen
as state aggregation
\citep{Li2006},
as the removal of distinguishing features,
two states appearing the same as a result.
This traditional view still stands,
and can for example be found
in the more recent work of
\citeauthor{Akrour2018}
on robot manipulation
\citep{Akrour2018}
and
\citeauthor{Abel2016}
on lifelong learning
\citep{Abel2016, Abel2018a}.
Finally,
\citeauthor{Frommberger2010}
describe different types of abstraction
and stress that abstraction
is not merely about reducing data,
but also about focusing on relevant information
\citep{Frommberger2010}.

\paragraph{Agent-Space and Deictic Representations}
Our state abstractions
that describe possible behaviors
hark back to the
\emph{deictic} (roughly meaning \enquote{centered on the agent}) representations
of \citeauthor{Agre1987},
who realized
that \emph{ego-centric} state representations
capture general concepts
\citep{Agre1987}.
\citeauthor{Konidaris2007}
picked up this idea again
and termed their
own agent-centric
state abstractions
\emph{agent-space representations}
\citep{Konidaris2007}.
While agent-spaces are useful for transfer,
their problem lies with their manual engineering
\citep{Konidaris2012a}.
Problems aside,
task-agnostic and sufficiently coarse,
agent-spaces perfectly capture what it takes to transfer beneficially.
It is for this reason that we think that so much related work
has implicitly opted for agent-space--like representations
that they ultimally owe their ability
for transfer to
\citep{Chalmers2018,Ferret2020}.
Others have attempted to learn
shared representations
which inevitably must have commonalities
with agent-spaces
\citep{Roy2020,Wu2020}.

\paragraph{Successor Representations}
While perhaps not its own category of representation,
successor representations
(originally by \citep{Dayan1993})
deserve their own discussion
for the frenzy of related work that they sparked.
In the paper
that blew new life into the idea,
\citeauthor{Barreto2017}
introduced
\emph{successor features},
representations
(themselves learned on top
of a given agent-space--like representation)
that enabled transfer
between tasks with the same transition dynamics
but different reward dynamics
by learning a value function
independent of the task's reward weighting
\citep{Barreto2017, Barreto2018}.
Since then,
the framework has been
insightfully compared to model-based RL
\citep{Lehnert2019},
had its applicability for transfer extended
\citep{Madarasz2019},
and been extended with options
\citep{Ramesh2019, Barreto2019}.

\paragraph{Relational Representations}
A traditional idea
that has recently gained more traction again
is that
\emph{relational representations}
are required for generalisation
\citep{Dzeroski2001}.
It has even been used to achieve
transferable options
\citep{Croonenborghs2007}.
\citet{Hamrick2018}
concluded that
a \emph{relational inductive bias}
is elementary for generalisation
in experiments not unlike our own.
This premise is for example explored
by
\citep{Zambaldi2019}
who implemented it in a deep learning context,
though the authors note that it is not yet fully understood
why relational representations afford generalisation.
Also relying on graph networks,
\citet{Garg2020}
learned generalized policies
for relational MDPs,
expressing their hopes
for the future of relational representations.
We posit that all relational representations
are predictive in nature
and suspect that
they can be translated into our framework of outcomes.
In the example of the Box World of \citep{Zambaldi2019}
where the agent needs to open a sequence of boxes with corresponding keys,
a high weight between a key and a box corresponds
to the \emph{outcome} of the chest opening
after picking up the key
and walking up to the box.

\paragraph{Predictive Representations}
There have been multiple avenues of work
pursuing predictive representations in RL.
\citeauthor{Littman2002}
sparked an interesting series of research
on Predictive State Representations in
Partially Observable MDP's
\citep{Littman2002, Singh2004}.
Based on the idea of
\emph{tests}
\citep{Rivest1993},
they represented a state
in a POMDP
in terms of predictions of observations
given a history of observations.
These representations were sufficient to describe
the underlying hidden states.
Though different from our setting
in the task descriptions
and the use of observations,
\citet{Rafols2005}
showed that PSRs enabled transfer
and even showed that there was a trade-off
in the expressivity of the representation
and the benefit to learning.
PSRs have even been extended with options
\citep{Wolfe2006}
and combined with relational knowledge
\citep{Wingate2007b}.
PSRs found new life
through the spectral learning approaches
of
\citeauthor{Boots2012}
(culminating in their thesis; \citeyear{Boots2012}),
which still makes up an active area of research
\citep{Kulesza2015}.

In an originally different avenue of work,
\citet{Sutton2009}
claims that all knowledge is predictive in nature.
\citet{Sutton2011}
later
continue this idea in a practical setting
where a robot learns to predict sensorimotor signals
as \emph{general value functions}.
\citeauthor{Sutton2015}
build on this and investigate whether
the approach can be used to learn PSRs,
with promising results
\citep{Sutton2015}.
\citeauthor{Modayil2014}
linked this work to
\emph{nexting}
\citep{Modayil2014},
the phenomenon of humans and animals
continually making a large number of predictions about
what happens around them
on multiple timescales,
and demonstrated a scalable way to implement
the same phenomenon in a robot.
Since the purpose of nexting is not yet entirely understood,
we suppose that nexting
could be the mechanism
through which animals
represent abstract concepts,
in line with the claims of
\citet{Sutton2009}.

Leveraging predictions
for state representations
is still a very active area of research
to this day
\citep{Misra2020,Guo2020,Schwarzer2021,Lee2020b}.
It is because of this entire body of work
regarding predictive representations
that we are optimistic
about the feasibility of learning OPSRs.

\paragraph{Affordances}
Finally,
there has been some interesting
work on \emph{affordances},
originally defined by
\citeauthor{Gibson1977}
as the possible interactions
an organism can detect in its environment
\citep{Gibson1977}.
Despite inconsistencies
in the ontological use of the term in the literature
\citep{Dohn2006, Dohn2009},
work on affordances generally
tackles the question
\enquote{What can I do here?}
\citep{Khetarpal2020}.
Perhaps unsurprisingly then,
the concept of affordances
has inspired work in
rich environments
requiring both navigation and interaction
\citep{Abel2014, Qi2020, Nagarajan2020}.
We discuss affordances in the context of our work
because OPSRs capture what is possible in a particular state
by capturing all future outcomes.
We believe that there is a deep connection
between predictive representations
and affordances
--
one that merits more discussion.

\paragraph{Multi-Task Representation Learning}
Finally,
\emph{multi-task representation learning}
does not subscribe to a particular view
and instead tries to learn good (shared) state representations
in a setting
where one has access to
a distribution over tasks
or even to both a source and target task
\citep{Maurer2016, Ruder2017, Zhang2019, BenDavid2020}.
This approach relies on a considerable amount of given information
which is not always feasible.
The well-defined setting does give rise to interesting
theoretical investigations
into using representations that are shared across tasks
\citep{Wu2020, Deramo2020, Du2020b}.

\subsection{Hierarchical Reinforcement Learning}%
\label{sec:related_work_hrl}
\paragraph{HRL per se}
The main goal
currently under investigation in hierarchical reinforcement learning,
i.e.\ RL relying on action abstraction,
is to make longer and more complex tasks more tractable
by virtue of task decomposition
(be it through options, skills, abstract actions, macros, etc)
\citep{wen2020efficiency}.
One way HRL delivers on this is
targeted exploration
\citep{Jinnai2019b, Jinnai2020}.
Another way is more consistent, targeted action
that aims to reach subgoals,
an endeavour that has sparked a series of work
in \emph{subgoal} discovery
\citep{Machado2017, Nair2019, Goyal2019, Jurgenson2020,wen2020efficiency}
and even
\emph{subgoal direction} discovery
\citep{Vezhnevets2017, Nachum2019}.
While abstract actions
are usually combined temporally,
some work also exists on
\emph{composite actions},
leading to a different kind of hierarchy
\citep{Peng2019, Tian2020}.

\paragraph{State Abstraction in HRL}
One part of any HRL approach
is the state space that the abstract action operates over.
It is possible to simply use the original state space,
the same one the high-level agent uses
\citep{
Nachum2018}.
Alternatively,
end-to-end approaches
that learn both high-level
and low-level agents (abstract actions)
can simultaneously learn representations for
the abstract actions to operate over
\citep{Bacon2017, Vezhnevets2017}.
These option-specific representations
can be seen as state abstractions.
In contrast with previous work,
\citeauthor{Nachum2019}
learn the low-level representation from unsupervised objectives
\citep{Nachum2019}
This representation transfers to different tasks,
although this representation transfer does not result
in policy transfer.
Of special interest to us,
this work also investigates
sub-optimality of a representation
when used in a hierarchical setting.
Other approaches to
low-level state representations
include
grounding in language
\citep{Jiang2019}
and relying on concealed information
\citep{Vezhnevets2020}.

\paragraph{Transfer in HRL}
Finally
there is the holy grail,
a long-standing dream for HRL:
reusing abstract actions in other tasks
\citep{Barto2013}.
Or,
in the common case where state abstraction is used
as the underlying method:
combining state abstraction and action abstraction
\citep{Konidaris2019}.
Most of the HRL approaches discussed so far
relied on task-specific representations
and as such did not transfer to new tasks.
The insight that state representation is the limiting factor
led to the early proposal
of \emph{agent-space--based options}
by
\citeauthor{Konidaris2007}
\citep{Konidaris2007, Konidaris2012a},
though their agent-spaces (task-agnostic representations)
were engineered
and their options partially pre-defined.
Recent approaches try to learn both a
task-agnostic state abstraction
and the transferable abstract action on top of that.
One such a line of enquiry
focuses on \emph{object-centric representations}
\citep{Topin2015, Zadaianchuk2020}
which are again reminiscent of agent-spaces.
\citeauthor{Modhe2020}
achieve transfer of options by
manually introducing partial observability
(and therefore state aliasing through coarsening)
in state abstractions.
Others attempt to learn a suitable state representation
from a source domain
which tends to need adaptation in a target task
\citep{Frans2018, Peng2019, Qureshi2020}.

Outcome-predictive state representations
and the theory behind them
give an answer to the question
of what makes a state representation useful for transfer,
taking away the need to rely on agent-space--like engineering.
What's more,
combined with options
they also give an answer
to the big question of how to
make HRL work for transfer
as we have demonstrated.
Finally,
work on learning transferable state representations for HRL
is very recent and we believe that
the insights around OPSRs
can guide this area of research in the right direction.

\subsection{A Cognitive Perspective and Motivation}


\citet{Wheeler1999}
argue that much of complicated animal behavior
emerges from a complicated
interplay between animal and environment,
where often this behavior is attributed only
to the animal's \enquote{intelligence},
its cognitive functions.
This phenomenon of over-attribution
has been termed
\emph{causal spread}
by
\citet{Wheeler1999}.
In this paper
we heed the caution of \emph{causal spread}
and advocate for a shift away from the
traditional agent that simply interfaces with the world
in favor
of an agent that is fully embedded in the world
(an \emph{embodied, embedded mind}, \citep{Clark2008}).
In the traditional cartesian view
the agent perceives its environment,
reconstructs an internal representation,
and only then leverages this representation for action
\citep{Wheeler2005}.
In contrast
we advocate for perception
that is heavily geared towards
(and guided by)
action primarily,
where perceiving an object for example
does not happen
in terms of objective properties
but rather in terms
of how to negotiate that object.

\citet{Clark1997}
discusses a robot
which employs a map
that is encoded in terms of percepts and actions.
The map is prescriptive,
it describes locations in terms of which movements
would get the robot to different locations
and the robot has no planning module outside the map
-- the map is itself the controller.
\citeauthor{Clark1997}
coined the term
\emph{action-oriented representations}
to describe representations of this nature
\citep[p.~47-51]{Clark1997}.
An action-oriented representation of an object,
instead of describing the object,
immediately encodes how to negotiate the object.
Closely related is
Gibson's notion of \emph{affordances},
which are possible interactions
an organism can detect in its environment
relative to its own skills
\citep{Gibson1977}.
In later work,
\citeauthor{Clark2014}
describes perception as fundamentally
\emph{predictive}
\citep{Clark2014},
delivering action-oriented representations
that,
once modulated,
prescribe how to act and respond
--
that is, representations of
\emph{affordances}
\citep{Clark2016}.
\comment{
  This would be great:
  both predictions
  and affordances have been thought to be great basis for generalizability
  because...
}

In our example of unlocking a door,
an action-oriented representation of a lock
is nothing so descriptive as
\emph{a brass rectangle with a keyhole},
but perhaps as
\emph{fitting a key}
--
the key itself
as \emph{unlocking the door}
or as \emph{fitting the lock}.
Now,
how can an agent know in advance that a sequence of primitive actions
will lead to a complicated behavior such as opening a lock?
It needs to be capable of postulating goals.
Rather than the agent imagining complicated future states,
we follow the caution of causal spread
(of over-attributing intelligence to an agent),
and instead
we think it is the environment which signals
to the agent
when something has succeeded,
signals what is relevant to the agent's subsequent behavior
and its goals.
When turning a key in its lock
we do not model
the `openness' of the door,
be it as a percept or otherwise,
instead we simply anticipate
a click of the lock
or a jolt of the door.
Anticipating select signals
is far cheaper than complex internal modelling
and,
perhaps more importantly,
constitutes a kind of universal knowledge.
All locks click or jolt when opened,
their shapes or sizes matter not.

To formalize this
we introduce \emph{outcomes},
cues derived from transitions in the environment
that are important to the agent.
This is not to be taken too literally:
the environment does not actively signal the agent.
Instead the agent perceives outcomes
in transitions of the environment
because they are relevant to the agent.
Example outcomes are
\enquote{the key slides in the lock}
or
\enquote{the lock clicks open}.
Outcomes are derived from the agent's perception of these events.
They tend to signal the enabling of future behavior:
the key being in the lock allows turning the key,
the lock clicking allows one to successfully turn the door handle.
Importantly,
outcomes form a domain-wide language
which allows us to talk about state transitions
in a way that is completely independent of the underlying
task-specific state.
They capture what is really important
for a particular goal
and ignore what is extraneous.
They form the base of our abstractions.


To achieve an action-oriented representation
like that of
\citet{Clark1997}
we represent states in terms of
the future outcomes they afford.
That is,
we represent states in terms of predictions of what outcomes follow,
following different short term behaviors.
To take our example,
there are a few
combinations of muscular movements
that would successfully
enter a key into a lock,
quite a few more than would result in missing the lock
and perhaps some that end in scratching the door.
Put together,
the representation of a key becomes the union
of all these predicted outcomes
following different behaviors.
We term this representation
an
\emph{Outcome-Predictive State Representation} (OPSR).
It ticks all the boxes
of the previous narrative:
an action-oriented, predictive representation
which captures affordances.

\paragraph{}
In this paper, we both argue and demonstrate that it is possible to leverage this representation
to learn transferable skills.
Two sets of locks and keys may visually look very different,
yet look very similar if we look at them in terms of predicted outcomes.
It is this aggregation
of distinct situations
under a shared representations
that makes OPSR
an abstraction.
This combined state and action abstraction
\citep{Konidaris2019}
is key to this paper.

\section{Conclusion and Future Work}%
\label{sec:conclusion}
This paper contains both theoretical and practical contributions
on transfer in Reinforcement Learning.
We first introduce a strong theoretical framework
based on state abstractions
that describes
the quality of a transferred policy
(\emph{transfer value})
in both single-
and multi-task settings.
Based on these insights
we propose a new state abstraction
based on predictions of
reward-relevant changes in the environment
(outcomes),
termed \emph{Outcome-Predictive State Representation} (OPSR).
Since outcomes describe transitions
in a task-agnostic manner,
OPSRs equally form a task-agnostic state description
which naturally leads to transfer.
We show that OPSRs are particularly useful
because of their bottom-up construction
which automatically leads to
\emph{transfer optimality}
in a class of MDPs which we term
\emph{plannable}.

Next,
we demonstrate how predicting fewer outcomes
coarsens the state abstraction,
leading to state aliasing and
increased opportunity for transfer.
While coarsening normally incurs a penalty in \emph{transfer value},
subtasks can offer the best of both worlds.
We introduce OPSR-based skills
that transfer between tasks
as a means to leverage
both state abstraction and action abstraction
to achieve transfer.
In a series of experiments
conducted with
a new library for transfer research
(VGDL 2.0),
we show that OPSR-based options
can be fully learned in one set of tasks
and subsequently speed up learning in
new and unseen tasks
with zero processing.

The practical investigation
in this paper
is a proof of concept
of the usefulness of OPSRs
that relies on two assumptions:
both \emph{outcomes}
and their predictions are available to the agent.
The next step left for future work
is to learn both of these.
The work in related areas
(as discussed in Section~\ref{sec:related_work})
and the observation that
outcomes are far more parsimonious goal representations
than fully concrete future states
make us confident in the feasability of this step.
In addition,
human brains appear to heavily limit predictions
according to needs and available actions
\citep{Konig2013};
such limitation would greatly improve tractability.
Another extension is to realize
an agent that learns OPSR-based skills
in an online fashion
rather than offline after episodes.
This would allow the options
to not only help with transfer but also
solving complex tasks.
Finally,
it would be interesting to consider multiple levels
in the hierarchy of skills.
Instead of sequences of actions after which outcomes are predicted
we would have sequences of options,
resulting in a high-level OPSR
that enables transfer at the higher level
in addition to the lower level.


\vskip 0.2in
\bibliography{references}

\appendix
\newpage
\section{Proofs}
\label{app:proofs}

\begin{customtheorem}{\ref{thm:abstraction_control}}
\end{customtheorem}

\begin{proof}[{\ref{thm:abstraction_control}}]
  \label{proof:abstraction_control}
  Let
  $\mdp =\langle \sset, \aset, p, r \rangle$
  be an MDP
  and let
  $\phi: \sset \mapsto \Phi$
  be a state abstraction
  that maps two states
  $s, t \in \sset$
  to the same equivalence class $\astate$.
  We will prove
  \eqref{eq:thm_state_abstraction_control}
  by contradiction
  in two parts,
  one for each direction of the biconditional.

  \paragraph{$(\bm{\Rightarrow})$}
  {We proceed with a proof by contradiction. }
  Assume that $\exists s, t \in \sset$ such that
  both of the following hold:
  \begin{gather}
      \phi(s) = \phi(t)
      \label{eq:proof_state_abstraction_control_right_given1}
      \\
      \exists \delta \in \Pi_\delta(\phi, \mdp):
      \delta(\cdot\mid s) \neq \delta(\cdot\mid t)
      \label{eq:proof_state_abstraction_control_right_given2}
  \end{gather}
  Any
  abstract policy
  $\pi_\phi:\Phi \mapsto \probset(\aset)$
  will have the same distribution over actions for
  the
  abstract state
  $\phi(s) = \phi(t)$:
  \begin{equation*}
    \pi_\phi(\cdot \mid \phi(s)) = \pi_\phi(\cdot \mid \phi(t))
  \end{equation*}
  By definition of derived policies
  along with
  the assumption in
  \eqref{eq:proof_state_abstraction_control_right_given1}
  we have
\begin{equation}
   \delta(\cdot\mid s) = \pi_\phi(\cdot\mid \phi(s))
    =
    \pi_\phi(\cdot\mid \phi(t)) = \delta(\cdot\mid t)
  \end{equation}
  This is in contradiction with
  \eqref{eq:proof_state_abstraction_control_right_given2}
  which proves the forward implication.

  \paragraph{$(\bm{\Leftarrow})$}
    {We again proceed with a proof by contradiction. }
    Assume  that $\exists s, t \in \sset$ such that:
  \begin{gather}
      \phi(s) \neq \phi(t)
      \label{eq:proof_state_abstraction_control_left_given1}
      \\
      \forall \delta \in \Pi_\delta(\phi, \mdp):
      \delta(\cdot\mid s) = \delta(\cdot\mid t)
      \label{eq:proof_state_abstraction_control_left_given2}
  \end{gather}
  Consider an abstract policy
  $\pi_\phi:\Phi \mapsto \probset(\aset)$
  which assigns different distributions over actions
  for abstract states $\phi(s)$ and $\phi(t)$:
  \begin{equation*}
    \pi_\phi(\cdot \mid \phi(s)) \neq \pi_\phi(\cdot \mid \phi(t))
  \end{equation*}
  By definition of derived policies along with
  the assumption in
  \eqref{eq:proof_state_abstraction_control_left_given1}
  we have
  \begin{equation*}
    \begin{gathered}
      \delta(\cdot\mid s)
      =
      \pi_\phi(\cdot\mid \phi(s))
      \neq
      \pi_\phi(\cdot\mid \phi(s))
      =
      \delta(\cdot\mid s)
    \end{gathered}
  \end{equation*}
  This is in contradiction with
  \eqref{eq:proof_state_abstraction_control_left_given2}
  which proves the reverse implication.

  \qed
\end{proof}

\begin{customtheorem}{\ref{thm:control_coarseness_tradeoff}}
    Let
  $\mdp =\langle \sset, \aset, p, r \rangle$
  be an MDP
  and let
  $\phi: \sset \mapsto \Phi$
  be a state abstraction
  that maps two states
  $s, t \in \sset$
  to the same equivalence class $\astate$.
  Then all derived policies for $\phi$
  must have the same distributions over actions
  for both $s$ and $t$.
  Formally,
  $
      \forall s, t \in \sset:
  $
  \begin{equation}
      \phi(s) = \phi(t)
      \;\Leftrightarrow\;
      \forall \delta \in \Pi_\delta(\phi, \mdp):
      \delta(\cdot\mid s) = \delta(\cdot\mid t)
      \label{eq:thm_state_abstraction_control}
  \end{equation}

\end{customtheorem}

\begin{proof}[{\ref{thm:control_coarseness_tradeoff}}]
  \label{proof:control_coarseness_tradeoff}
  Let
  $\mdp =\langle \sset, \aset, p, r \rangle$
  be an MDP
  and let
  $\phi: \sset \mapsto \Phi$
  and
  $\phi': \sset \mapsto \Phi'$
  be two state abstractions.
  We prove
  \eqref{eq:thm_derivability}
  by proving its inverse instead:
  \begin{equation}
  \begin{gathered}
\phi' \not \finereq \phi
\;\Longleftrightarrow\;
\Pi_\delta(\phi', \mdp)
\not \supseteq
\Pi_\delta(\phi, \mdp)
  \end{gathered}
  \end{equation}

  Or, more explictitly:
  \begin{equation}
  \begin{gathered}
    \exists \delta \in \Pi_\delta(\phi, \mdp):
    \delta \notin \Pi_\delta(\phi', \mdp) \\
    \Updownarrow \\
      \exists s, t \in \sset:
      \phi(s) \neq \phi(t)
      \wedge
      \phi'(s) = \phi'(t)
  \end{gathered}
  \label{eq:proof_derivability_inverse}
  \end{equation}

  We proceed by proving both directions of the bidirectional
  separately by contradiction.

  \paragraph{$(\bm{\Rightarrow})$}
    It is given that
  \begin{equation}
    \exists \delta \in \Pi_\delta(\phi, \mdp):
    \delta \notin \Pi_\delta(\phi', \mdp)
  \end{equation}
  which can be expanded to
  \begin{equation}
    \exists \delta \in \Pi_\delta(\phi, \mdp):
    \forall \delta' \in \Pi_\delta(\phi', \mdp):
    \exists s \in \sset:
    \delta(\cdot \mid s) \neq \delta'(\cdot \mid s)
    \label{eq:proof_derivability_forward1}
  \end{equation}

  We prove the forward implication
  by contradiction.
  Assume that
  \begin{equation}
      \nexists s, t \in \sset:
      \phi(s) \neq \phi(t)
      \wedge
      \phi'(s) = \phi'(t)
  \end{equation}
  This can be rewritten as
  \begin{equation}
      \forall s, t \in \sset:
      \phi'(s) = \phi'(t)
      \Rightarrow
      \phi(s) = \phi(t)
  \end{equation}
  The contrapositive yields a more useful perspective:
  \begin{equation}
      \forall s, t \in \sset:
      \phi(s) \neq \phi(t)
      \Rightarrow
      \phi'(s) \neq \phi'(t)
  \end{equation}

  In other words,
  for every two states
  $s,t \in \sset$
  that $\phi$
  maps to different abstract states,
  $\phi'$ will also map them to different states.
  Every distinction that $\phi$ makes,
  $\phi'$ also makes.
  This means that for every
  abstract policy
  $\pi_\phi$,
  there exists
  an equivalent
  abstract policy
  $\pi_{\phi'}$
  such that:
  \begin{equation}
    \pi_\phi(\cdot \mid s) = \pi_{\phi'}(\cdot \mid s)
    \;\wedge\;
    \pi_\phi(\cdot \mid t) = \pi_{\phi'}(\cdot \mid t)
  \end{equation}

  By definition of derived policies,
  for every derived policy
  $\delta_\phi \in \Pi_\delta(\phi, \mdp)$,
  there is a matching derived policy
  $\delta_{\phi'} \in \Pi_\delta(\phi', \mdp)$
  such that
  \begin{equation}
    \delta_\phi(\cdot \mid s) = \delta_{\phi'}(\cdot \mid s)
    \;\wedge\;
    \delta_\phi(\cdot \mid t) = \delta_{\phi'}(\cdot \mid t)
  \end{equation}
  This is in contradiction with the assumption
  \eqref{eq:proof_derivability_forward1}
  that there exists at least one
  derived policy
  $\delta_\phi \in \Pi_\delta(\phi, \mdp)$
  that does not have an equivalent
  derived policy
  $\delta_{\phi'} \in \Pi_\delta(\phi', \mdp)$.
  The contradiction proves the forward implication.

  \paragraph{$(\bm{\Leftarrow})$}
  It is given that
  \begin{equation}
  \begin{gathered}
      \exists s, t \in \sset:
      \phi(s) \neq \phi(t)
      \wedge
      \phi'(s) = \phi'(t)
    \label{eq:proof_derivability_reverse1}
  \end{gathered}
  \end{equation}
  To prove the reverse implication,
  we proceed by contradiction.
  Assume that
  \begin{equation}
    \forall \delta \in \Pi_\delta(\phi, \mdp):
    \exists \delta' \in \Pi_\delta(\phi', \mdp):
    \forall s \in \sset:
    \delta(\cdot \mid s) = \delta'(\cdot \mid s)
    \label{eq:proof_derivability_reverse2}
  \end{equation}
  In other words,
  for each derived policy
  $\delta_\phi \in \Pi_\delta(\phi, \mdp)$
  there exists an equivalent
  derived policy
  $\delta_{\phi'} \in \Pi_\delta(\phi', \mdp)$.

  From
  \eqref{eq:proof_derivability_reverse1}
  we know that there
  exists
  an abstract policy
  which assigns different distributions over actions
  to abstract states
  $\phi(s)$
  and
  $\phi(t)$:
  $\pi_\phi(\cdot \mid \phi(s)) \neq \pi_\phi(\cdot \mid \phi(t))$.
  To contrast,
  because
  $\phi'(s) = \phi'(t)$
  there can be no
  $\pi_{\phi'}$
  such that
  $\pi_{\phi'}(\cdot \mid \phi'(s)) \neq \pi_{\phi'}(\cdot \mid \phi'(t))$.
  By the definition of derived policies,
  we have
  \begin{align}
    \exists \delta \in \Pi_\delta(\phi, \mdp)
    &:
    \delta(\cdot \mid s) \neq \delta(\cdot \mid t)
    \\
    \forall \delta' \in \Pi_\delta(\phi', \mdp)
    &:
    \delta'(\cdot \mid s) = \delta'(\cdot \mid t)
  \end{align}
  That is,
  it is possible
  that
  $
    \delta'(\cdot \mid s) = \delta(\cdot \mid s)
  $
  or
  $
    \delta'(\cdot \mid t) = \delta(\cdot \mid t)
  $,
  but never both.
  This contradicts the assumption
  in
  \eqref{eq:proof_derivability_reverse2},
  proving the reverse implication
  and completing the proof for the biconditional in
  \eqref{eq:thm_derivability}.

  \qed
\end{proof}

\begin{customtheorem}{\ref{thm:transfer_value_control_tradeoff}}
    Let
  $\phi: \sset \mapsto \Phi$
  and
  $\phi': \sset \mapsto \Phi'$
  be two state abstractions
  such that
  $\phi$
  has
  \emph{strictly higher transfer value than}
  $\phi'$
  for some MDP
  $\mdp \defas \langle \sset, \aset, p, r \rangle, \gamma$.
  Then there must
  exist two states
  $s, t \in \sset$
  such that
  $\phi'$ does not distinguish between them
  while $\phi$ does,
  and there exists a derived policy for
  $\phi$
  that takes advantage of this distinction
  by controlling $s$ and $t$ differently
  in a way that no derived policy for $\phi'$ can do.
  By giving up representational capacity
  and increasing potential for transfer,
  $\phi'$ sacrifices transfer value.
  Formally:
  \begin{gather*}
    v_\phi(\mdp) > v_{\phi'}(\mdp)
    \\
    \Downarrow \\
      \exists s, t \in \sset:\\
    \begin{aligned}
      &1.\; \phi(s) \neq \phi(t) \\
      &2.\; \phi'(s) = \phi'(t) \\
      &3.\; \exists \delta \in \Pi_\delta(\phi, \mdp):
      \delta(\cdot\mid s) \neq \delta(\cdot\mid t)
      \\
      &4.\; \forall \delta' \in \Pi_\delta(\phi', \mdp):
      \delta'(\cdot\mid s) = \delta'(\cdot\mid t)
    \end{aligned}
  \end{gather*}

\end{customtheorem}

\begin{proof}[\ref{thm:transfer_value_control_tradeoff}]
  \label{proof:transfer_value_control_tradeoff}
  Let
  $\mdp = \langle \sset, \aset, p, r \rangle$
  be an MDP
  and
  $\phi: \sset \mapsto \Phi$
  and
  $\phi': \sset \mapsto \Phi'$
  be two state abstractions.
  In addition we have:
  \begin{equation}
    \exists \delta \in \Pi_\delta(\phi, \mdp): \forall \delta' \in \Pi_\delta(\phi', \mdp): v_\delta > v_{\delta'}
    \label{eq:control_granularity_given}
  \end{equation}
  For simplicity we will assume deterministic policies in this proof.
  Instead of distributions over actions,
  a policy $\pi: \sset \mapsto \aset$ simple maps a state to an action.
  The extension to stochastic policies is straightforward.

  First,
  we show that there must exist a state $s \in \sset$
  for which there is a derived policy for $\phi$
  that takes a different action from every derived policy for $\phi'$.
  By \eqref{eq:control_granularity_given}
  and the definition for inequality of value functions
  we have
  $\exists \delta \in \Pi_\delta(\phi, \mdp): \forall \delta' \in \Pi_\delta(\phi', \mdp): \exists s \in \sset:$
  \begin{multline}
    v_\delta(s) > v_{\delta'}(s)
    \iff
      \\
    r(s, \delta(s)) + \expected_{s' \sim p(s, \delta(s))} \left[ \gamma v_\delta(s') \right]
    >
    r(s, \delta'(s)) + \expected_{s' \sim p(s, \delta'(s))} \left[ \gamma v_{\delta'}(s') \right]
    \qquad
    \label{eq:control_gran_given_extrapolated}
  \end{multline}
If $\delta(s) = \delta'(s)$,
the inequality in \eqref{eq:control_gran_given_extrapolated}
can only hold if
  \begin{align}
    \expected_{s' \sim p(s, \delta(s))} \left[ \gamma v_\delta(s') \right]
    >
    \expected_{s' \sim p(s, \delta'(s))} \left[ \gamma v_{\delta'}(s') \right]
    \label{eq:control_gran_expected_value_inequality}
  \end{align}
  %
  %
Hence, either the actions at $s$ are different or inequality
\eqref{eq:control_gran_expected_value_inequality} also holds. The inequality
\eqref{eq:control_gran_expected_value_inequality}
can only hold if there is some other state
$s^\dagger$
for which the inequality in \eqref{eq:control_gran_given_extrapolated}
holds
by means of $\delta(s^\dagger) \neq \delta'(s^\dagger)$
or that there is some other state, $s^{\dagger\dagger}$, reachable in one step from\ $s^{\dagger}$, for which inequality \eqref{eq:control_gran_given_extrapolated} also holds.
By a process of induction, if Equation \eqref{eq:control_gran_given_extrapolated} holds for $s$ there must exist some state, reachable from $s$ in a finite number of steps under $\delta$ (and $\delta'$) for which the actions under these two policies differ.
In short,
there exists a state $s$
where $\delta(s) \neq \delta'(s)$.
We have:
\begin{equation}
  \begin{gathered}
  \exists \delta \in \Pi_\delta(\phi, \mdp): \forall \delta' \in \Pi_\delta(\phi', \mdp): \exists s \in \sset:
  \delta(s) \neq \delta'(s)
  \\
  \Updownarrow\\
  \exists \delta \in \Pi_\delta(\phi, \mdp): \nexists \delta' \in \Pi_\delta(\phi', \mdp): \forall s \in \sset:
  \delta(s) = \delta'(s)
  \\
  \Updownarrow\\
  \exists \delta \in \Pi_\delta(\phi, \mdp): \nexists \delta' \in \Pi_\delta(\phi', \mdp): \delta = \delta'
  \\
  \Updownarrow\\
    \exists \delta \in \Pi_\delta(\phi, \mdp):
    \delta \notin \Pi_\delta(\phi', \mdp)
  \end{gathered}
\end{equation}
By
Proposition~\ref{thm:control_coarseness_tradeoff}
we now have that
there do exists at least two states
$s, t \in \sset$
such that
\begin{equation}
  \phi(s) \neq \phi(t)
  \wedge
  \phi'(s) = \phi'(t)
\end{equation}
This proves implications (1) and (2) of the theorem.
By repeated invocation of
Proposition~\ref{thm:abstraction_control},
implications (3) and (4) follow naturally
from implications (1) and (2) respectively.

\qed
\end{proof}

\begin{customtheorem}{\ref{thm:transfer_value_control_tradeoff_corollary}}
    Let
  $\phi: \sset \mapsto \Phi$
  and
  $\phi': \sset \mapsto \Phi'$
  be two state abstractions
  such that
  $\phi$
  has
  \emph{strictly higher transfer value than}
  $\phi'$
  for some MDP $\mdp$.
  Then
  $\phi'$ can both be \emph{no finer}
  than $\phi$
  and \emph{no more control granular}
  than $\phi$.
  Formally:
  \begin{equation}
    v_\phi(\mdp) > v_{\phi'}(\mdp)
    \Longrightarrow
    \phi' \not \finereq \phi
    \;\wedge\;
    \Pi_\delta(\phi, \mdp)
    \not \subseteq \Pi_\delta(\phi', \mdp)
    \label{eq:transfer_value_control_tradeoff_corollary1}
  \end{equation}
  The following also holds.
  If $\phi$ is coarser
  and less control granular than $\phi'$,
  it is impossible for
  $\phi$
  to have a higher transfer value than
  $\phi'$.
  Formally:
  \begin{equation}
    \phi' \finereq \phi
    \;\wedge\;
    \Pi_\delta(\phi, \mdp)
    \subseteq \Pi_\delta(\phi', \mdp)
    \Longrightarrow
    v_\phi(\mdp) \not > v_{\phi'}(\mdp)
    \label{eq:transfer_value_control_tradeoff_corollary2}
  \end{equation}

\end{customtheorem}

\begin{proof}[{\ref{thm:transfer_value_control_tradeoff_corollary}}]
  \label{proof:transfer_value_control_tradeoff_corollary}
  This corollary is the result of a weakening of
  Theorem~\ref{thm:transfer_value_control_tradeoff}.
  The first result
  \eqref{eq:transfer_value_control_tradeoff_corollary1}
  is
  a direct result
  of the application of the definitions of
  coarseness and control granularity.

  \paragraph{}
  For the second result,
  consider the contrapositive of
  \eqref{eq:transfer_value_control_tradeoff_corollary1}:
  \begin{equation}
    \phi' \finereq \phi
    \;\vee\;
    \Pi_\delta(\phi, \mdp)
    \subseteq \Pi_\delta(\phi', \mdp)
    \Longrightarrow
    v_\phi(\mdp) \not > v_{\phi'}(\mdp)
  \label{eq:transfer_value_control_tradeoff_corollary_proof1}
  \end{equation}
  Application of
  Proposition~\ref{thm:control_coarseness_tradeoff}
  turns the disjunction
  in
  \eqref{eq:transfer_value_control_tradeoff_corollary_proof1}
  into a conjunction,
  arriving at
  \eqref{eq:transfer_value_control_tradeoff_corollary2}.

  \qed
\end{proof}

\begin{customtheorem}{\ref{thm:outcome_equivalent_length}}
    Given two
  MDPs
  $\mdp_\alpha=\langle \sset_\alpha, \aset, p_\alpha, r_\alpha \rangle$
  and
  $\mdp_\beta = \langle \sset_\beta, \aset, p_\beta, r_\beta \rangle$,
  two states
  $s_\alpha \in \sseta, s_\beta \in \ssetb$,
  and an action sequence
  $(a_1,\dots,a_{n-1}, a_n)$
  such that
  \begin{equation*}
    \sigmaseqa (s_\alpha, (a_1,\dots,a_{n-1}, a_n)) =
    \sigmaseqb (s_\beta, (a_1,\dots,a_{n-1}, a_n))
  \end{equation*}
  Then it also holds that
  \begin{equation*}
    \sigmaseqa (s_\alpha, (a_1,\dots,a_{n-1})) =
    \sigmaseqb (s_\beta, (a_1,\dots,a_{n-1}))
  \end{equation*}
  In other words,
  if two outcome sequence tests
  for $s_\alpha$ and $s_\beta$
  are the same for some
  length-$n$ action sequence,
  then the two outcome sequence tests
  for $s_\alpha$ and $s_\beta$
  constructed from
  the action subsequence of length $n-1$
  starting from $a_1$
  are also equal.

\end{customtheorem}

\begin{proof}[\ref{thm:outcome_equivalent_length}]
  \label{proof:outcome_equivalent_length}
  This follows directly
  from
  Proposition~\ref{thm:expected_outcome_sequence_sequence_expectations}.
  If two sequences are equal,
  then any matching subsequences are also equal.
  Given that
  $
    \langle \sigma \rangle (s, (a_1,\dots,a_{n-1}))
  $
  and
  $
    \langle \sigma \rangle (s', (a_1,\dots,a_{n-1}))
  $
  are subsequences of
  $
    \langle \sigma \rangle (s, (a_1,\dots,a_{n-1}, a_n))
  $
  and
  $
    \langle \sigma \rangle (s', (a_1,\dots,a_{n-1}, a_n))
  $
  respectively,
  if the latter are equal then so must the former be.

  \qed
\end{proof}

\begin{customtheorem}{\ref{thm:expected_outcome_sequence_sequence_expectations}}
    Let
  $\mdp = \langle \sset, \aset, p, r \rangle$
  be an MDP
  with outcome function
  $\sigma$.
  The
  \emph{expected outcome sequence}
  for state $s \in \sset$
  and action sequence
  $\aseq = (a_1\dots a_n) \in \aset^n$
  can also be written as the sequence of
  expected outcomes after action sequences
  $(a_1\dots a_i)$ for
  $i$ ranging from $1$ to $n$.
  Formally:
  \begin{multline}
    \sigmaseq(s, \aseq)
    =
    \big(
    \expected_{p}
    \left[
      \sigma(S_t, A_t, S_{t+1}) \mid S_t=s, A_t=a_1
    \right]
    ,\dots,
    \\
    \expected_{p}
    \left[
      \sigma(S_{t+n-1}, A_{t+n-1}, S_{t+n}) \mid
      S_t=s, A_t=a_1,\dots,
      A_{t+n-1}=a_n
    \right]
  \big)
  \label{eq:expected_outcome_sequence_sequence_expectations}
  \end{multline}

\end{customtheorem}

\begin{proof}[\ref{thm:expected_outcome_sequence_sequence_expectations}]
  \label{proof:expected_outcome_sequence_sequence_expectations}
  Since $\sigma(s, a, s')$ is defined as an expectation,
  we have:
  \begin{equation}
    \sigma(s, a, s') = \int_{\Sigma}
    p_\sigma(\sigma_i \mid s, a, s')
    (\sigma_1,\dots,\sigma_n)
    d(\sigma_1,\dots,\sigma_n)
    \label{eq:outcome_integral}
  \end{equation}
  Since $c(\sigma_1,\dots,\sigma_n) = (c\sigma_1,\dots,c\sigma_n)$
  and the integral of sequences in
  Eq.~\eqref{eq:expected_outcome_seq_integral}
  is a sequence of integrals, we can look at each element of the resulting
  sequence in turn.
  Consider the element in the resulting sequence at position
  $i$, with $1 \leq i \leq n$:
  \begin{equation}
    \int_{\Sigma}
    p_\sigma((\sigma_1,...,\sigma_n) \mid s_t, \aseq)
    \sigma_i
    d^n(\sigma_1,...,\sigma_n)
  \end{equation}
  Writing $p_\sigma((\sigma_1,\dots,\sigma_n)\mid s_t, \aseq)$ in full, this becomes:
  \begin{equation*}
    \begin{split}
      \int_{\Sigma}
      \sum_{s_{t+1},...,s_{t+n} \in \sset^{n}}
      p(s_{t+1} \mid s_t, a_1)
      \dots
      p(s_{t+n} \mid s_{t+n-1}, a_{n})
      \\
      \quad
      p_\sigma(\sigma_1 \mid s_{t}, a_1, s_{t+1})
      \dots
      &p_\sigma(\sigma_n \mid s_{t+n-1}, a_n, s_{t+n})
      \sigma_i
      d^n(\sigma_1,...,\sigma_n)
    \end{split}
  \end{equation*}
  Since $p_\sigma$
  is a probability density function
  conditioned on a transition $(s, a, s')$,
  we have
  $\int_\Sigma p_\sigma(\sigma_j\mid s, a, s')d\sigma_j=1$.
  This allows for further simplification:
  \begin{align*}
    \begin{split}
      \sum_{s_{t+1},...,s_{t+n} \in \sset^{n}}
      p(s_{t+1} \mid s_t, a_1)
      \dots
      p(s_{t+n} \mid s_{t+n-1}, a_{n})
      \int_{\Sigma}
      p_\sigma(\sigma_i \mid s_{t+i-1}, a_i, s_{t+i})
      \\
      \int_{\Sigma}
      p_\sigma(\sigma_1 \mid s_{t}, a_1, s_{t+1})
      \dots
      p_\sigma(\sigma_{i-1} \mid s_{t+i-2}, a_{i-1}, s_{t+i-1})
      \\
      p_\sigma(\sigma_{i+1} \mid s_{t+i}, a_{i+1}, s_{t+i+1})
      p_\sigma(\sigma_n \mid s_{t+n-1}, a_n, s_{t+n})
      \sigma_i
      d^n(\sigma_1,...,\sigma_n)
    \end{split}
  \end{align*}
  \begin{align*}
    =
      \sum_{s_{t+1},...,s_{t+n} \in \sset^{n}}
      p(s_{t+1} \mid s_t, a_1)
      \dots
      p(s_{t+n} \mid s_{t+n-1}, a_{n})
      \int_{\Sigma}
      p_\sigma(\sigma_i \mid s_{t+i-1}, a_i, s_{t+i})
      \sigma_i
      d\sigma_i
  \end{align*}
  Combining this with
  Eq.~\eqref{eq:outcome_integral}
  and the fact that for a given start state $s$ and action $a$,
  all transition probabilities must sum up to $1$,
  i.e. $\sum_{s' \in \sset} p(s'\mid s, a) = 1$,
  this further simplifies to:
  \begin{align*}
    \begin{split}
      &=
      \sum_{s_{t+1},...,s_{t+n} \in \sset^{n}}
      p(s_{t+1} \mid s_t, a_1)
      \dots
      p(s_{t+n} \mid s_{t+n-1}, a_{n})
      \sigma(s_{t+i-1}, a_i, s_{t+i})
    \end{split}
    \\
    \begin{split}
      &=
      \sum_{s_{t+1},...,s_{t+i} \in \sset^{i}}
      p(s_{t+1} \mid s_t, a_1)
      \dots
      p(s_{t+i} \mid s_{t+i-1}, a_i)
      \sigma(s_{t+i-1}, a_i, s_{t+i})
      \\
      &\qquad
      \sum_{s_{t+i+1},...,s_{t+n} \in \sset^{n-i}}
      p(s_{t+i+1} \mid s_{t+i}, a_{i+1})
      \dots
      p(s_{t+n} \mid s_{t+n-1}, a_{n})
    \end{split}
    \\
    \begin{split}
      &=
      \sum_{s_{t+1},...,s_{t+i} \in \sset^{i}}
      p(s_{t+1} \mid s_t, a_1)
      \dots
      p(s_{t+i} \mid s_{t+i-1}, a_i)
      \sigma(s_{t+i-1}, a_i, s_{t+i})
    \end{split}
  \end{align*}
  This is exactly the expected value of the outcomes
  starting from state $s_t$
  after actions $a_1,\dots,a_i$,
  i.e.
  $
    \expected_{p}
    \left[
      \sigma(S_{t+i-1}, A_{t+i}, S_{t+i}) \mid
      S_t=s_t, A_t=a_1,\dots,
      A_{t+i}=a_i
    \right]
    $.
  Putting this together for all
  $i$ such that $1 \leq i \leq n$
  results in the sequence of expectations
  in
  Eq.~\eqref{eq:expected_outcome_sequence_sequence_expectations},
  the object of this proof.

  \qed
\end{proof}

\begin{customtheorem}{\ref{thm:outcome_equivalent_abstract_ground}}
    Let
  $\mdp_\alpha=\langle \sset_\alpha, \aset, p_\alpha, r_\alpha, \gamma \rangle$
  and
  $\mdp_\beta = \langle \sset_\beta, \aset, p_\beta, r_\beta, \gamma \rangle$
  be two MDPs
  with respective outcome functions
  $\sigma_\alpha$ and $\sigma_\beta$
  and reward weightings
  $\rwa$ and $\rwb$
  such that
  $r_\alpha = \sigma_\alpha^\top \rwa$
  and
  $r_\beta = \sigma_\beta^\top \rwb$,
  and
  let
  $\phi_\alpha: \sset_\alpha \mapsto \asset_\alpha$
  and
  $\phi_\beta: \sset_\beta \mapsto \asset_\beta$
  be two
  \emph{outcome equivalent} state abstractions.
  Then all abstract states
  in $\Phia$
  are outcome equivalent with their ground states
  with respect to
  any abstract MDP
  $\mdp_\phia=\langle \Phia, \aset, p_\phia, r_\phib, \gamma \rangle$
  created from $\mdpa$ and $\phia$
  with outcome function
  $\sigma_\phia$
  and reward weighting
  $w_{r_\phia}$
  such that
  $r_\phia = \sigma_\phia^\top w_{r_\phia}$,
  given that the reward weighting
  for that abstract MDP is the same
  as those of the ground MDPs,
  i.e.
  $w_{r_\phia} = \rwb = \rwa$.
  Formally,
  $\forall \astate \in \Phia, \forall s_\beta \in \phib^{-1}(\astate)$:
  \begin{equation}
    \sigmaseq_\phia (\astate, \aset^\infty) = \sigmaseqb (s_\beta, \aset^\infty)
  \end{equation}

\end{customtheorem}

\begin{proof}[\ref{thm:outcome_equivalent_abstract_ground}]
  \label{proof:outcome_equivalent_abstract_ground}
  For this proof, we assume for simplicity
  that the abstract state set $\Phia$ is finite.
  We proceed by a kind of \emph{backward} induction,
  as will become clearer in the induction hypothesis.
  For the hypothesis we will assume that the `next' abstract states
  are outcome equivalent with their ground states.

  \paragraph{Base case}
  Since this is a proof by backward induction,
  we start at the end.
  The base case is a final absorbing state
  $\astate \in \Phia$.
  All ground states $s_\beta \in \astate$
  must also be final.
  Just like there can be no reward past the
  final state,
  there can be no outcomes past the final state either.
  Since both $\astate$
  and all $s_\beta \in \astate$
  are final states,
  they must be outcome equivalent.

  \paragraph{Induction hypothesis}%
  We assume that all `next' abstract states
  $\astate_{t+1} \in \Phia$
  are outcome equivalent with their ground states.
  Let $\astaterv_{t}$ be the random variable describing the abstract state
  at time $t$.
  Then it holds that for all states
  $s_{t+1} \in \phibi(\astate_{t+1})$:
  $\forall n \in \naturalnumbers_1$:
  $\forall (a_2,\dots,a_n) \in \aset^{n-1}$:
  \begin{equation}
      \begin{split}
      \expected_{p_\phia} \left[
        \sigma(\astaterv_{t+n-1}, A_{t+n}, \astaterv_{t+n})
        \mid
        \astaterv_{t+1} = \astate_{t+1},
        A_{t+2} = a_2,\dots,
        A_{t+n} = a_n
        \right]
        \\
        =
      \expected_{p_\beta}
      \left[
        \sigma(S_{t+n-1}, A_{t+n-1}, S_{t+n}) \mid
        S_{t+1}=s_{t+1},
        A_{t+2} = a_2,\dots,
        A_{t+n} = a_n
      \right]
      \end{split}
  \end{equation}

  \paragraph{Inductive step}%
  \label{par:inductive_step}
  We now take a step `backward' and consider
  the abstract state
  $\astate_{t}$,
  direct predecessor to states
  $\astate_{t+1}$.
  We need to show that $\astate_t$ is outcome equivalent with its ground states.

  First we consider the single-step expected outcomes separately
  for any single action $a \in \aset$.
  From the construction of the abstract reward dynamics
  $r_\phia$
  we can derive two identities.
  The first is based on the definition of
  the expected outcomes.
  \begin{align}
    r_\phia(\astate_t, a)
    &=
    \sum_{\astate_{t+1} \in \Phia}
    p_\phia(\astate_{t+1} \mid \astate, a)
    \sigma(\astate_t, a, \astate_{t+1})^\top w_{r_\phia}
    \nonumber
    \\
    &=
    \expected_{p_\phia}
    \left[
      \sigma(\astaterv_t, A_t, \astaterv_{t+1}) \mid \astaterv_t=\astate_t, A_t=a
    \right]^\top
    w_{r_\phia}
    \label{eq:proof_expected_outcome_sequence_step1}
  \end{align}

  The second is based on the construction
  of $r_\phia$
  from ground reward dynamics $r_\alpha$
  with the help of a weighting function
  $w_\phia: \sset \mapsto [0,1]$ (see definition abstract MDP),
  which disappears
  because all ground states $s_\alpha \in \phiai(\astate_t)$
  are outcome equivalent and therefore have
  the same expected
  outcome for any action $a$.
  \begin{align}
    r_\phia(\astate_t, a)
    &=
    \sum_{s_\alpha \in \phiai(\astate_t)}
    \sum_{s_\alpha' \in \sseta}
    w_{\phia}(s_\alpha)
    p_\alpha(s_\alpha' \mid s_{\alpha}, a)
    \sigma_\alpha(s_\alpha, a, s_\alpha')^\top w_{r_\alpha}
    \nonumber
    \\
    &=
    \sum_{s_{\alpha}' \in \sseta}
    w_{\phia}(s_\alpha)
    p_\alpha(s_\alpha' \mid s_\alpha, a)
    \sigma_\alpha(s_\alpha, a, s_\alpha')^\top w_{r_\alpha}
    &&\forall s_\alpha \in \phiai(\astate_t)
    \nonumber
    \\
    &=
    \expected_{p_\alpha}
    \left[
      \sigma_\alpha(S_t, A_t, S_{t+1}) \mid S_t=s_\alpha, A_t=a
    \right]^\top
    w_{r_\alpha}
    \nonumber
    \\
    &=
    \expected_{p_\beta}
    \left[
      \sigma_\beta(S_t, A_t, S_{t+1}) \mid S_t=s_\beta, A_t=a
    \right]^\top
    w_{r_\beta}
    &&
    \text{Proposition~\ref{thm:outcome_equivalent_relation_from_abstractions};}\; w_{r_\alpha} = w_{r_\beta}
    \label{eq:proof_expected_outcome_sequence_step2}
  \end{align}
  Combining Equations~\eqref{eq:proof_expected_outcome_sequence_step1}
  and \eqref{eq:proof_expected_outcome_sequence_step2}
  along with the given fact that
  $w_{r_\phia} = w_{r_\beta}$,
  we have that the expected outcomes for a single action
  $a$
  are the same for $\astate_t$ and its ground states
  in $\ssetb$.

  Second, we consider expected outcomes for all action sequences longer
  than a single action.
  To prevent assuming the induction hypothesis for the entire abstract state set,
  it is sufficient to assume
  the hypothesis only for those states
  where $p_\phia(\astate_{t+1} \mid \astate_t, \cdot) \neq 0$.
  We start with the expected outcome
  starting from $\astate_t$
  and any action sequence
  $(a_1,\dots,a_n)$
  of length $n \in \naturalnumbers$
  with $n>1$:
  \begin{align}
      \sum_{\astate_{t+1},...,\astate_{t+n} \in \Phia^{n}}
      p_\phia(\astate_{t+1} \mid \astate_t, a_1)
      \dots
      p_\phia(\astate_{t+n} \mid \astate_{t+n-1}, a_n)
      \sigma_\phia(\astate_{t+n-1}, a_n, \astate_{t+n})
      \label{eq:proof_expected_outcome_sequence_abstract_step}
  \end{align}
  By construction of $p_\phia$
  this expands to:
  \begin{equation*}
    \begin{split}
      \sum_{\astate_{t+1} \in \Phia}
      \sum_{s_t \in \phibi(\astate_t)}
      \sum_{s_{t+1} \in \phibi(\astate_{t+1})}
      w_{\phia}(s_t)p_\beta(s_{t+1} \mid s_{t}, a_1)
      \\
      \sum_{\astate_{t+2},...,\astate_{t+n} \in \Phia^{n-1}}
      p_\phia(\astate_{t+2} \mid \astate_{t+1}, a_2)
      \dots
      &p_\phia(\astate_{t+n} \mid \astate_{t+n-1}, a_n)
      \sigma_\phia(\astate_{t+n-1}, a_n, \astate_{t+n})
    \end{split}
  \end{equation*}
  Substituting according to the induction hypothesis:
  \begin{equation*}
    \begin{split}
      \sum_{\astate_{t+1} \in \Phia}
      \sum_{s_t \in \phibi(\astate_t)}
      \sum_{s_{t+1} \in \phibi(\astate_{t+1})}
      w_{\phia}(s_t)p_\beta(s_{t+1} \mid s_{t}, a_1)
      \\
      \sum_{s_{t+2},...,s_{t+n} \in \ssetb^{n-1}}
      p_\beta(s_{t+2} \mid s_{t+1}, a_2)
      \dots
      &p_\beta(s_{t+n} \mid s_{t+n-1}, a_n)
      \sigma_\beta(s_{t+n-1}, a_n, s_{t+n})
    \end{split}
  \end{equation*}
  Since
  for any two abstract states $\astate, \astate' \in \Phia$
  there is no overlap in their ground states,
  i.e.
  $\phibi(\astate) \cap \phibi(\astate') = \emptyset$,
  we can simplify the summation:
  \begin{align*}
    \begin{split}
      &\sum_{s_t \in \phibi(\astate_t)}
      \sum_{s_{t+1} \in \ssetb}
      w_{\phia}(s_t)
      p_\beta(s_{t+1} \mid s_{t}, a_1)
      \\
      &\qquad
      \sum_{s_{t+2},...,s_{t+n} \in \ssetb^{n-1}}
      p_\beta(s_{t+2} \mid s_{t+1}, a_2)
      \dots
      p_\beta(s_{t+n} \mid s_{t+n-1}, a_n)
      \sigma_\beta(s_{t+n-1}, a_n, s_{t+n})
    \end{split}
    \\
    \begin{split}
      =
      &\sum_{s_t \in \phibi(\astate_t)}
      w_{\phia}(s_t)
      \expected_{p_\beta}
      \left[
        \sigma(S_{t+n-1}, A_{t+n-1}, S_{t+n}) \mid
        S_{t}=s_{t},
        A_{t+2} = a_2,\dots,
        A_{t+n} = a_n
      \right]
    \end{split}
  \end{align*}
  Finally,
  since we have that all states
  $s_t \in \phibi(\astate_t)$
  are outcome equivalent
  and that for any $\astate$ it holds that
  $\sum_{s \in \phibi(\astate)} w_{\phia}(s)=1$,
  this further simplifies to
  \begin{equation}
      \expected_{p_\beta}
      \left[
        \sigma(S_{t+n-1}, A_{t+n-1}, S_{t+n}) \mid
        S_{t}=s_{t},
        A_{t+2} = a_2,\dots,
        A_{t+n} = a_n
      \right]
      \label{eq:proof_expected_outcome_sequence_ground_step}
  \end{equation}
  This proves the equality
  between Equations~\eqref{eq:proof_expected_outcome_sequence_abstract_step}
  and \eqref{eq:proof_expected_outcome_sequence_ground_step},
  completing the inductive step.
  We conclude that
  any abstract state
  $\astate \in \Phia$
  is outcome equivalent to
  all of its ground states
  with respect to $\phib$.

  \qed
\end{proof}

\begin{lemma}
  \label{thm:plannable_abstract_value_compatible}
  Let
  $\mdp = \langle \sset, \aset, p, r, \gamma \rangle$
  be a
  \emph{plannable} MDP
  with outcome function
  $\sigma$
  and reward weighting
  $\rw$
  such that
  $r = \sigma^\top \rw$,
	and let
	$\phi: \sset \mapsto \Phi$
	be an \emph{outcome equivalent} state abstraction.
	Then any abstract policy
  $\pi_\phi$
  for an abstract MDP
  $\mdp_\phi=\langle \Phi, \aset, p_\phi, r_\phi, \gamma \rangle$
	created from $\mdp$ and $\phi$
  with outcome function
  $\sigma_\phi$
  and reward weighting
  $w_{r_\phi}$
  such that
  $r_\phi = \sigma_\phi^\top w_{r_\phi}$
  and
  with the same reward weighting,
  i.e.
  $w_{r_\phi} = \rw$,
  is \emph{derived value--compatible}.
\end{lemma}

\begin{proof}[\ref{thm:plannable_abstract_value_compatible}]
  We again rely on the notion of the
  \emph{k-horizon value}
  of a state:
  \begin{equation*}
    v^k_{\pi}(s) \defas \expected\left[ \sum_{i=1}^{k} \gamma^{i-1} R_{t+i} \relmiddle| S_t=s \right]
  \end{equation*}
  We proceed by induction.
  We will show that for any abstract policy $\pi_\phi$,
  it is value--compatible with
  any derived policy
  $\delta$ for $\mdp$.
  That is,
  $\forall \astate \in \Phi, \forall s \in \astate$:
  \begin{equation}
    v^k_\delta(s; \mdp) = v^k_{\pi_\phi}(\astate; \mdp_\phi)
    \label{eq:proof_plannable_outcome_equivalent_value_compatible_goal}
  \end{equation}

  \paragraph{}
  For the base case take $k=1$.
  This is simply the expected return for a single step.
  \begin{align*}
    v_{\pi_\phi}^1(\astate)
    &=
    \sum_{a} \pi_\phi(a|\astate) r_\phi(\astate, a)
    \\
    &=
    \sum_{a} \delta(a|s) r_\phi(\astate, a)
    &&
    \forall s \in \phii(\astate);\; \text{def. } \delta
    \\
    &=
    \sum_{a} \derived(a|s) \sum_{\astate' \in \Phi} \sigma(\astate, a, \astate')^\top w_{r_\phi}
    &&
    \text{def. outcomes}
    \\
    &=
    \sum_{a} \derived(a|s) \sum_{s' \in \ssetb} \sigma(s, a, s')^\top w_{r_\phi}
    &&
    \text{def. outcome equivalent};\; w_{r_\phi} = \rwb
    \\
    &=
    \sum_{a} \derived(a|s) r(s, a)
    &&
    \text{def. outcomes}
    \\
    &=
    v_{\derived}^1(s)
  \end{align*}
  For the inductive step,
  suppose
  Eq.~\eqref{eq:proof_plannable_outcome_equivalent_value_compatible_goal}
  holds for
  horizon $k-1$.
  Starting from the recursive definition of
  value functions,
  we get
  $\forall \astate \in \Phi,
  \forall s \in \astate$:
  \begin{align*}
    v_{\pi_\phi}^k(\astate)
    &=
    \sum_{a\in\aset} \pi_\phi(a|\astate)
    \left[
    r_\phi(\astate, a)
    + \gamma
    \sum_{\astate'\in\Phi}
    p_\phi(\astate' | \astate, a)
    v_{\pi_\phi}^{k-1}(\astate')
  \right]
    \\
    &=
    \sum_{a\in\aset} \delta(a|s)
    \left[
    r(s, a)
    + \gamma
    \sum_{\astate'\in\Phi}
    p_\phi(\astate' | \astate, a)
    v_{\pi_\phi}^{k-1}(\astate')
  \right]
    &&
    \text{def. $\delta$;}\;\
    v_{\pi_\phi}^1(\astate)
    =
    v_{\derived}^1(s)
    \\
    &=
    \sum_{a\in\aset} \delta(a|s)
    \left[
    r(s, a)
    + \gamma
    \sum_{\astate'\in\Phi}
    \sum_{t \in \phii(\astate)}
    \sum_{s' \in \phii(\astate')}
    w(t) p(s' | t, a)
    v_{\pi_\phi}^{k-1}(\astate')
    \right]
    &&
    \text{def. } p_\phi
    \\
    &=
    \sum_{a\in\aset} \delta(a|s)
    \left[
    r(s, a)
    + \gamma
    \sum_{\astate'\in\Phi}
    \sum_{s' \in \phii(\astate')}
    p(s' | s, a)
    v_{\delta}^{k-1}(s)
    \right]
    &&
    \text{Lemma~\ref{thm:plannable_outcome_equivalent_value_equivalent}; inductive step}
    \\
    &=
    \sum_{a\in\aset} \delta(a|s)
    \left[
    r(s, a)
    + \gamma
    \sum_{s' \in \sset}
    p(s' | s, a)
    v_{\delta}^{k-1}(s)
    \right]
    &&
    \bigcup_{\astate' \in \Phi} = \sset, \bigcap_{\astate' \in \Phi} = \emptyset
    \\
    &=
    v^k_\delta(s)
  \end{align*}
  This completes the proof.
  All abstract policies
  $\pi_\phi$ are \emph{derived value--compatible}.
  \qed
\end{proof}

\begin{lemma}
  \label{thm:plannable_abstract_also_plannable}
  Let
  $\mdp = \langle \sset, \aset, p, r, \gamma \rangle$
  be a
  \emph{plannable} MDP
  with outcome function
  $\sigma$
  and reward weighting
  $\rw$
  such that
  $r = \sigma^\top \rw$,
	and let
	$\phi: \sset \mapsto \Phi$
	be an \emph{outcome equivalent} state abstraction.
	Then any abstract MDP
  $\mdp_\phi=\langle \Phi, \aset, p_\phi, r_\phi, \gamma \rangle$
	created from $\mdp$ and $\phi$
  with outcome function
  $\sigma_\phi$
  and reward weighting
  $w_{r_\phi}$
  such that
  $r_\phi = \sigma_\phi^\top w_{r_\phi}$
  and
  with the same reward weighting,
  i.e.
  $w_{r_\phi} = \rw$,
	is also plannable.
\end{lemma}

\begin{proof}[\ref{thm:plannable_abstract_also_plannable}]
  Since
  $\mdp$ is plannable,
  we know that for all states
  $s \in \sset$
  there exists an action sequence
  $(a_1,...,a_n) \in \aset^{n}$
  with $n$ going to $\infty$
  in the limit such that the following holds:
  \begin{align*}
    v_{\pi}(s)
        &=
    \expected_{p} \left[ \sum^{n}_{k=1} \gamma^{k-1} R_{t+k} \relmiddle| S_t=s, A_t=a_1, ..., A_{t+n-1}=a_n\right]
  \end{align*}
  From
  Proposition~\ref{thm:outcome_equivalent_abstract_ground},
  we know that
  all abstract states $\astate \in \Phi$
  are outcome equivalent with their ground states
  $s \in \astate$.
  Because $\astate$ and $s$ are outcome equivalent,
  the expected outcomes are also the same for any subsequence
  of actions
  (Proposition~\ref{thm:expected_outcome_sequence_sequence_expectations}).
  Their $\gamma$-discounted sums must also be,
  and so we can rewrite the above equality
  in terms of abstract outcomes,
  and,
  consequently,
  in terms of abstract rewards.
  Let the following shorthand expectations be conditioned
  on the action sequence such that
  $
      A_t=a_1,\dots,
      A_{t+n-1}=a_n
  $:
  \begin{align*}
    v_{\pi}(s)
    &=
    \expected_{p}
    \left[
      \sum^n_{k=1}
      \gamma^{k-1}
      \sigma(S_{t+k-1}, A_{t+k-1}, S_{t+k}) \relmiddle| S_t=s
    \right]^\top
    \rw
    \\
    &=
    \expected_{p}
    \left[
      \sum^n_{k=1}
      \gamma^{k-1}
      \sigma(S_{t+k-1}, A_{t+k-1}, S_{t+k}) \relmiddle| S_t=s
    \right]^\top
    w_{r_\phi}
    &&
    w_{r_\phi} = \rw
    \\
    &=
    \expected_{p_\phi}
    \left[
      \sum^n_{k=1}
      \gamma^{k-1}
      \sigma_\phi(\astaterv_{t+n-1}, A_{t+n-1}, \astaterv_{t+n}) \relmiddle| \astaterv_t=\phi(s)
    \right]^\top
    w_{r_\phi}
    &&
    \text{(Proposition~\ref{thm:expected_outcome_sequence_sequence_expectations})}
    \\
    &=
    \expected_{p_\phi}
    \left[
      \sum^n_{k=1}
      \gamma^{k-1}
      R_{t+k} \relmiddle| \astaterv_t=\phi(s)
    \right]
  \end{align*}
  We have that the value function for every ground policy
  $\pi$ can be expressed as the expected reward
  given an abstract state and a sequence of actions.
  Some of these ground policies are derived from abstract policies.
  Lemma~\ref{thm:plannable_abstract_value_compatible}
  states that all abstract policies are \emph{derived value--compatible}.
  Putting this together,
  we have
  \begin{equation*}
    v_{\pi}(s)
    =
    \expected_{p_\phi}
    \left[
      \sum^n_{k=1}
      \gamma^{k-1}
      R_{t+k} \relmiddle| \astaterv_t=\phi(s)
    \right]
  \end{equation*}
  This proves
  that $\mdp_\phi$ is also \emph{plannable}
  because it is constructed from a plannable MDP
  and an outcome equivalent state abstraction.
  \qed
\end{proof}

\begin{lemma}
  \label{thm:plannable_outcome_equivalent_value_equivalent}
  Let
  $\mdp = \langle \sset, \aset, p, r, \gamma \rangle$
  be a
  \emph{plannable} MDP
  and let
  $\phi: \sset \mapsto \Phi$
  be an \emph{outcome equivalent} state abstraction.
  Then it holds that all states
  in the same equivalence class
  under $\phi$
  have the same value for any policy.
  Formally,
  $\forall \astate \in \Phi$,
  $\forall s, s' \in \astate$
  and for any policy $\pi$:
  \begin{equation*}
    v_\pi(s) = v_\pi(s')
  \end{equation*}
\end{lemma}

\begin{proof}[\ref{thm:plannable_outcome_equivalent_value_equivalent}]
  Consider an abstract state $\astate \in \Phi$
  with two ground states
  $s, s' \in \astate$.
  They are outcome equivalent,
  meaning that for any action sequence
  $(a_1,...,a_n) \in \aset^{n}$
  (for any length $n$),
  the expected outcomes are the same
  when starting from either $s$ or $s'$.

	The following is analogous to
	Lemma~\ref{thm:plannable_outcome_equivalent_value_equivalent}.
  Let the following shorthand expectations be conditioned
  on the action sequence such that
  $
      A_t=a_1,\dots,
      A_{t+n-1}=a_n
  $:
  \begin{align*}
    \expected_{p}
    \left[
      \sigma(S_{t+n-1}, A_{t+n-1}, S_{t+n} \mid S_t=s)
    \right]
    =
    \expected_{p}
    \left[
      \sigma(S_{t+n-1}, A_{t+n-1}, S_{t+n} \mid S_t=s')
    \right]
  \end{align*}
  Because $s$ and $s'$ are outcome equivalent,
  the expected outcomes are also the same for any subsequence
  of actions
  (Proposition~\ref{thm:expected_outcome_sequence_sequence_expectations}).
  Their $\gamma$-discounted sums must also be.
  \begin{align*}
    \sum^n_{k=1}
    \gamma^{k-1}
    \expected_{p}
    \left[
      \sigma(S_{t+k-1}, A_{t+k-1}, S_{t+k}) \mid S_t=s
    \right]
    &=
    \sum^n_{k=1}
    \gamma^{k-1}
    \expected_{p}
    \left[
      \sigma(S_{t+k-1}, A_{t+k-1}, S_{t+k}) \mid S_t=s'
    \right]
    \\
    &\Updownarrow\\
    \expected_{p}
    \left[
      \sum^n_{k=1}
      \gamma^{k-1}
      \sigma(S_{t+k-1}, A_{t+k-1}, S_{t+k}) \relmiddle| S_t=s
    \right]^\top
    \rw
    &=
    \expected_{p}
    \left[
      \sum^n_{k=1}
      \gamma^{k-1}
      \sigma(S_{t+k-1}, A_{t+k-1}, S_{t+k}) \relmiddle| S_t=s'
    \right]^\top
    \rw
    \\
    &\Updownarrow\\
    \expected_{p}
    \left[
      \sum^n_{k=1}
      \gamma^{k-1}
      R_{t+k} \relmiddle| S_t=s
    \right]
    &=
    \expected_{p}
    \left[
      \sum^n_{k=1}
      \gamma^{k-1}
      R_{t+k} \relmiddle| S_t=s'
    \right]
  \end{align*}
  Now consider a policy $\pi$.
  Given that the above equality holds
  for all action sequences
  and $\mdp$ is plannable,
  it must be that
  $v_\pi(s) = v_\pi(s')$
  (by the definition of plannable).
  This concludes the proof.
  \qed
\end{proof}

\begin{lemma}
  \label{thm:plannable_outcome_equivalent_abstract_value_equivalent}
  Let
  $\mdp_\alpha=\langle \sset_\alpha, \aset, p_\alpha, r_\alpha, \gamma \rangle$
  and
  $\mdp_\beta = \langle \sset_\beta, \aset, p_\beta, r_\beta, \gamma \rangle$
  be two MDPs
  with respective outcome functions
  $\sigma_\alpha$ and $\sigma_\beta$
  and reward weightings
  $\rwa$ and $\rwb$
  such that
  $r_\alpha = \sigma_\alpha^\top \rwa$
  and
  $r_\beta = \sigma_\beta^\top \rwb$,
  and
  let
  $\phi_\alpha: \sset_\alpha \mapsto \asset_\alpha$
  and
  $\phi_\beta: \sset_\beta \mapsto \asset_\beta$
  be two state abstractions.
  Then for any abstract MDP
  $\mdp_\phia=\langle \Phia, \aset, p_\phia, r_\phib, \gamma \rangle$
  with outcome function
  $\sigma_\phia$
  and reward weighting
  $w_{r_\phia}$
  such that
  $r_\phia = \sigma_\phia^\top w_{r_\phia}$
  and
  such that
  $w_{r_\phia}=\rwa=\rwb$,
  it holds that all abstract policies
  for $\mdp_\phia$
  are \emph{derived value--compatible}
  with respect to $\mdpb$.
  Formally,
  for all abstract policies $\pi_\phia$,
  partially derived policies $\delta_\beta$,
  and for all states
  $s_\beta \in \phi_\beta^{-1}(\Phi_\alpha)$:
  \begin{equation*}
    v_{\delta_\beta}(s_\beta; \mdp_\beta) = v_{\pi_{\phi_\alpha}}(\phi_\beta(s_\beta); \mdp_{\phi_\alpha})
  \end{equation*}
\end{lemma}

\begin{proof}[\ref{thm:plannable_outcome_equivalent_abstract_value_equivalent}]
  Consider an abstract state $\astate \in \Phia \cap \Phib$
  with a ground state
  $s \in \phibi(\astate)$.
  From Proposition~\ref{thm:outcome_equivalent_abstract_ground},
  we know that $\astate$ is outcome equivalent with $s$,
  meaning that for any action sequence
  $(a_1,...,a_n) \in \aset^{n}$
  (for any length $n$),
  the expected outcomes are the same
  when starting from either $\astate$ or $s$.

  Let the following shorthand expectations be conditioned
  on the action sequence such that
  $
      A_t=a_1,\dots,
      A_{t+n-1}=a_n
  $:
  \begin{align*}
    \expected_{p_\phia}
    \left[
      \sigma_\phia(\astaterv_{t+n-1}, A_{t+n-1}, \astaterv_{t+n}) \mid \astaterv_t=\astate
    \right]
    =
    \expected_{p}
    \left[
      \sigma(S_{t+n-1}, A_{t+n-1}, S_{t+n} \mid S_t=s)
    \right]
  \end{align*}
  Because $\astate$ and $s$ are outcome equivalent,
  the expected outcomes are also the same for any subsequence
  of actions
  (Proposition~\ref{thm:expected_outcome_sequence_sequence_expectations}).
  Their $\gamma$-discounted sums must also be.
  \begin{align*}
    \sum^n_{k=1}
    \gamma^{k-1}
    \expected_{p_\phia}
    \left[
      \sigma_\phia(\astaterv_{t+n-1}, A_{t+n-1}, \astaterv_{t+n}) \mid \astaterv_t=\astate
    \right]
    &=
    \sum^n_{k=1}
    \gamma^{k-1}
    \expected_{p}
    \left[
      \sigma(S_{t+k-1}, A_{t+k-1}, S_{t+k}) \mid S_t=s
    \right]
    \\
    &\Updownarrow \qquad
    w_{r_\phia} = \rwb
    \\
    \expected_{p_\phia}
    \left[
      \sum^n_{k=1}
      \gamma^{k-1}
      \sigma_\phia(\astaterv_{t+n-1}, A_{t+n-1}, \astaterv_{t+n}) \relmiddle| \astaterv_t=\astate
    \right]^\top
    w_{r_\phia}
    &=
    \expected_{p}
    \left[
      \sum^n_{k=1}
      \gamma^{k-1}
      \sigma(S_{t+k-1}, A_{t+k-1}, S_{t+k}) \relmiddle| S_t=s
    \right]^\top
    \rw
    \\
    &\Updownarrow\\
    \expected_{p_\phia}
    \left[
      \sum^n_{k=1}
      \gamma^{k-1}
      R_{t+k} \relmiddle| \astaterv_t=\astate
    \right]
    &=
    \expected_{p}
    \left[
      \sum^n_{k=1}
      \gamma^{k-1}
      R_{t+k} \relmiddle| S_t=s
    \right]
  \end{align*}
  Now consider a policy $\pi_\phia$
  and its partially derived policy
  $\delta$.
  Both
  $\mdp_\phia$ and $\mdpb$
  are plannable
	as
	$\mdpb$
	is given to be plannable
	and $\mdp_\phia$ is plannable
because it is constructed from a plannable MDP
  (Lemma \ref{thm:plannable_abstract_also_plannable}).
  Given that the above equality holds
  for all action sequences
  it must be that
  $v_{\pi_\phia}(\astate) = v_\delta(s)$
  (by the definition of plannable).
  This proves that all policies
  $\pi_\phia$
  for
  $\mdp_\phia$
  are
  \emph{derived value--compatible}
  w.r.t. $\mdpb$.
\end{proof}

\begin{customtheorem}{\ref{thm:plannable_outcome_equivalent_transfer_optimal}}
    Let
  $\mdpa=\langle \sset_\alpha, \aset, p_\alpha, r_\alpha, \gamma \rangle$
  and
  $\mdpb= \langle \sset_\beta, \aset, p_\beta, r_\beta, \gamma \rangle$,
  be two \emph{plannable} MDPs
  with respective outcome functions
  $\sigma_\alpha$ and $\sigma_\beta$
  and reward weightings
  $\rwa$ and $\rwb$
  such that
  $r_\alpha = \sigma_\alpha^\top \rwa$
  and
  $r_\beta = \sigma_\beta^\top \rwb$,
  and let
  $\phi_\alpha: \sset_\alpha \mapsto \asset_\alpha$
  and
  $\phi_\beta: \sset_\beta \mapsto \asset_\beta$
  be two outcome equivalent state abstractions state abstractions.
  Then any abstract MDP
  $\mdp_\phia=\langle \Phia, \aset, p_\phia, r_\phib, \gamma \rangle$
  created from $\mdpa$ and $\phia$
  with outcome function
  $\sigma_\phia$
  and reward weighting
  $w_{r_\phia}$
  such that
  $r_\phia = \sigma_\phia^\top w_{r_\phia}$
  and
  such that
  $w_{r_\phia}=\rwa=\rwb$,
  is \emph{transfer optimal} for $\mdpb$.

\end{customtheorem}

\begin{proof}[\ref{thm:plannable_outcome_equivalent_transfer_optimal}]
  \label{proof:plannable_outcome_equivalent_transfer_optimal}
  Let
  $\mdpa$
  and
  $\mdpb$
  be two \emph{plannable} MDPs
  and let
  $\phi_\alpha: \sset_\alpha \mapsto \asset_\alpha$
  and
  $\phi_\beta: \sset_\beta \mapsto \asset_\beta$
  be two state abstractions
  such that $\Phia \cup \Phib$.

  In order to prove that
  the abstract MDP
  $\mdp_\phia$
  is
  \emph{transfer optimal}
  for
  $\mdpb$,
  we need to prove two properties
  (from the definition of transfer optimality):
  \begin{enumerate}
    \item
      all optimal abstract policies
      $\pi^*_{\phi_\alpha}: \Phi_\alpha \mapsto \probset(\aset)$
      of
      $\mdp_{\phi_\alpha}$
      are \emph{derived value--compatible}
      with respect to
      $\mdp_\beta$; and
    \item
      the \emph{guaranteed transfer value}
      of
      $\mdp_{\phi_\alpha}$
      for $\mdp_\beta$
      is maximal
      (i.e.
      the optimal policy for $\mdp_\beta$ can be partially derived from
      every optimal abstract policy for $\mdp_{\phi_\alpha}$).
  \end{enumerate}

  \paragraph{}
  First we prove \emph{derived value--compatibility}
  for all optimal abstract policies.
  This is a direct consequence of
  Lemma~\ref{thm:plannable_outcome_equivalent_abstract_value_equivalent},
  which states that all
  abstract policies are
  \emph{derived value--compatible},
  which is a stronger property.

  \paragraph{}
  For the second part of the proof,
  we prove that the optimal policy for $\mdp_\beta$
  can be partially derived from every optimal abstract policy
  $\pi_\phia$ for $\mdp_\phia$.
  We prove this part by contradiction,
  by assuming there exists
  some policy $\bar{\pi}$
  that is not derived from any $\pi_\phia$
  with a higher value than any partially derived policy
  $\delta$.
  We proceed in the same vain as the
  \emph{transfer value--control trade-off} theorem
  (Theorem~\ref{thm:transfer_value_control_tradeoff}).
  For it to be possible that $\bar\pi$
  has a higher value than any derived policy $\delta$,
  there must exist two states
  $s, t \in \ssetb$
	such that
	$\phib(s) = \phib(t) \in \Phia$,
	and
	$\bar\pi(\cdot | s) \neq \bar\pi(\cdot | t)$.
	That is,
	$\bar\pi$ has more control leading to an increase in value
	compared to $\delta$,
	i.e.
	$v_{\bar\pi}(s) > v_{\bar\pi}(t)$
	and $v_\delta(s) = v_\delta(t)$.
	Lemma~\ref{thm:plannable_outcome_equivalent_value_equivalent}
outcome equivalent states in a plannable MDP
have the same value for any policy,
i.e.
	$v_{\bar\pi}(s) = v_{\bar\pi}(t)$.
This is a direct contradiction and so
we conclude that there is no optimal policy
$\bar\pi$
that is not a derived policy.

Putting the two properties together,
we conclude that $\mdp_\phia$
is \emph{transfer optimal}
for $\mdpb$.

\qed
\end{proof}

\newpage
\section{Option Discovery}
\label{app:option_discovery}
Most of the below is due to
\citet{Daniel2016},
albeit reformulated.
Major changes are indicated.

\paragraph{}
Consider a demonstrated trajectory
as a sequence of states and actions
as experienced by an agent
in some MDP:
$${\tau = \langle s_1, a_1,\dots,s_T,a_{T},s_{T+1} \rangle},$$
The goal is to discover a hierarchical agent
that employs a high-level policy over a set of options
  $$\Omega = \left\lbrace \omega \middle| \omega=(\sset_\omega, \pi_\omega, \beta_\omega) \right\rbrace$$
that is most likely
to give rise to this same trajectory.
Since we deal with multiple tasks
with distinct state sets,
we assume
that for each task there is an
associated
state representation $\rep: \sset \mapsto \sset_\omega$.
The method outlined here works for any shared state set $\sset_\omega$,
though in our case
$\rep$ will always be an Outcome-Predictive State Representation,
so $\rep=\phi_{\sigma,k}$ with $\sset_\omega=\Phi_{\sigma,k}$.
That is,
all options operate on the same abstract OPSR.

Different from previous work,
we also consider an auxiliary option token
$\nulloption$
which represents ``no option'',
giving rise to the extended option set
\begin{equation}
  \Omega^+ = \Omega \cup \left\lbrace \nulloption{} \right\rbrace
\end{equation}
This allows us to write
the latent trajectory
describing the hierarchical agent's option occupancy as
\begin{equation}
  \label{eq:latent_trajectory}
  {\tau_L = \langle \omega_1,\dots,\omega_{T} \rangle \in {\Omega^+}^T }
\end{equation}
with tokens indicating
the active option $\omega_t$ at each timestep $t$.

\begin{algorithm}[t]
\caption{Pseudo-code of hierarchical generative model}
\label{alg:generative_hierarch}
\begin{algorithmic}
  \State $s_1 \sim p_{\text{init}}, \omega_{0} \gets \nulloption$
  \For{$t \gets 1\dots T$}
  \State draw $\beta_t \sim \beta_{\omega_{t-1}}(\rho_{\omega_{t-1}}(s_t))$
      \Comment{$\beta_{\nulloption}$ always returns $1$}
      \If{$\beta_t$}
        \State $\omega_t \gets \nulloption{}$
        \Comment{option terminates}
      \Else
        \State $\omega_t \gets \omega_{t-1}$
        \Comment{option continues}
      \EndIf
    \If{$\omega_t = \nulloption{}$}
    \State draw $\omega_t \sim \pi_H(\cdot|s_t)$
        \Comment{Draw new option or $\nulloption$}
    \EndIf
    \State draw $a_t \sim \pi_{\omega_t}(\cdot|\rho_{\omega_t}(s_t))$
    \State draw $s_{t+1} \sim p(\cdot|\rho_{\omega_{t}}(s_t), a_t)$
  \EndFor
\end{algorithmic}
\end{algorithm}

\paragraph{}
We assume that the two trajectories
$\tau$ and $\tau_L$
are generated by the generative model
described in
Algorithm~\ref{alg:generative_hierarch}.
This allows us to rewrite the probability of a trace as
\begin{align}
  \Pr(\tau) &= {\sum_{\tau_L} \Pr(\tau, \tau_L)} \\
  &= \sum_{t=1}^{T}
                            \Pr(\omega_t|s_t,\omega_{t-1})
                            \pi_{\omega_t}(a_t|s_t)
                            p(s_{t+1}|s_t,a_t)
  \label{eq:prob_trajectory}
\end{align}
The term
$\Pr(\omega_t|s_t,\omega_{t-1})$
is dependent on whether
an option was previously in use:
\begin{equation}
  \Pr(\omega_t|s_t,\omega_{t-1})
  = \begin{cases}
    \pi_H(\omega_t|s_t), & \omega_{t-1} = \nulloption{}\\
    (1-\beta_{\omega_{t-1}}(s_t))\delta_{\omega_{t-1}=\omega_t}
    + \beta_{\omega_{t-1}}(s_t)
    \pi_H(\omega_t|s_t) & \omega_{t-1} \neq \nulloption{}
  \end{cases}
  \label{eq:discovery_option_trans}
\end{equation}

Both option control and termination policies
are parameterized by parameters
$\theta_\omega = (\theta_{\pi_\omega}, \theta_{\beta_\omega})$,
and the high-level policy is parameterized by
$\theta_{\pi_H}$.
For example,
these could be the weights of neural networks.
Jointly, we write
$\theta = (\theta_{\pi_H}, (\theta_\omega)_{\omega \in \Omega})$%
.

We now phrase the problem
of option discovery
as maximizing the log-likelihood
of the demonstrated trajectory.
\begin{equation}
  \theta_{\textsc{max}} = \argmax_\theta L[\theta; \tau]
  = \argmax_\theta \log \Pr_\theta(\tau)
\end{equation}
The original Expectation-Gradient algorithm that performs this maximization
is due to
\citet{Daniel2016},
along with its extension by
\citet{Fox2017}
to allow for neural networks.
These versions
contain a more complicated representation
of the latent trajectory which included termination tokens.
This greatly complicates
both the representation and the implementation of the algorithm.
Below we present
the simpler yet equivalent algorithm
that results from removing termination tokens from the latent trajectory.
\begin{align}
  \nabla_\theta L[\theta; \tau]
  &=
  \nabla_\theta \log \Pr(\tau) \\
  &=
  \frac{\nabla_\theta \Pr(\tau)}{\Pr(\tau)}
  && \nabla \log x = \frac{\nabla x}{x} \\
  &=
  \sum_{\tau_L} \frac{\nabla_\theta \Pr(\tau, \tau_L)}{\Pr(\tau)} \\
  &=
  \sum_{\tau_L} \frac{\Pr(\tau, \tau_L)}{\Pr(\tau)} \nabla_\theta\log \Pr(\tau, \tau_L)
  && \nabla x = x \nabla \log x \\
  &=
  \sum_{\tau_L} \Pr(\tau_L | \tau) \nabla_\theta\log \Pr(\tau, \tau_L) \\
  &=
  \label{eq:expect_tau}
  \expected_{\tau_L|\tau; \theta}
  \left[ \nabla_\theta \log \Pr(\tau_L, \tau) \right]
\end{align}
\paragraph{}
Expand \eqref{eq:expect_tau}
using \eqref{eq:prob_trajectory}.
Note how the environment's transition probabilities drop away
in the derivative as they do not rely on $\theta$.
\begin{align}
  &\expected_{\tau_L|\tau; \theta}
  \left[ \nabla_\theta \log \Pr(\tau_L, \tau) \right]
  = \nonumber \\
  \label{eq:expect_tau_expanded}
  &\sum_{t=1}^{T-1}\sum_{\omega_t}\sum_{\omega_{t-1}}
  \Pr(\omega_t, \omega_{t-1}|\tau)
  \nabla_\theta \log\Pr(\omega_t|\omega_{t-1},s_t)
  +
  \sum_{t=1}^{T-1}\sum_{\omega_t}
  \Pr(\omega_t|\tau)
  \nabla_\theta \log\pi_{\omega_{t}}(a_t|s_t,\omega_t)
\end{align}
\subsubsection*{Baum-Welch forward-backward algorithm}
Equation \eqref{eq:expect_tau_expanded}
relies on two posteriors:
\begin{align}
  \label{eq:marginal_posterior}
  u_t(\omega)&=\Pr({\bbomega}_t=\omega|\tau)\\
  \label{eq:joint_posterior}
  v_t(\omega,\omega')&=\Pr({\bbomega}_t=\omega', \bbomega_{t-1}=\omega|\tau)
\end{align}
\paragraph{}
We can model the problem as a
Hidden Markov Model (HMM)
and make use of a modified version
of the Baum-Welch algorithm
to make computing these posteriors computationally feasible.
To sketch the HMM,
consider that the observed action at
each time step relies only on the state
and the latent active option.

The Baum-Welch algorithm is based on the following observation.
We start by rewriting
equation \eqref{eq:marginal_posterior}.
\begin{align}
  \Pr(\omega_t|\tau)
  &= \frac{\Pr(\tau|\omega_t)\Pr(\omega_t)}{\Pr(\tau)}
  && \text{Bayes' rule}
  \\
  &= \frac{\Pr(\tau_{1:t}|\omega_t)\Pr(\tau_{t+1:T})|\omega_t)\Pr(\omega_t)}{\Pr(\tau)}
  && \tau_{1:t} \perp \tau_{t+1:T} | \omega_t
  \\
  &=
  \label{eq:marginal_posterior_expanded}
  \frac{%
  \Pr(\tau_{1:t}, \omega_t)\Pr(\tau_{t+1:T}|\omega_t)}
  {\Pr(\tau)}
  && \text{}
\end{align}
Equation \eqref{eq:marginal_posterior_expanded}
gives us the two main entities of the algorithm.
\begin{align}
  \alpha_t(\omega_t)
  &= \Pr(\tau_{1:t},\omega_t) \\
  \beta_t(\omega_t)
  &= \Pr(\tau_{t+1:T}|\omega_t)
\end{align}
These can be computed efficiently
with a recursive forward and backward pass
respectively,
following the dynamic programming paradigm.
Finally,
the marginal and joint posteriors can be recovered.
\begin{align}
  \Pr(\omega_t|\tau)
  &= z^{-1} \alpha_t(\omega_t)\beta_t(\omega_t)\\
  \Pr(\omega_t, \omega_{t+1}|\tau)
  &= z^{-1} \alpha_t(\omega_t)\beta_{t+1}(\omega_{t+1})
  \Pr(\omega_{t+1}|\omega_t,s_{t+1})
  \Pr(a_{t+1}|\omega_{t+1},s_{t+1})
\end{align}
where $z=\Pr(\tau)$ is the normalisation constant
first introduced in
\eqref{eq:marginal_posterior_expanded}.

Assuming $\theta$ fixed,
we now have a closed-form and exact expression
for the gradient in
\eqref{eq:expect_tau_expanded}.
Expectation-Gradient
simply alternates between
solving the expectation in
\eqref{eq:expect_tau_expanded}
which gives rise to a gradient.
This gradient can now be climbed
with a gradient ascent method.

\end{document}